\def\year{2021}\relax
\documentclass[letterpaper]{article} 
\usepackage{aaai21}  
\usepackage{times}  
\usepackage{helvet} 
\usepackage{courier}  
\usepackage[hyphens]{url}  
\usepackage{graphicx} 
\usepackage{natbib}  
\usepackage{caption} 
\usepackage{bm}
\usepackage{amsmath}
\usepackage{amsthm}
\usepackage{amsfonts}
\usepackage{algorithm}
\usepackage{algorithmic}
\usepackage{enumitem,color}
\usepackage{subcaption}
\usepackage{multirow}
\usepackage{booktabs}

\newtheorem{theorem}{Theorem}

\newtheorem{lemma}{Lemma}
\newtheorem{remark}{Remark}
\newtheorem{assumption}{Assumption}

\title{Fast and Scalable Adversarial Training of Kernel SVM\\ via Doubly Stochastic Gradients}
\author {
	Huimin Wu\textsuperscript{\rm 1},
	Zhengmian Hu\textsuperscript{\rm 2},
	Bin Gu\textsuperscript{\rm 1,3,4}\thanks{Corresponding Authors}\\\ 
}
\affiliations{
	\textsuperscript{\rm 1} School of Computer \& Software, Nanjing University of Information Science \& Technology, P.R.China\\
	\textsuperscript{\rm 2}Department of Electrical \& Computer Engineering, University of Pittsburgh, PA, USA\\
	\textsuperscript{\rm 3}JD Finance America Corporation, Mountain View, CA, USA\\
	\textsuperscript{\rm 4}MBZUAI, United Arab Emirates\\

	wuhuimin@nuist.edu.cn,
	huzhengmian@gmail.com,
	jsgubin@gmail.com,
	
}

\begin{document}
	
	\maketitle
	
	\begin{abstract}
		Adversarial attacks by generating examples which are almost indistinguishable from natural examples, pose a serious threat to learning models. Defending against adversarial attacks is a critical element for a reliable learning system. Support vector machine (SVM) is a classical yet still important learning algorithm even in the current deep learning era. Although a wide range of researches have been done in recent years to improve the adversarial robustness of learning models, but most of them are limited to deep neural networks (DNNs) and the work for kernel SVM is still vacant. In this paper, we aim at kernel SVM and propose adv-SVM to improve its adversarial robustness via adversarial training, which has been demonstrated to be the most promising defense techniques. To the best of our knowledge, this is the first work that devotes to the fast and scalable adversarial training of kernel SVM. Specifically, we first build connection of perturbations of samples between original and kernel spaces, and then give a reduced and equivalent formulation of adversarial training of kernel SVM based on the connection. Next, doubly stochastic gradients (DSG) based on two unbiased stochastic approximations (i.e., one is on training points and another is on random features) are applied to update the solution of our objective function. Finally, we prove that our algorithm optimized by DSG converges to the optimal solution at the rate of O(1/$t$) under the constant and diminishing stepsizes. Comprehensive experimental results show that our adversarial training algorithm enjoys robustness against various attacks and meanwhile has the similar efficiency and scalability with classical DSG algorithm.
	\end{abstract}
	
	\section{Introduction}
	
	\begin{table*}[htbp]
		\small
		\begin{center}
			\begin{tabular}{cccccc}
				\toprule
				\multicolumn{1}{l}{}                      & \textbf{Algorithm} & \textbf{Reference} & \textbf{Applicable model} & \textbf{Inner maximization problem} & \textbf{Layers of objective function} \\ \hline
				\multicolumn{1}{c|}{\multirow{4}{*}{DNN}} & Standard                      &\cite{madry2017towards} &  nonlinear & $K$-PGD attack &2 \\
				\multicolumn{1}{c|}{}                     & MART & \cite{wang2020improving} & nonlinear & $K$-PGD attack &2 \\
				\multicolumn{1}{c|}{}                     &VAT & \cite{miyato2019virtual} & nonlinear & $K$-PGD attack &2 \\
				\multicolumn{1}{c|}{}                     & FAT & \cite{shafahi2019adversarial} & nonlinear & single-step FGSM attack &2 \\ \hline
				\multicolumn{1}{c|}{\multirow{2}{*}{SVM}} & adv-linear-SVM & \cite{zhou2012adversarial} & linear & closed-form solution  &1 \\
				\multicolumn{1}{c|}{}                     & \textbf{adv-SVM} &Ours & nonlinear & closed-form solution &1 \\ \bottomrule
			\end{tabular}
		\end{center}
		\caption{Comparisons of different adversarial training algorithms on DNN and SVM. 
			\label{tab:algorithm}}
	\end{table*}
	
	Machine learning models have long been proved to be vulnerable to adversarial attacks which generate subtle perturbations to the inputs that lead to incorrect outputs. The perturbed inputs are defined as adversarial examples where the perturbations that lead to misclassification are often imperceptible. This serious threat has recently led to a large influx of contributions in adversarial attacks especially for deep neural networks (DNNs). These methods of adversarial attacks include FGSM \cite{goodfellow2014explaining}, PGD \cite{madry2017towards}, C$\&$W \cite{carlini2017towards}, ZOO \cite{chen2017zoo:} and so on. We  give a brief review of  them in the section of related work.
	
	The topic of adversarial attacks has also attracted much attention in the field of SVM. In 2004, Dalvi et al. \shortcite{2004Adversarial} and later Lowd and Meek \shortcite{Lowd2005Adversarial,Lowd2005GoodWA} studied the task of spam filtering, showing that linear SVM could be easily tricked by few carefully crafted changes in the content of spam emails, without affecting their readability. Some other attacks such as label flipping attack \cite{Biggio_supportvector,xiao2012adversarial}, poison attack \cite{biggio2012poisoning,xiao2015is} and evasion attack \cite{biggio2013evasion} have also proved the vulnerability of SVM to adversarial examples.
	
	Due to the serious threat of these attacks, there is no doubt that defense techniques that can improve adversarial robustness of learning models are crucial for secure machine learning. 
	Most defensive strategies nowadays focus on DNNs, such as defensive distillation \cite{papernot2016distillation}, gradient regularization \cite{ross2018improving} and adversarial training \cite{madry2017towards}, among which adversarial training has been demonstrated to be the most effective \cite{athalye2018obfuscated}. This method focuses on a min-max problem, where the inner maximization is to find the most aggressive adversarial examples and the outer minimization is to find model parameters that minimize the loss on the adversarial examples. Up till now, there have been many forms of adversarial training on DNNs which further improve their robustness and training efficiency compared with standard adversarial training \cite{shafahi2019adversarial,carmon2019unlabeled,miyato2019virtual,wang2020improving}. 
	
	Since SVM is a classical and important learning model in machine learning, the improvement of its security and robustness is also critical. However, to the best of our knowledge, the only work of adversarial training for SVM is limited to linear SVMs. Specifically, Zhou et al. \shortcite{zhou2012adversarial} formulated a convex adversarial training formula for linear SVMs, in which the constraint is defined over the sample space based on two kinds of attack models. As we know, datasets with complex structures can be hardly classified by linear SVMs, but can be easily handled by kernel SVMs. We give a brief review of adversarial training strategies 
	of DNNs and SVMs  in Table \ref{tab:algorithm}. From this table, it is easy to find that how to improve the robustness of kernel SVMs against adversarial examples is still an unstudied problem.
	
	To  fill the vacancy, in this paper, we focus on kernel SVMs and propose adv-SVM to improve their adversarial robustness via adversarial training. To the best of our knowledge, this is the first work that devotes to fast and scalable adversarial training of kernel SVMs. Specifically, we first build connections of perturbations between the original and kernel spaces, i.e., $\phi(x+\delta)$ and $\phi({x})+\delta_\phi$, where $\delta$ is the perturbation added to the normal example in the original space,  $\delta_\phi$ is the perturbation in the kernel space and $\phi(\cdot)$ is the corresponding feature mapping. Then we construct the simplified and equivalent form of the inner maximization and transform the min-max objective function into a convex minimization problem based on the connection of the perturbations. However, directly optimizing this minimization problem is still difficult since the kernel function, which is necessary for the optimization, needs $O(n^2d)$ operations to be computed, where $n$ is the number of training examples and $d$ is the dimension. Huge requirement of computation complexity hinders its application on large scale datasets.
	
	To further improve its scalability, we apply the doubly stochastic gradient (DSG) algorithm \cite{dai2014scalable} to solve our objective. Specifically, in each iteration, we randomly select one training example and one random feature parameter for approximating the value of a kernel function instead of computing it directly.  The DSG algorithm effectively reduces the computation complexity of solving kernel methods from $O(n^2d)$ to $O(td)$, where $t$ is the number of iterations \cite{gu2018asynchronous}.  
	
	The main contributions of this paper are summarized as follows:
	\begin{itemize}
		\item We develop an adversarial training strategy, adv-SVM, for kernel SVM based on the relationship of perturbations between the original and kernel spaces and transform it into an equivalent convex optimization problem, then apply the DSG algorithm to solve it in an efficient way.
		\item We provide comprehensive theoretical analysis to prove that  our proposed adv-SVM can converge to the optimal solution at the rate of $O(1/t)$ with either a constant stepsize or a diminishing stepsize even though it is based on the approximation principle.
		\item We investigate the performance of adv-SVM under typical white-box and black-box attacks. The empirical results suggest our proposed algorithm can complete the training process in a relatively short time and stay robust in face of various attack strategies at the same time.
	\end{itemize}

	\section{Related Work}
	
	During the development of adversarial machine learning, a wide range of attacking methods have been proposed for the crafting of adversarial examples. Here we mention some frequently used ones which are useful for generating adversarial examples in our experiments.
	
	\begin{itemize}
		\item \textbf{Fast Gradient Sign Method (FGSM)}. FGSM, which belongs to white-box attacks, perturbs normal examples for one step by the amount $\epsilon$ along the gradient \cite{goodfellow2014explaining}. 
		\item \textbf{Projected Gradient Descent (PGD)}. PGD, which is also a white-box attack method, perturbs normal examples for a number of steps $K$ with a smaller step size and keeps the adversarial examples in the $\epsilon$-ball where $\epsilon$ is the maximum allowable perturbation \cite{madry2017towards}. 
		\item \textbf{C$\&$W}. C$\&$W is also a white-box attack method, which aims to find the minimally-distorted adversarial examples. This method is acknowledged to be one of the strongest attacks up to date \cite{carlini2017towards}.
		\item \textbf{Zeroth Order Optimization (ZOO)}. ZOO is a black-box attack based on coordinate descent. It assumes that the attackers only have access to the prediction confidence from the victim classifier's outputs. This method is proved to be as effective as C$\&$W attack \cite{chen2017zoo:}. 
	\end{itemize}

	It should be noted that although these methods are proposed for DNN models, they are also applicable to other learning models. We can  apply them to generate adversarial examples for SVM models.

	\section{Background}
	In this section, we give a brief review of adversarial training on linear SVM and the random feature approximation algorithm.
	\subsection{Adversarial Training on Linear SVM}
	We assume that SVM has been trained on a 2-class dataset $\mathbb{P}=\{(x_i,y_i)\}_{i=1}^n$  with $x_i\in\mathbb{R}$ as a normal example in the $d$-dimensional input space and $y_i\in\left\lbrace +1,\;-1 \right\rbrace $ as its label.
	
	The adversarial training process aims to train a robust learning model using adversarial examples. As illustrated in \cite{zhou2012adversarial}, it is formulated as solving a min-max problem. The inner maximization problem simulates the behavior of an attacker which constructs adversarial examples leading to the maximum output distortion:
	\begin{align}\label{eq: inner maximization}
	\max_{x'_i}&\quad\left[1-y_i\left(w^Tx'_i+b\right)\right]_+\\
	\nonumber s.t. &\quad \left\|x'_i-x_i\right\|_{2}\leq\epsilon
	\end{align}
	where $x'_i$ is the adversarial example of $x_i$, $\epsilon$ denotes the limited range of perturbations, $w$ and $b$ are the parameters of SVM. The loss function here is the commonly used hinge loss and we express it as $[\cdot]_+$.
	
	The outer minimization problem targets to find parameters that minimize the loss caused by inner maximization.
	\begin{small}
		\begin{align}\label{eq: outer minimization}
		\nonumber \min_{w,b}&\quad\frac{1}{2}\left\|w\right\|_2^2 +\frac{C}{n}\sum_{i=1}^{n}\max_{x'}\ \left[1-y_i\left(w^Tx'_i+b\right)\right]_+ \\
		s.t.& \quad \left\|x'_i-x_i\right\|_{2}\leq\epsilon
		\end{align}
	\end{small}
	Due to the limited application of linear SVMs, we aim to extend the adversarial training strategy to kernel SVMs.
	\subsection{Random Feature Approximation}
	Random feature approximation is a powerful technique used in DSG to make kernel methods scalable. This method approximates the kernel function by mapping the decision function to a lower dimensional random feature space. Its theoretic foundation relies on the intriguing duality between kernels and random processes \cite{geng2019scalable}. 
	
	Specifically,  
	the Bochner theorem \cite{rudin1962fourier} provides a relationship between the kernel function and a random process $\omega$ with measure $p$: for any stationary and continuous kernel $k(x,x')=k(x-x')$, there exits a probability measure $p$, such that $k(x,x')=\int_{\mathbb{R}^d}e^{i\omega^T(x-x')}p(\omega)d\omega$. In this way, the value of the kernel function can be approximated by explicitly computing random features $\phi_{\omega_i}(x)$, i.e., 
	\begin{align}\label{eq: random feature approximation}
	k(x,x')\approx\frac{1}{m}\sum_{i=1}^{m}\phi_{\omega_i}(x)\phi_{\omega_i}^T(x')=\phi_{\omega}(x)\phi_{\omega}^T(x')
	\end{align}
	where $m$ is the number of random features, $\phi_{\omega_i}(x)$ denotes $[\cos(\omega_i^Tx),\sin(\omega_i^Tx)]^T$ and $\phi_{\omega}(x)$ denotes $\frac{1}{\sqrt{m}}[\phi_{\omega_1}(x),\phi_{\omega_2}(x),\cdots,\phi_{\omega_m}(x)]$. For the detailed derivation process, please refer to \cite{ton2018spatial}.
	
	It is clearly that feature mappings are $\mathbb{R}^d\rightarrow\mathbb{R}^{2m}$, where $m\ll d$. To further alleviate computational costs, $\phi_{\omega_i}(x)$ can be expressed as $\sqrt{2}\cos(\omega_i^Tx+b)$, where $\omega_i$ is drawn from $p(\omega)$ and $b$ is drawn uniformly from $[0,2\pi]$.
	
	It is known that most kernel methods can be  expressed as convex optimization problems in  reproducing kernel Hilbert space (RKHS) \cite{dai2014scalable}. A RKHS $\mathcal{H}$ has the reproducing property, i.e., $\forall x\in\mathcal{X},\ k(x,\cdot)\in\mathcal{H}$ and $\forall f\in\mathcal{H}$, $\langle f(\cdot), k(x,\cdot)\rangle_{\mathcal{H}}=f(x)$. Thus we have $\nabla f(x)=k(x,\cdot)$ and $\nabla\left\|f\right\|_{\mathcal{H}}^2=2f$ \cite{dang2020Large}.
	
	\section{Adversarial Training on Kernel SVM}
	In this section, we extend the objective function (\ref{eq: outer minimization}) of linear SVM to kernel SVM, where the difficulty lies in the uncontrollable of the perturbations mapped in the kernel space.
	\begin{figure}[!ht]
		\centering
		\includegraphics[width=0.8\linewidth]{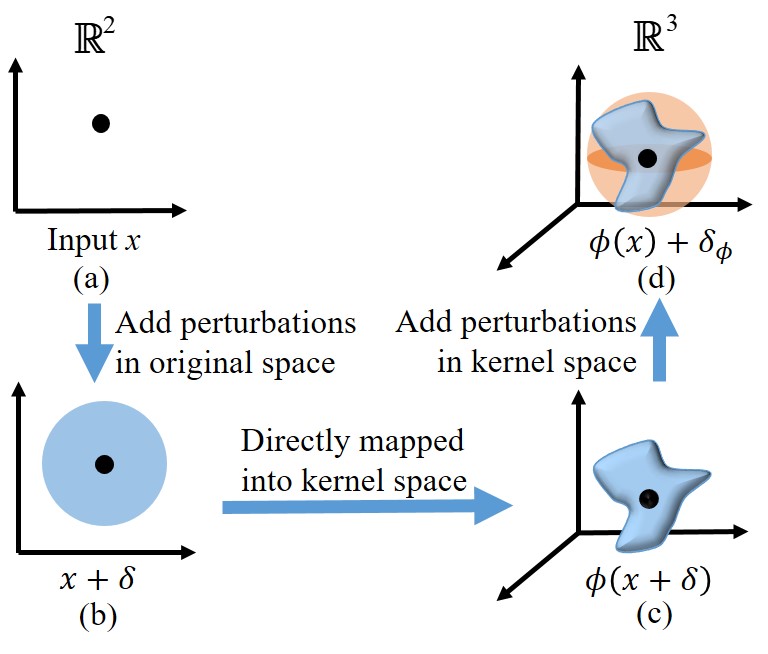}
		\caption{Conceptual illustration of perturbations in the original and kernel spaces.}
		\label{fig:perturbation}
	\end{figure}
	
	\subsection{Kernelization}
	Firstly, we discuss the kernelization of the perturbations. When constructing an adversarial example, we first add perturbations to the normal example $x$ in the original space constrained by $\left\|\delta\right\|_{2}^2\leq\epsilon^2$ as shown in Figure \ref{fig:perturbation}(a) and \ref{fig:perturbation}(b). But once the adversarial example is mapped into the kernel space, it will become unpredictable like Figure \ref{fig:perturbation}(c). Then the irregular perturbations greatly increase the difficulty of computation and the obtainment of the closed-form solution.

	Fortunately, the following theorem provides a tight connection between perturbations in the original space and the kernel space.
	\begin{theorem}\cite{xu2009robustness}
		Supposing that the kernel function has the form $k(x,x')=f(\left\|x-x'\right\|)$, with $f: \mathbb{R}^+\rightarrow\mathbb{R}$, a decreasing function, which is denoted by $\mathcal{H}$ the RKHS space of $k(\cdot,\cdot)$ and $\phi(\cdot)$ the corresponding feature mapping, then we have for any $x\in\mathbb{R}^n, w\in\mathcal{H}$ and $c\textgreater0$,
		\begin{align}
		\nonumber\sup_{\left\|\delta\right\|\leq c}\langle w,\phi(x+\delta)\rangle\leq\sup_{\left\|\delta_\phi\right\|_\mathcal{H}\leq\sqrt{2f(0)-2f(c)}}\langle w,\phi({x})+\delta_\phi\rangle.
		\end{align}
	\end{theorem}
	Since the perturbation range of $\phi(x)+\delta_\phi$ tightly covers that of $\phi(x+\delta)$, which is also intuitively illustrated in Figure \ref{fig:perturbation}(d), then we apply $\phi(x)+\delta_\phi$ to deal with the following computation, making the problem a linear problem in the kernel space. Thus, the inner maximization problem (\ref{eq: inner maximization}) in the kernel space can be expressed as
	\begin{align}\label{eq: kernelized inner maximization}
	 \max_{x'}\quad&\left[1-y_i\left(w^T\Phi(x'_i)+b\right)\right]_+\\
	\nonumber \ s.t. \quad& \left\|\Phi(x'_i)-\phi(x_i)\right\|_{2}\leq\epsilon'
	\end{align}
	where $\Phi(x'_i)$  denotes  $\phi(x_i)+\delta_\phi$ and $\epsilon'$ is $\sqrt{2f(0)-2f(\epsilon)}$.
	\subsection{Construction of the Equivalent Form\footnote{There is an error here in the version of AAAI 2021, please refer to this version as the correct one.}}
	In this part, we aim to get the simplified and equivalent form of Eq.   (\ref{eq: kernelized inner maximization}) via the following theorem, then the min-max optimization problem in the kernel space can be transformed into a minimization problem.
		\begin{theorem}\label{theorem: theorem1}
		With the constraint $\|\Phi(x'_i)-\phi(x_i)\|\leq\epsilon'$, the maximization problem  
		$
		\max_x[1-y_i(w^T\Phi(x'_i)+b)]_+
		$
		is equivalent to the  regularized loss function
		$
		[1-y_iw^T\phi(x_i)+\epsilon'\left\|w\right\|_2-y_ib]_+
		$.
	\end{theorem}
	The detailed proof of Theorem \ref{theorem: theorem1} is provided in our appendix. 
	Then, the original min-max objective function can be rewritten as the following minimization problem:
	\begin{align}\label{eq: simplified objective function}
	\min_{w\in\mathcal{H},b}\!\frac{1}{2}\!\left\|\!w\!\right\|_2^2\!+\!\frac{C}{n}\!\sum_{i=1}^{n}\!\left[1\!-\!y_iw^T\!\phi(x_i)\!+\!\epsilon'\left\|\!w\!\right\|_2\!-\!y_ib\right]_+
	\end{align}
	
	\section{Learning Strategy of adv-SVM}
	In this section, we extend the DSG algorithm to solve the objective minimization problem (\ref{eq: simplified objective function}), since DSG has been proved to be a powerful technique for scalable kernel learning \cite{dai2014scalable}. 
	
	For easy expression, we substitute $f(\cdot)$ for $w^T\phi(\cdot)$ in Eq. (\ref{eq: simplified objective function}), as $\left\|f\right\|_2^2=w^T\phi(\cdot)\phi(\cdot)^Tw=\left\|w\right\|_2^2$, which is accessible to kernels which satisfy $\phi(\cdot)\phi(\cdot)^T=1$, such as RBF and Laplacian kernels \cite{hajiaboli2011edge}, then the objective function can be expressed as follows:
	\begin{align}\label{eq: DSG objective function}
	&\min_{f\in\mathcal{H}} R(f)\\
	\nonumber= & \min_{f\in\mathcal{H}} \frac{1}{2}\left\|f\right\|_2^2+\frac{C}{n}\sum_{i=1}^{n}\left[1-y_if(x_i)+\epsilon'\left\|f\right\|_2-y_ib\right]_+
	\end{align}
	\subsection{Doubly Stochastic Gradients}
	In this part, we use the DSG algorithm to update the solution of Eq. (\ref{eq: DSG objective function}). For convenience, here we only discuss the case when the hinge loss is greater than 0.

	\paragraph{Stochastic Gradient Descent.}
	To iteratively update $f$ in a stochastic manner, we need to sample a data point $(x,y)$ each iteration from the data distribution. The stochastic functional gradient for $R(f)$ is
	\begin{align}\label{eq: R(f)_1}
	\nabla R(f) = f(\cdot)+C\left[-yk(x,\cdot)+\epsilon'\frac{f(\cdot)}{\left\|f\right\|_2}\right]
	\end{align}
	It is noted that $\nabla R(f)$ is the derivative wrt. $f$. Since it still costs too much if we compute the kernel functions directly, next, we apply the random feature approximation algorithm introduced earlier to approximate the value of the kernels.
	\paragraph{Random Feature Approximation.}
	According to Eq. (\ref{eq: random feature approximation}), when sampling random process $\omega$ from its probability distribution $p(\omega)$, we can further approximate Eq. (\ref{eq: R(f)_1}) as
	\begin{align}\label{eq: R(f)_2}
	\nabla \hat{R}(f) = f(\cdot)+C\left[-y\phi_{\omega}(x)\phi_{\omega}(\cdot)+\epsilon'\frac{f(\cdot)}{\left\|f\right\|_2}\right]
	\end{align}
	
	\paragraph{Update Rules.}
	According to the principle of SGD method, the update rule for $f$ in the $t$-th iteration is
	\begin{align}\label{eq: update rules_true} f_{t+1}(\cdot)&=f_t(\cdot)-\gamma_t\nabla \hat{R}(f)=\sum_{i=1}^{t}a_t^i\zeta_i(\cdot)
	\end{align}
	where $\gamma_t$ is the stepsize of the $t$-th iteration,   $\zeta_i(\cdot)$  denotes $-Cy_i\phi_{\omega_i}(x_i)\phi_{\omega_i}(\cdot)$ and the initial value $f_1(\cdot)=0$. The value of $a_t^i$ can be easily inferred as  $-\gamma_i\prod_{j=i+1}^{t}\left(1-\gamma_j\left(1+\frac{\epsilon'C}{\left\|f_j\right\|_2}\right)\right)$\footnote{The value of $a_t^i$ is gotten by expanding the middle term of Eq. (\ref{eq: update rules_true}) iteratively with the definition of $\nabla \hat{R}(f)$.}.
	
	Note that if we compute the value of kernels explicitly instead of using random features, the update rule for $f$ is 
	\begin{align}\label{eq: update rules}
	h_{t+1}(\cdot)=h_t(\cdot)-\gamma_t\nabla R(f)=\sum_{i=1}^ta_t^i\xi_i(\cdot).
	\end{align}
	where $\xi_i(\cdot)=-Cy_ik(x_i,\cdot)$.
	Our algorithm apply the update rule (\ref{eq: update rules_true}) instead, which can reduce the cost of kernel computation.
	\paragraph{Detailed Algorithm.}
	Based on Eq. (\ref{eq: update rules_true}) above, we propose the training and prediction algorithms for the adversarial training of kernel SVM in Algorithm \ref{alg: algorithm 1} and \ref{alg: algorithm 2} respectively.
	
	\begin{algorithm}
		\caption{$\{\alpha_i\}_{i=1}^t=\textbf{Train}(\mathbb{P}(x,y))$}
		\renewcommand{\algorithmicrequire}{\textbf{Input:}}
		\renewcommand{\algorithmicensure}{\textbf{Output:}}
		\label{alg: algorithm 1}
		\begin{algorithmic}[1]
			\REQUIRE $\mathbb{P}(x,y),\;p(\omega),\;C$.
			\FOR {$i=1,\cdots,t$}
			\STATE Sample $(x_i,y_i)\sim\mathbb{P}(x,y).$
			\STATE Sample $\omega_i\sim p(\omega)$ with seed $i$;
			\STATE $f(x_i)=\textbf{Predict}(x_i,\{\alpha_j\}_{j=1}^{i-1}).$
			\STATE Define $\gamma_i=\eta$ (constant) or $\gamma_i=\frac{\theta}{i}$ (diminishing).
			\STATE $\alpha_i=\gamma_i Cy_i\phi_{\omega_i}(x_i).$
			\STATE $\alpha_j=(1-\gamma_i(1+\frac{\epsilon'C}{\left\|f_j\right\|_2}))\alpha_j$ for $j=1,\cdots,i-1$
			\ENDFOR
		\end{algorithmic}
	\end{algorithm}

	\begin{algorithm}
		\caption{$f(x)=\textbf{Predict}(x,\{\alpha_i\}_{i=1}^t)$}
		\renewcommand{\algorithmicrequire}{\textbf{Input:}}
		\renewcommand{\algorithmicensure}{\textbf{Output:}}
		\label{alg: algorithm 2}
		\begin{algorithmic}[1]
			\REQUIRE $p(\omega),\phi_{\omega}(x)$.
			\STATE Set $f(x)=0.$
			\FOR {$i=1,\cdots,t$}
			\STATE Sample $\omega_i\sim p(\omega)$ with seed $i$;
			\STATE $f(x)=f(x)+\alpha_i\phi_{\omega_i}(x)$.
			\ENDFOR
		\end{algorithmic}
	\end{algorithm}
	
	A crucial step of DSG in Algorithm \ref{alg: algorithm 1} and \ref{alg: algorithm 2} is sampling $\omega_i$ with seed $i$. As the seeds are aligned for the training and prediction processes in the same iteration \cite{shi2019quadruply}, we only need to save the seeds instead of the whole random features, which is memory friendly.
	
	Different to the diminishing stepsize used in the original version of DSG \cite{dai2014scalable}, our algorithm here supports both diminishing and constant stepsize strategies (line 5 of Algorithm \ref{alg: algorithm 1}). The process of gradient descent is composed of a transient phase followed by a stationary phase. In the diminishing stepsize case, the transient phase is relatively long and can be impractical if the stepsize is misspecified \cite{toulis2017asymptotic}, but once entering the stationary phase, it will converge to the optimal solution $f_*$ gently. While in the constant stepsize case, the transient phase is much shorter and less sensitive to the stepsize \cite{chee2018convergence}, but it may oscillate in the region of $f_*$ during the stationary phase. 
	\section{Convergence Analysis}
	In this section, we aim to prove that adv-SVM can converge to the optimal solution at the rate of $O(1/t)$ based on the framework of \cite{dai2014scalable}, where $t$ is the number of iterations. We first provide some assumptions.
	
	\begin{assumption}(Bound of kernel function)
		There exists $\kappa\textgreater 0$, such that $k(x,x')\leq \kappa$.
	\end{assumption}
	
	\begin{assumption}(Bound of random feature norm)
		There exits $\phi\textgreater 0$, such that $\left|\phi_{\omega}(x)\phi_{\omega}(x')\right|\leq\phi$.
	\end{assumption}
	
	\begin{assumption}\label{assumption:upper bound}
		The spectral radius $\rho(f)$ of a function $f(\cdot)$ has a lower bound that $\rho(f)\geq\epsilon'C\geq 0$, where a spectral radius is the maximum modulus of eigenvalues\cite{Gurvits2007On}, i.e., $\rho(f)=\max_{1\leq i\leq \infty}\left\lbrace\left|\sqrt{\lambda_i}\right| \right\rbrace$.
	\end{assumption}
	
	For Assumption \ref{assumption:upper bound}, it is known that we can find $n$ eigenvalues $\left\lbrace \lambda_i \right\rbrace_{i=1}^n $ for a matrix $A$ in $\mathbb{R}^n$ space and the spectral radius $\rho(A)$ of matrix $A$ is defined as the maximum modulus of the eigenvalues of $A$ \cite{Gurvits2007On}, i.e., $\rho(A)=\max_{1\leq i\leq n}\left\lbrace\left|\lambda_i\right| \right\rbrace $. Similar to matrix case, in RKHS space, a function $f(\cdot)$ can be viewed as an infinite matrix, then infinite eigenvalues $\left\lbrace \sqrt{\lambda_i} \right\rbrace_{i=1}^\infty $ and infinite eigenfunctions $\left\lbrace \psi_i \right\rbrace_{i=1}^\infty$ can be found \cite{Iii04fromzero}. Treat $\left\lbrace \sqrt{\lambda_i}\psi_i \right\rbrace_{i=1}^\infty$ as a set of orthogonal basis, then $f(\cdot)$ can be represented as the linear combination of the basis, i.e., $f=\sum_{i=1}^{\infty}f_i\sqrt{\lambda_i}\psi_i$. Similar to the definition of spectral radius in matrix, for function $f(\cdot)$, $\rho(f)=\max_{1\leq i\leq \infty}\left\lbrace\left|\sqrt{\lambda_i}\right| \right\rbrace$.
	
	We update the solution $f$ through random features and random data points according to (\ref{eq: update rules_true}). As a result, $f_{t+1}$ may be outside of RKHS $\mathcal{H}$, making it hard to directly evaluate the error between $f_{t+1}(\cdot)$ and the optimal solution $f_*$. In this case, we utilize $h_{t+1}(\cdot)$ as an intermediate value to decompose the difference between $f_{t+1}$ and $f_*$ \cite{shi2020Quadruply}:
	\begin{align}\label{eq: total error}
	&\left|f_{t+1}(x)-f_*\right|^2\\
	\nonumber\leq & 2\!\underbrace{\left|f_{t+1}(x)\!-\!h_{t+1}(x)\right|^2}_{\rm error\; due\; to\; random\; features} \! + \! 2\kappa \! \underbrace{\left\|h_{t+1}\!-\!f_*\right\|_2^2.}_{\rm error\; due\; to\; random\; data}
	\end{align}
	
	We introduce our main lemmas and theorems as below. All the detailed proofs are provided in our appendix.
	
	\subsection{Convergence Analysis on Diminishing Stepsize}
	We first prove that the convergence rate of our algorithm with diminishing stepsize is $O(1/t)$.
	\begin{lemma}\textbf{(Error due to random features)}\label{lemma: error due to random feature}
		For any $x\in\mathcal{X}$,
		\begin{align}
		\nonumber\mathbb{E}_{\mathcal{D}^t,\;\omega^t}\left[\left|f_{t+1}(x)-h_{t+1}(x)\right|^2 \right]	&\leq\frac{1}{t^2}C^2\theta^2(k+\phi)^2
		\end{align}
	\end{lemma}
	
	
	
	\begin{lemma}\label{lemma: error due to random data}\textbf{(Error due to random data)} Let $f_*$ be the optimal solution to our target problem, we set $\gamma_t=\frac{\theta}{t}$ with $\theta$ such that $\frac{1}{2}<\theta\leq1$, then we have
		\begin{align}
		\mathbb{E}_{\mathcal{D}_t,\bm{\omega}_t}[\left\|h_{t+1}-f_*\right\|_2^2]\leq\frac{Q_1^2}{t}
		\end{align}
		where $Q_1=\max\left\lbrace\left\|f_*\right\|_2,\frac{\beta_0+\sqrt{\beta_0^2+4(2\theta-1)\beta}}{2(2\theta-1)} \right\rbrace $, $\beta=C^2\theta^2\left[(\kappa+\epsilon')+\kappa^{1/2}\theta\right]^2$, $\beta_0$ is a constant value and $\beta_0\textgreater 0$.
	\end{lemma}
	
	\begin{theorem}(\textbf{Convergence in expectation})\label{theroem: 2}
		When $\gamma_t=\frac{\theta}{t}$ with $\theta\in(\frac{1}{2},1]$, $\forall x\in\mathcal{X}$,
		\begin{align}
		\mathbb{E}_{\mathcal{D}^t,\omega^t}\left[ \left|f_{t+1}(x)-f_*\right|^2\right]\leq\frac{2Q_0^2}{t}+\frac{2\kappa Q_1^2}{t}
		\end{align}
		where $Q_0=C\theta(\kappa+\phi)$.
	\end{theorem}
	
	\begin{remark}
		According to Eq. (\ref{eq: total error}), the error caused by doubly stochastic approximation can be computed via the combination of Lemma \ref{lemma: error due to random feature} and \ref{lemma: error due to random data} and we prove in Theorem \ref{theroem: 2} that it converges at the rate of $O(1/t)$.
	\end{remark}
	
	\subsection{Convergence Analysis on Constant Stepsize}
	In this part, we provide a novel theoretical analysis to prove that adv-SVM with constant stepsize converges to the optimal solution at a rate near $O(1/t)$.
	
	Notice that the diminishing stepsize $\theta/t$ provides $1/t$ to the convergence rate, while in the case of constant stepsize, the stepsize $\eta$ makes the analysis more challenging.
	\begin{lemma}\textbf{(Error due to random features)}\label{lemma: error due to random feature_constant}
		For any $x\in\mathcal{X}$,
		\begin{align}
		\nonumber\mathbb{E}_{\mathcal{D}^t,\omega^t}\left[\left|f_{t+1}(x)-h_{t+1}(x)\right|^2 \right] \leq
		C^2\frac{\eta}{c}(\kappa+\phi)^2 
		\end{align}
	\end{lemma}
	
	
	\begin{lemma} \label{lemma: error due to random data_constant}(\textbf{Error due to random data})
		Let $f_*$ be the optimal solution to our target problem, set $t\in[T]$ and $\eta\in(0,1)$, with $\eta=\frac{\epsilon\vartheta}{2B}$ for $\vartheta\in(0,1]$, we will reach $\mathbb{E}_{\mathcal{D}^t,\omega^t,\omega'^t}[\left\|h_{t+1}-f_*\right\|_2^2]\leq\epsilon$ after
		\begin{align}
		&T\geq\frac{B\log(2e_1/\epsilon)}{\vartheta\epsilon}
		\end{align}
		iterations, where $B=\frac{1}{2}C\left[(\kappa+\epsilon')+\kappa^{1/2}\frac{1}{c}\right]^2$ and $e_1=\mathbb{E}_{\mathcal{D}_1,\mathcal{\omega}_1}[\left\|h_1-f_*\right\|_2^2]$.
	\end{lemma}
	
	\begin{figure*}[htb]
		\centering
		\begin{subfigure}[b]{0.24\textwidth}
			\includegraphics[width=1.65in]{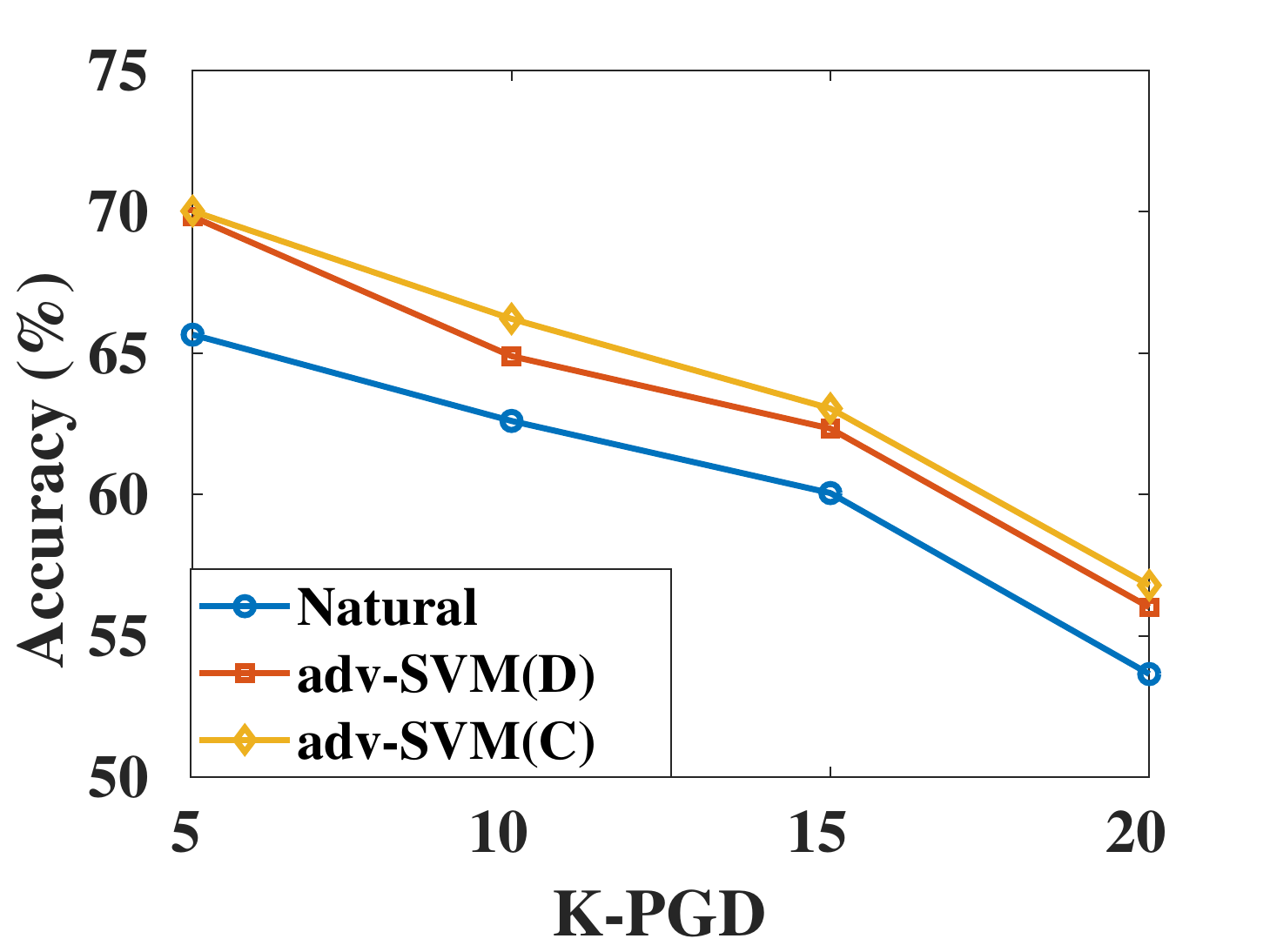}
			\caption{CIFAR automobile vs. truck}\label{fig:CIFAR K-PGD}
		\end{subfigure}
		\begin{subfigure}[b]{0.24\textwidth}
			\includegraphics[width=1.65in]{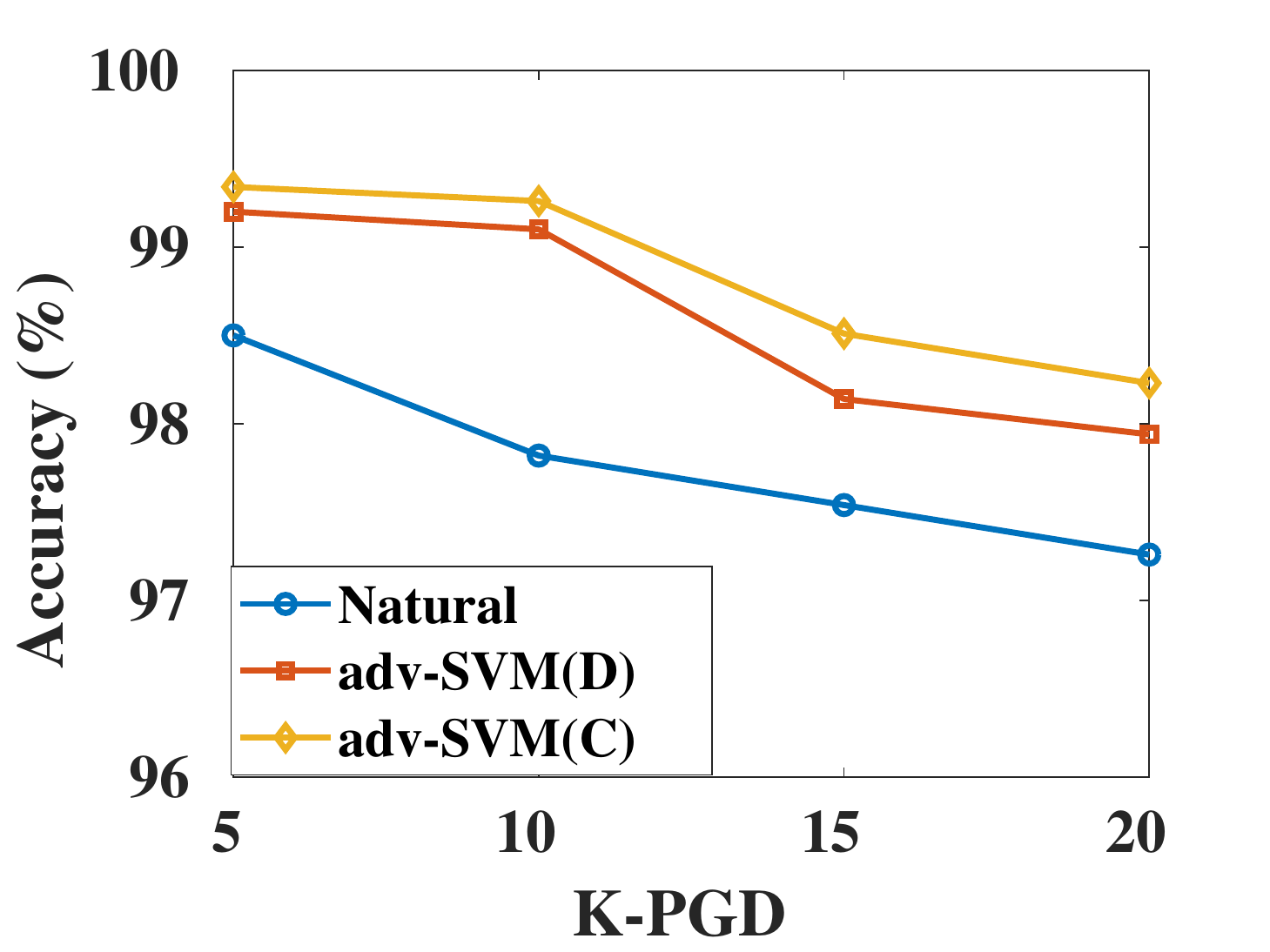}
			\caption{MNIST8m 0 vs. 4}\label{fig:MNIST K-PGD}
		\end{subfigure}
		\begin{subfigure}[b]{0.24\textwidth}
			\includegraphics[width=1.65in]{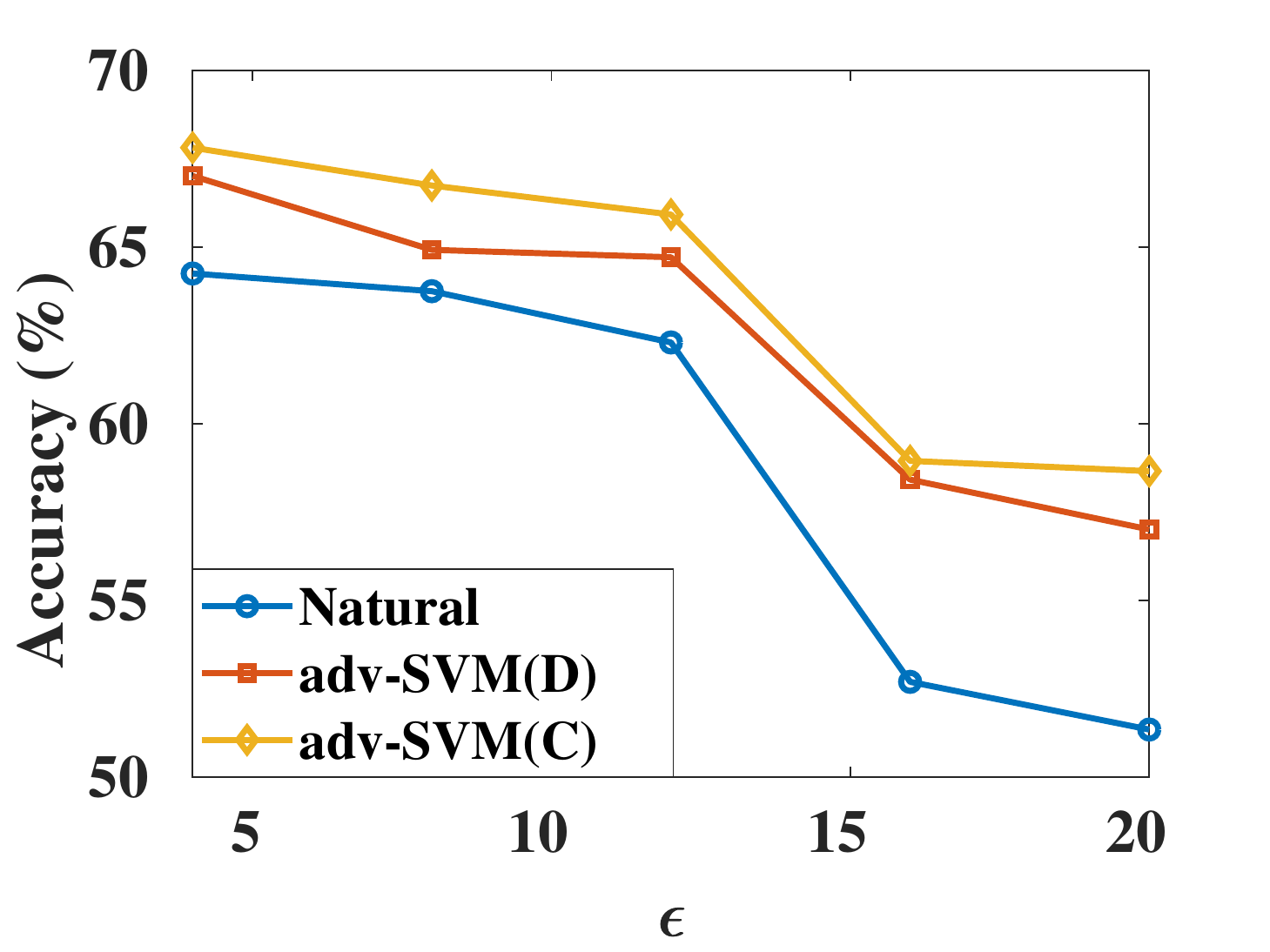}
			\caption{CIFAR automobile vs. truck}\label{fig:CIFAR epsilon}
		\end{subfigure}
		\begin{subfigure}[b]{0.24\textwidth}
			\includegraphics[width=1.65in]{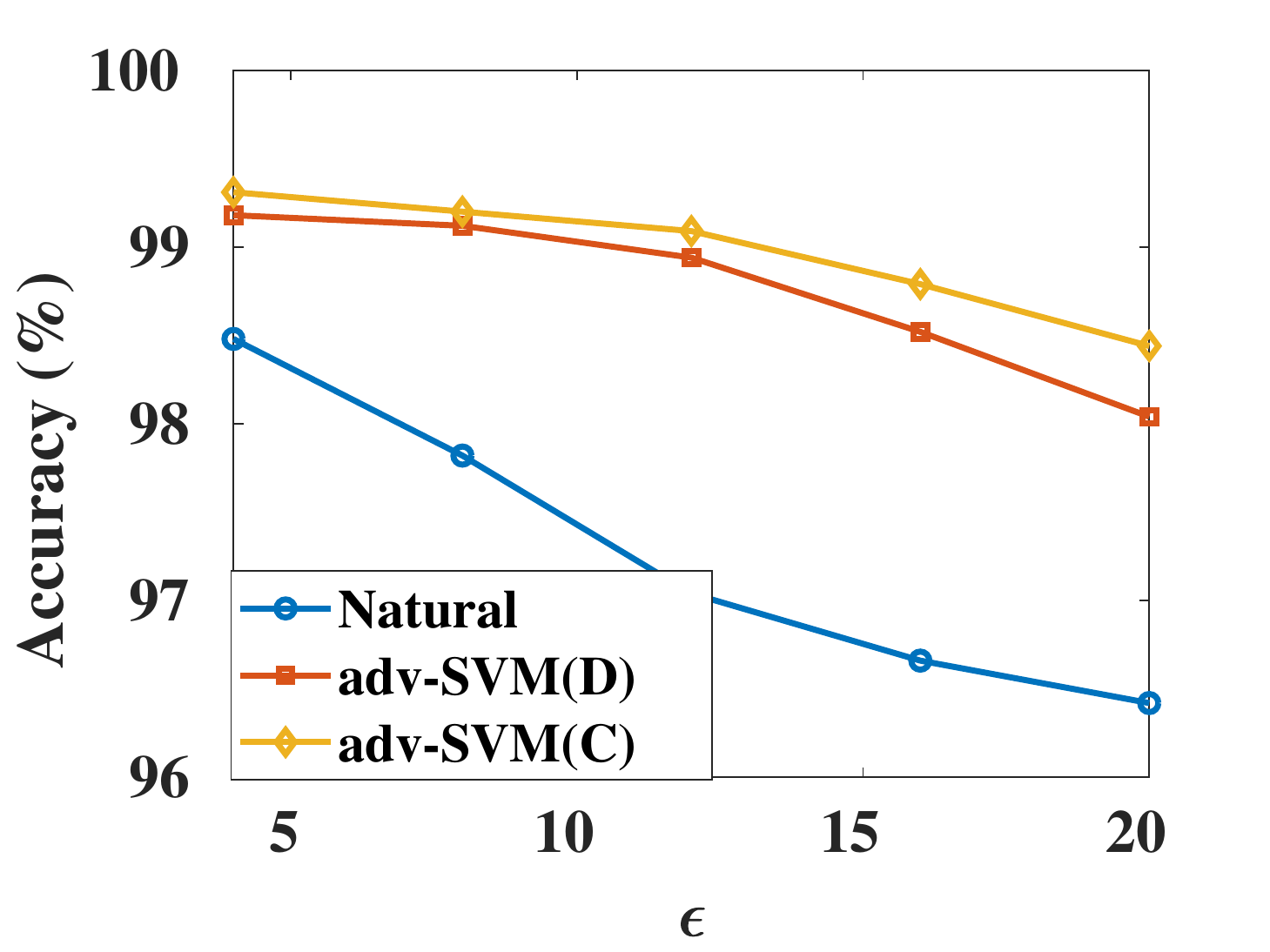}
			\caption{MNIST8m 0 vs. 4}\label{fig:MNIST epsilon}
		\end{subfigure}
		\caption{Accuracy of different models when applying different steps PGD attack (Fig. \ref{fig:CIFAR K-PGD}, \ref{fig:MNIST K-PGD}) and different max perturbation $\epsilon$ (Fig. \ref{fig:CIFAR epsilon}, \ref{fig:MNIST epsilon}) to generate adversarial examples.
		}
		\label{fig:figure 1}
	\end{figure*}
	
	\begin{figure*}[htb]
		\centering
		\begin{subfigure}[b]{0.24\textwidth}
			\includegraphics[width=1.7in]{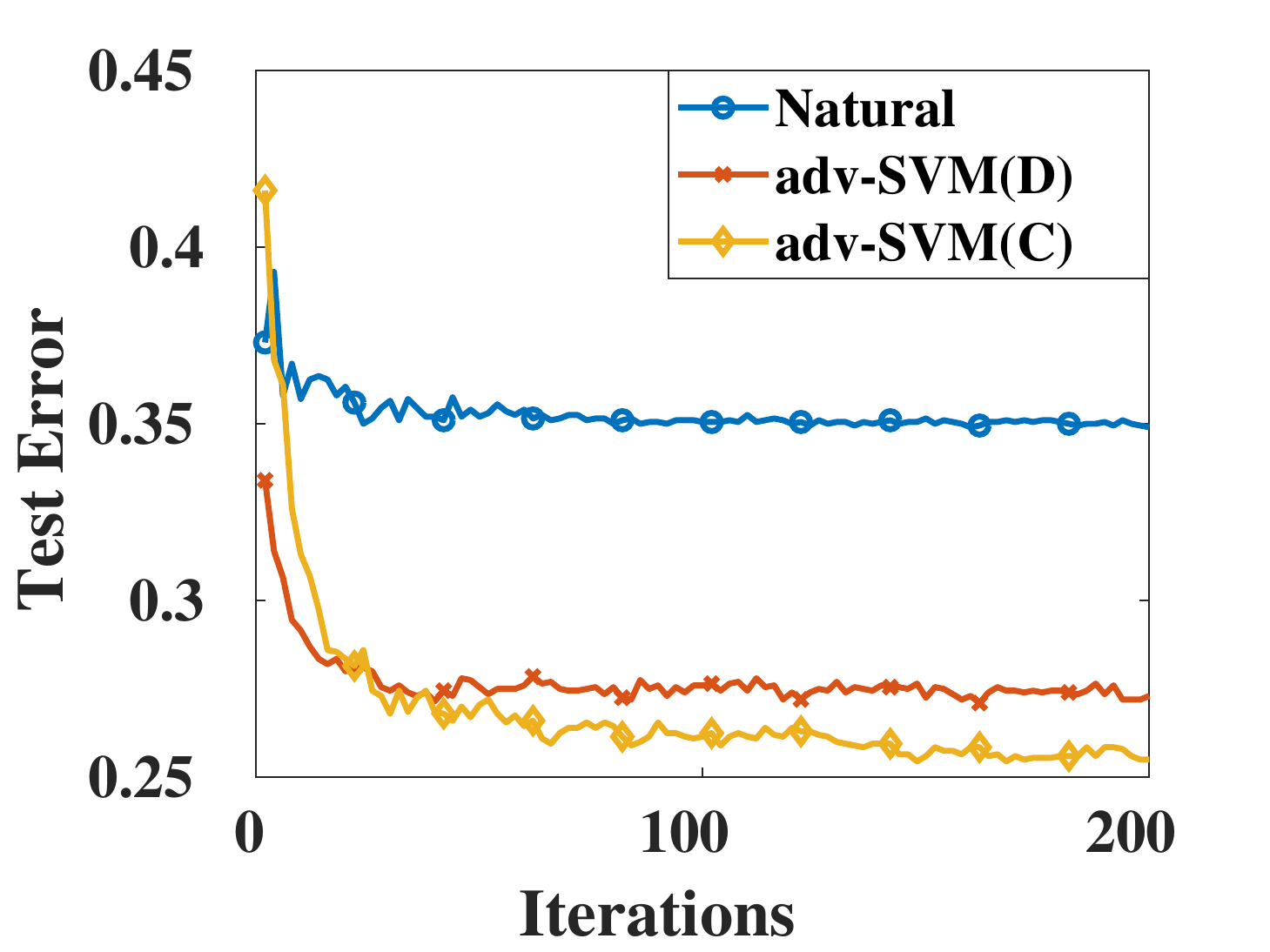}
			\caption{FGSM}\label{fig:cifar1v9_FGSM}
		\end{subfigure}
		\begin{subfigure}[b]{0.24\textwidth}
			\includegraphics[width=1.7in]{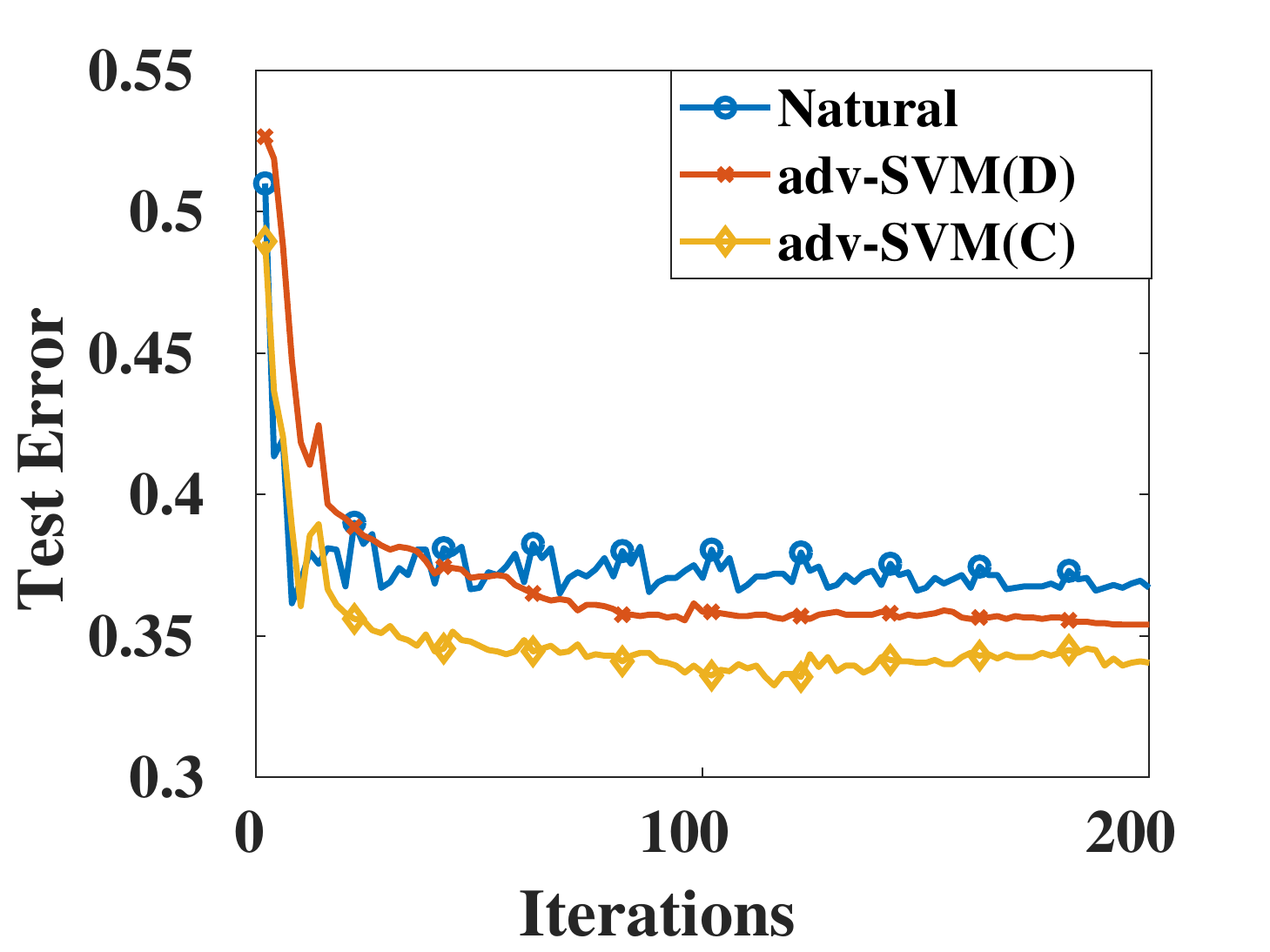}
			\caption{PGD}\label{fig:cifar1v9_PGD}
		\end{subfigure}
		\begin{subfigure}[b]{0.24\textwidth}
			\includegraphics[width=1.7in]{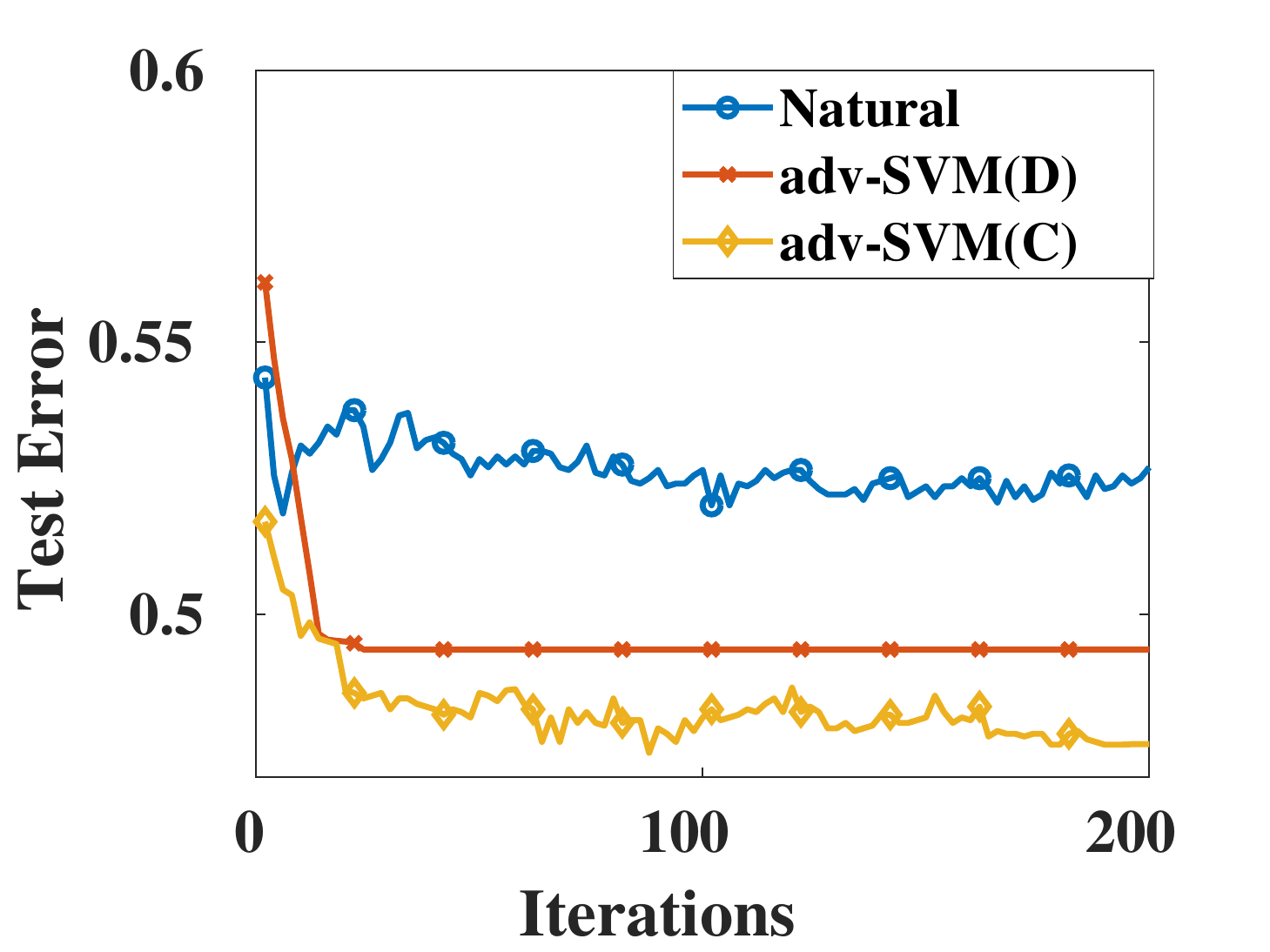}
			\caption{C$\&$W}\label{fig:cifar1v9_cw}
		\end{subfigure}
		\begin{subfigure}[b]{0.24\textwidth}
			\includegraphics[width=1.7in]{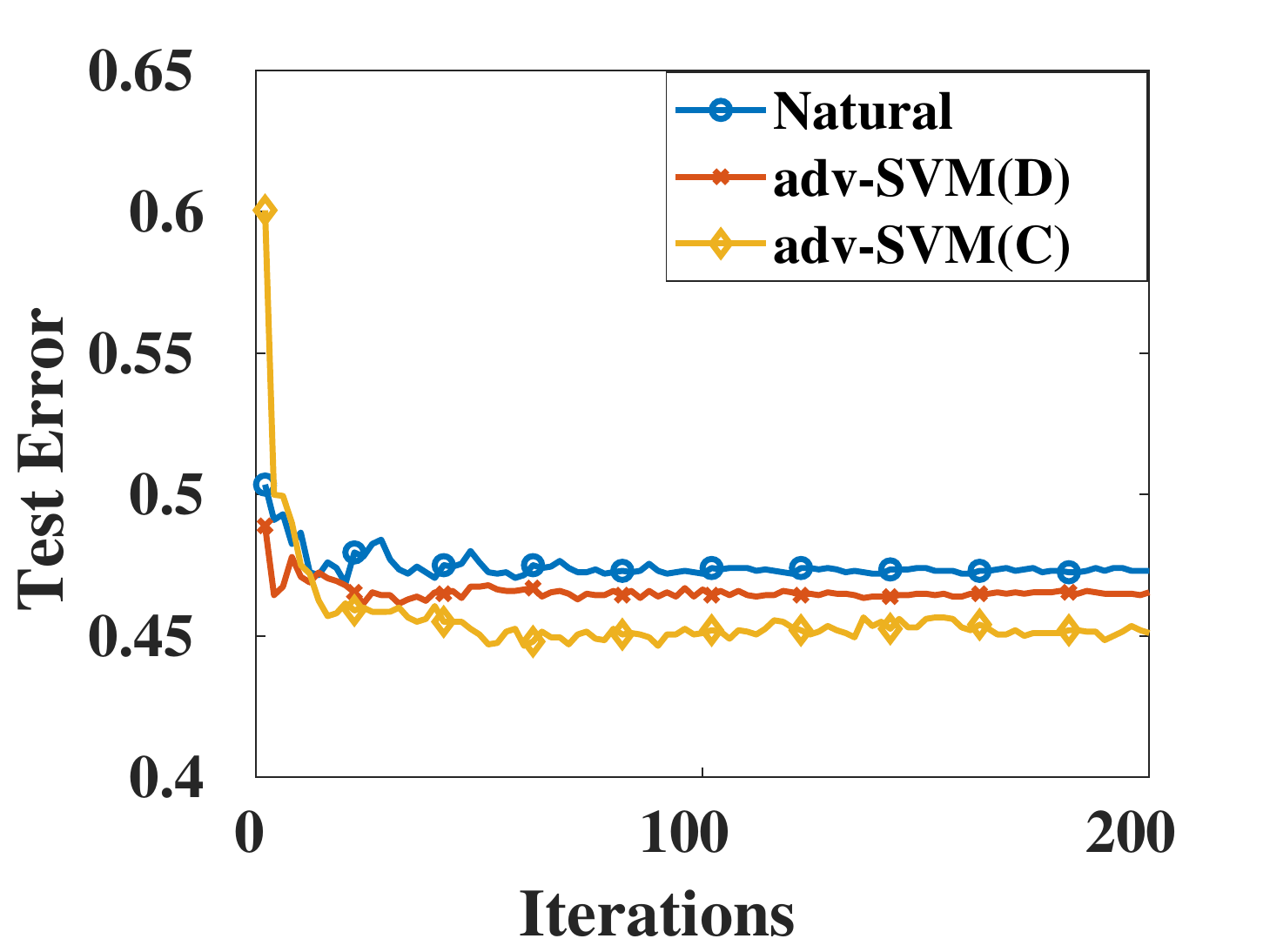}
			\caption{ZOO}\label{fig:cifar1v9_zoo}
		\end{subfigure}
		\begin{subfigure}[b]{0.24\textwidth}
			\includegraphics[width=1.7in]{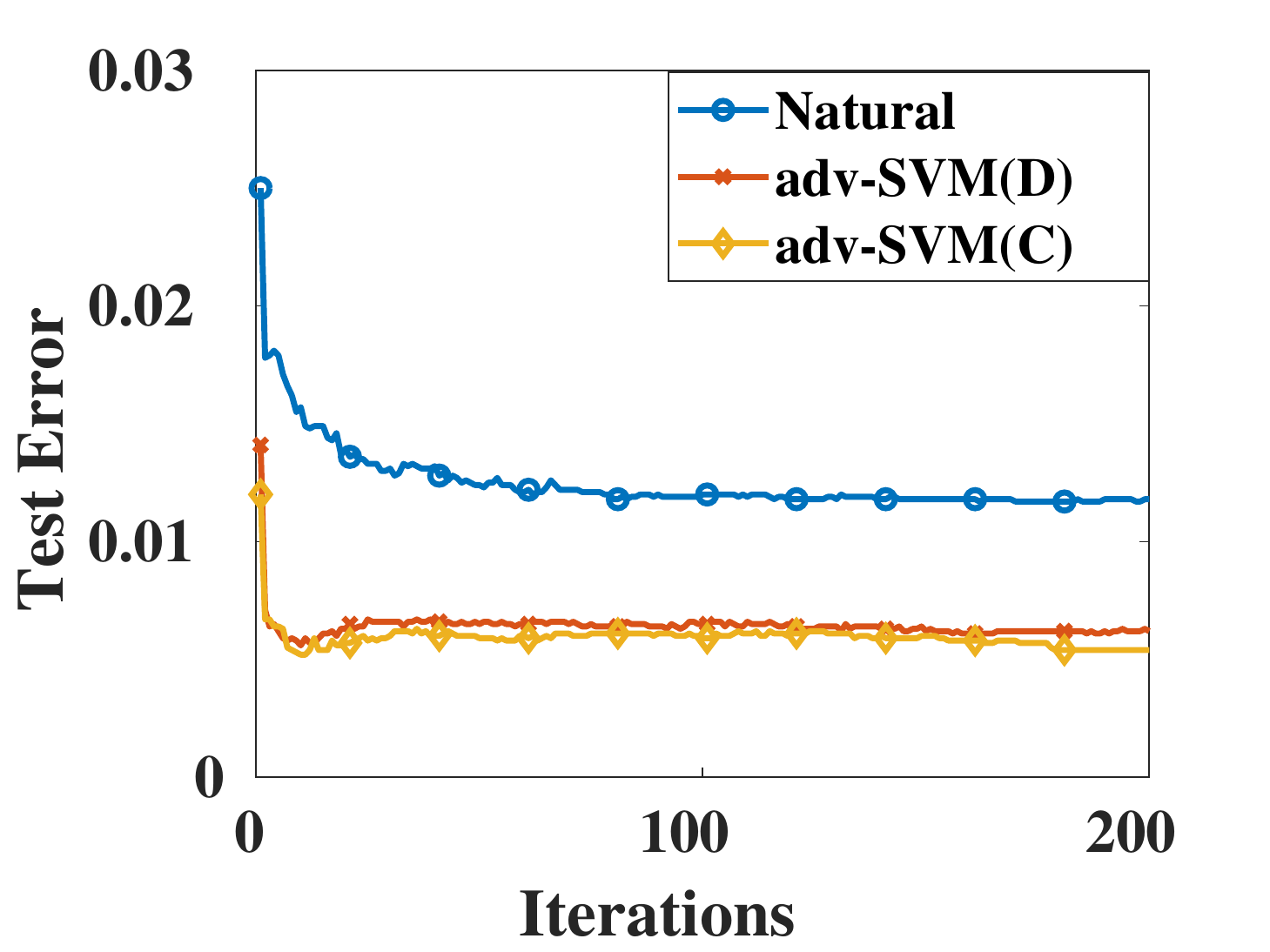}
			\caption{FGSM}\label{fig:mnist8m_0v4_FGSM}
		\end{subfigure}
		\begin{subfigure}[b]{0.24\textwidth}
			\includegraphics[width=1.7in]{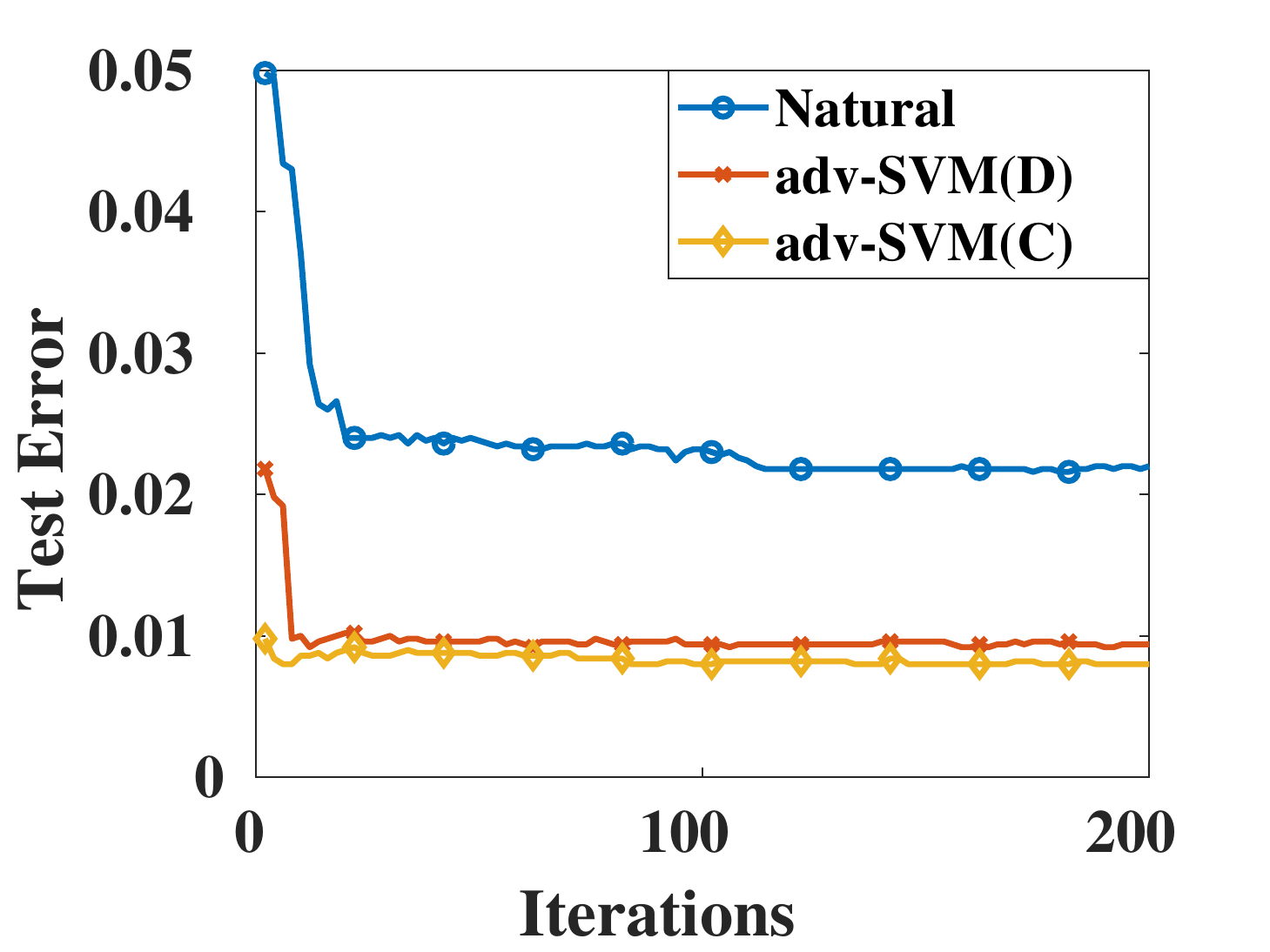}
			\caption{PGD}\label{fig:mnist8m_0v4_PGD}
		\end{subfigure}
		\begin{subfigure}[b]{0.24\textwidth}
			\includegraphics[width=1.7in]{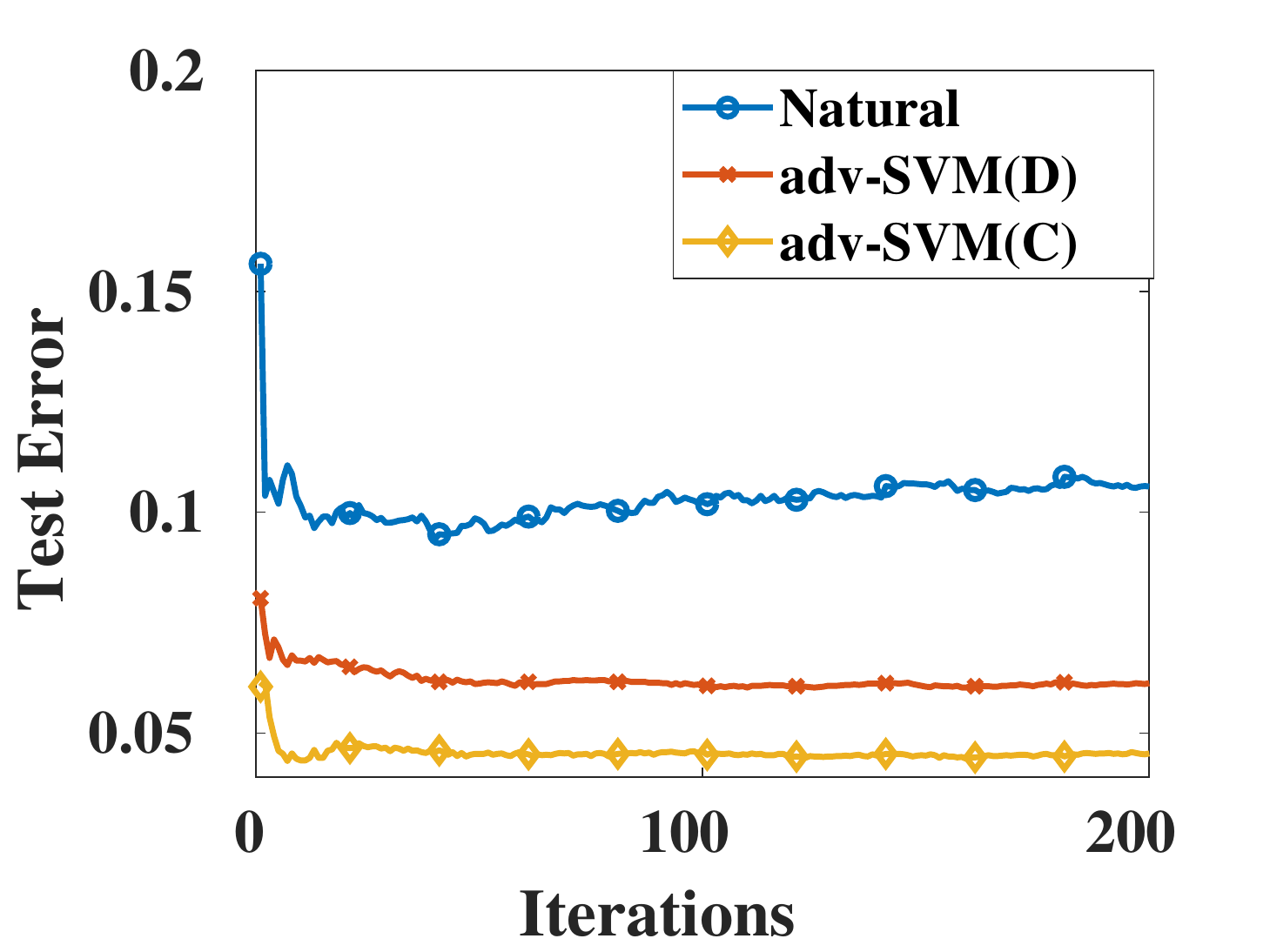}
			\caption{C$\&$W}\label{fig:mnist8m_0v4_cw}
		\end{subfigure}
		\begin{subfigure}[b]{0.24\textwidth}
			\includegraphics[width=1.7in]{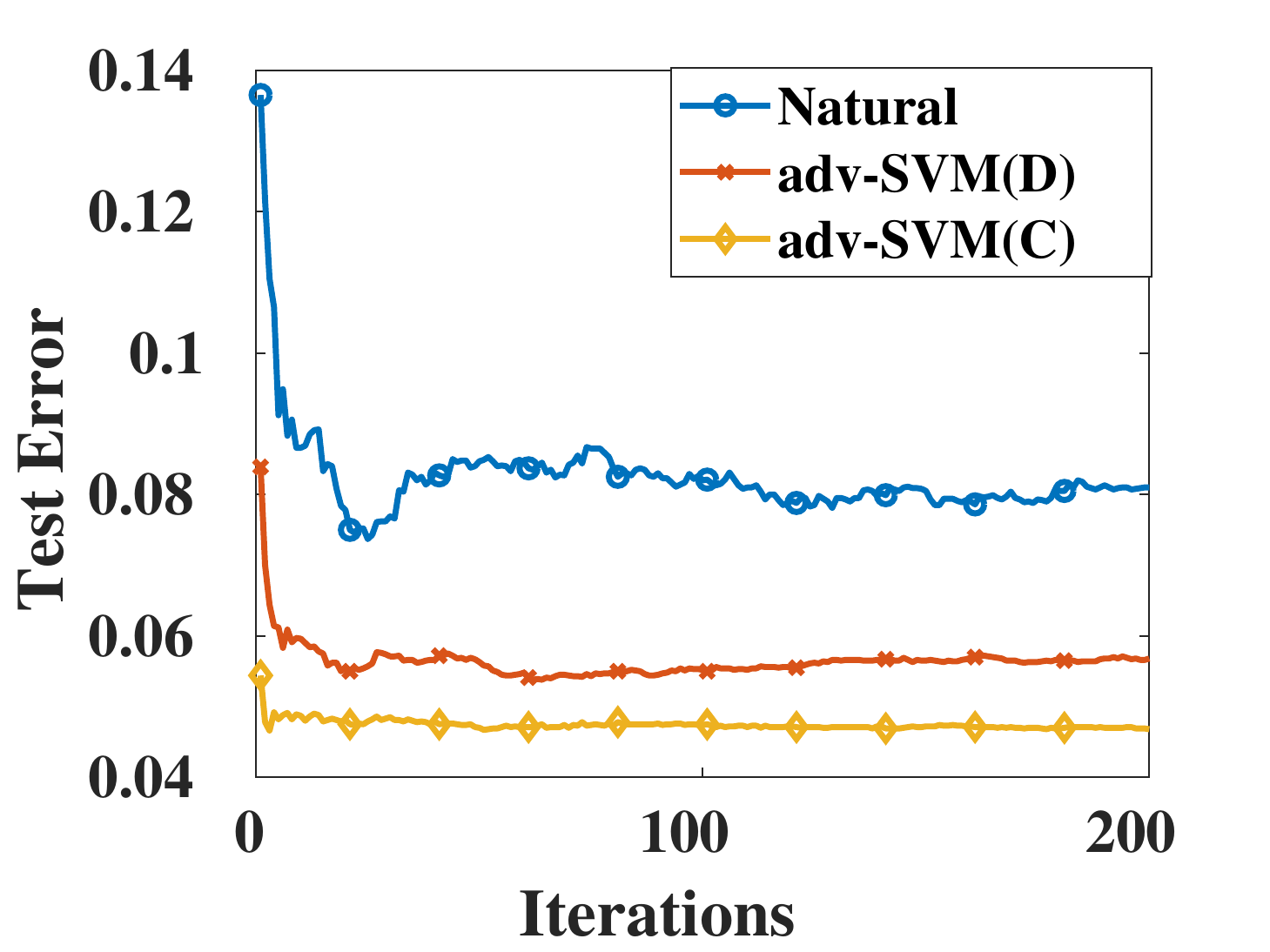}
			\caption{ZOO}\label{fig:mnist8m_0v4_zoo}
		\end{subfigure}
		\caption{Test error vs. iterations of different models on four attack methods on CIFAR10 automobile vs. truck (Fig. \ref{fig:cifar1v9_FGSM}- \ref{fig:cifar1v9_zoo}) and MNIST8m 0v4 (Fig. \ref{fig:mnist8m_0v4_FGSM}-\ref{fig:mnist8m_0v4_zoo}). (Since adv-linear-SVM cannot run iteratively like adv-SVM, we do not include it here.)
		}
		\label{fig:test error vs. iteration}
	\end{figure*}
	
	\begin{theorem}(\textbf{Convergence in expectation})\label{theorem: 3}
		Set $t\in [T]$, $T\textgreater 0$ and $\epsilon\textgreater 0$, $0\textless\eta\textless 1$, with $\eta=\frac{\epsilon\vartheta}{8\kappa B}$ where $\vartheta\in(0,1]$, we will reach $\mathbb{E}_{\mathcal{D}_t,\omega_t}\left[\left|f_{t+1}(x)-f_*\right|^2\right]\leq\epsilon$ after
		\begin{align}
		T\geq \frac{4\kappa B\log(8\kappa e_1/\epsilon)}{\vartheta\epsilon}
		\end{align}
		iterations, where $B$ and $e_1$ are defined in Lemma \ref{lemma: error due to random data_constant}.
	\end{theorem}

	\begin{table*}[]
		\small
		\centering
		\setlength{\tabcolsep}{1.4mm}
		\linespread{1.1}\selectfont	
		\begin{tabular}{c|c|c|c|c|c|c|c|c|c|c}
			\hline
			\multirow{2}{*}{model} & \multicolumn{2}{c|}{Normal} &
			\multicolumn{2}{c|}{FGSM} & \multicolumn{2}{c|}{PGD} & \multicolumn{2}{c|}{C\&W} & \multicolumn{2}{c}{ZOO} \\ \cline{2-11}
			& acc          & time          & acc         & time        & acc        & time        & acc         & time        & acc        & time        \\ \hline
			Natural               &  \textbf{75.65$\pm$0.36 }           &    18.03            &   65.10$\pm$0.51         &   22.50          &    63.32$\pm$0.47        & 19.63            & 47.30$\pm$0.76             & 23.52            &  52.76$\pm$0.89          &  21.23        \\ \hline
			adv-linear-SVM &73.65$\pm$0.58 &671.23 &70.00$\pm$0.52 &592.93 &60.39$\pm$0.75 &665.73 &48.82$\pm$0.93 &594.90 &50.00$\pm$0.24 &606.98 \\ \hline
			adv-SVM(D)             &  74.88$\pm$0.46           &  19.67             &   73.21$\pm$0.49          &  21.23        &   64.93$\pm$0.56         &   21.52          &      50.64$\pm$0.92       &    18.69         &   53.46$\pm$0.81         &   19.44          \\ \hline
			adv-SVM(C)             & 75.22$\pm$0.41            & 18.73              & \textbf{74.65$\pm$0.53}            & 18.33           & \textbf{66.37$\pm$0.59}           & 20.44           & \textbf{51.76$\pm$0.81}            & 19.28           &  \textbf{54.90$\pm$0.94}          & 19.68           \\ \hline
		\end{tabular}
		\caption{Accuracy ($\%$) and running time (min) on CIFAR10 automobile vs. truck against different attacks.}\label{table: table2}
	\end{table*}

	\begin{table*}[]
		\small
		\centering
		\setlength{\tabcolsep}{1mm}
		\linespread{1.1}\selectfont
		\begin{tabular}{c|c|c|c|c|c|c|c|c|c|c}
			\hline
			\multirow{2}{*}{model} & \multicolumn{2}{c|}{Normal} & \multicolumn{2}{c|}{FGSM} & \multicolumn{2}{c|}{PGD} & \multicolumn{2}{c|}{C\&W} & \multicolumn{2}{c}{ZOO} \\ \cline{2-11}
			& acc          & time          & acc         & time        & acc        & time        & acc         & time        & acc        & time        \\ \hline
			Natural                & \textbf{99.48$\pm$0.08}           & 50.11           &98.82$\pm$0.24            & 47.72         &97.82$\pm$0.31            &58.14         &89.42$\pm$0.37            & 49.18         & 91.90$\pm$0.48          & 39.50          \\ \hline
			adv-linear-SVM &98.39$\pm$0.37 &2472.93 &97.61$\pm$0.17 &2624.47 &98.81$\pm$0.34 &2795.34 &87.26$\pm$0.56 &2485.43 &88.39$\pm$0.75 &2602.25 \\ \hline
			adv-SVM(D)             & 99.47$\pm$0.08            &50.32               & 99.42$\pm$ 0.12           &  53.94        & 99.09$\pm$0.42        &53.87           &93.87$\pm$0.44            &50.07           & 94.32$\pm$0.51          & 55.79          \\ \hline
			adv-SVM(C)             &99.46$\pm$0.10           &  56.62         &\textbf{99.45$\pm$0.15}            & 50.31          &\textbf{\textbf{99.20$\pm$0.33 }}          & 53.46         & \textbf{95.46$\pm$0.49 }        & 49.74          & \textbf{95.32$\pm$0.56}      & 52.33          \\ \hline
		\end{tabular}
		\caption{Accuracy ($\%$) and running time (min) on MNIST8m 0 vs. 4 against different attacks.}\label{table: table3}
	\end{table*}

	\begin{remark}
		Based on Theorem \ref{theorem: 3},  $f_{t+1}(x)$ will converge to the optimal solution $f_*$ at a rate near $O(1/t)$ if eliminating the $\log(1/\epsilon)$ factor. This rate is nearly the same as the one of diminishing stepsize, even though the stepsize of our algorithm keeps constant. 
	\end{remark}
	\section{Experiments}
	
	In this section, we will accomplish comprehensive experiments to show the effectiveness and efficiency of adv-SVM.
	\subsection{Experimental Setup}
	Models compared in experiments include \textbf{Natural:} normal DSG algorithm \cite{dai2014scalable}; \textbf{adv-linear-SVM}: adversarial training of linear SVM proposed by Zhou et al. \shortcite{zhou2012adversarial}; \textbf{adv-SVM(C)}: our proposed adversarial training algorithm with constant stepsize; \textbf{adv-SVM(D)}: our proposed adversarial training algorithm with diminishing stepsize. 
	
	The four attack methods of constructing adversarial samples we applied cover both white-box and black-box attacks and are already introduced in the section of related work. For FGSM and PGD, the maximum perturbation $\epsilon$ is set as $8/255$ and the step size for PGD is $\epsilon/4$. We use the $L_2$ version of C$\&$W to generate adversarial examples. For ZOO, we use the ZOO-ADAM algorithm and set the step size $\eta=0.01$, ADAM parameters $\beta_1=0.9$, $\beta_2=0.999$.

	\textbf{Implementation.} We perform experiments on Intel Xeon E5-2696 machine with 48GB RAM. It has been mentioned that our model is implemented\footnote{The DSG code is available at  \url{https://github.com/zixu1986/Doubly_Stochastic_
			Gradients}.} based on the DSG framework \cite{dai2014scalable}. For the sake of efficiency, in the experiment, we use a mini-batch setting. The random features used in DSG are sampled according to pseudo-random number generators. 
	RBF kernel is used for natural DSG and adv-SVM algorithms, the number of random features is set as $2^{10}$ and the batch size is 500. 5-fold cross validation is used to choose the optimal hyper-parameters (the regularization parameter $C$ and the step size $\gamma$). The parameters $C$ and $\gamma$ are searched in the region $\{(C,\gamma)|-3\leq\log_{2}C\leq3\; ,-3\leq\log_{2}\gamma\leq3\}$. For algorithm adv-linear-SVM, we use its free-range training model and set the hyper-parameter $C_f$ as 0.1 according to their analysis. This algorithm is implemented in CVX$-$a package for specifying and solving convex programs \cite{cvx}. The stop criterion for all experiments is one pass over each entire dataset. All results are the average of 10 trials.
	
	\textbf{Datasets.}
	We evaluate the robustness of adv-SVM on two well-known datasets, MNIST \cite{1998Gradient} and CIFAR10 \cite{Krizhevsky09}. Since we focus on binary classification of kernel SVM, here we select two similar classes from the datasets respectively. Each pixel value of the data is normalized into $[0,1]^d$ via dividing its value by 255. Table \ref{tab:dataset} summarizes the 6 datasets used in our experiments. 
	Due to the page limit, we only show the results of CIFAR10 automobile vs. truck and MNIST8m 0 vs. 4 here. The results of other datasets are provided in the appendix .
	
	\begin{table}[htbp]
		\small
		\centering
		\linespread{1.2}\selectfont
		\begin{tabular}{cccc}
			\hline
			\textbf{Dataset} & \textbf{Features} & \textbf{Sample size} \\
			\hline
			MNIST 1 vs. 7   &   784   &  15,170  \\
			MNIST 8 vs. 9 &      784        &  13,783 \\
			CIFAR10 automobile vs. truck&  3,072& 12,000\\
			CIFAR10 dog vs. horse  &       3,072      &  12,000 \\
			MNIST 8M 0 vs. 4 &784 & 200,000 \\
			MNIST 8M 6 vs. 8 &784 & 200,000 \\
			\hline
		\end{tabular}
		\caption{Datasets used in the experiments.}
		\label{tab:dataset}
	\end{table}
\vspace*{-10pt}
	\subsection{Experimental Results}
	We explore the defensive capability of our model against PGD attack in terms of the attack steps $K$ (Fig. \ref{fig:CIFAR K-PGD}, \ref{fig:MNIST K-PGD}) and the maximum allowable perturbation $\epsilon$ (Fig. \ref{fig:CIFAR epsilon}, \ref{fig:MNIST epsilon}). For Fig. \ref{fig:CIFAR K-PGD} and \ref{fig:MNIST K-PGD}, the maximum allowable perturbation $\epsilon$ is fixed as $8/255$, for Fig. \ref{fig:CIFAR epsilon} and \ref{fig:MNIST epsilon}, the attack step $K$ is fixed as 10.
	
	It can be seen clearly that PGD attack strengthens with the increase of either $K$ or $\epsilon$. Meanwhile, increasing $\epsilon$ has greater impact on test accuracy than increasing $K$. However, due to the large allowable disturbance range, it increases the risks of the detection of adversarial examples at the same time since these perturbed examples are not so much similar as original examples, which explains the reason why our algorithm has a better defensive capability for large $\epsilon$.
	
	We evaluate robustness of the 4 competing methods against 4 types of attacks introduced earlier plus the clean datasets (Normal). Here the attack strategy for PGD is 10 steps with max perturbation $\epsilon=8/255$. From Table \ref{table: table2} and \ref{table: table3}, we can see that on both datasets, the natural model achieves the best accuracy on normal test images, but it's not robust to adversarial examples. Among four attacks, C$\&$W and ZOO have the strongest ability to trick models. Although PGD and FGSM belong to the same type attack method, PGD has stronger attack ability and is more difficult to defend since it's a multi-step iterative attack method rather than a single-step one. According to the results of adv-linear-SVM, we can see that this algorithm is not only time-consuming in training examples, but also vulnerable to strong attacks like C$\&$W and ZOO, which even gets higher test error than unsecured algorithm (natural DSG). In comparison, our proposed adv-SVM can finish tasks in just a few minutes and can defend both white-box and black-box attacks.
	
	Fig. \ref{fig:test error vs. iteration} shows test error vs. iterations on three models against four attacks. The results indicate that adv-SVM can converge in a fast speed. Moreover, compared with adv-SVM(D), adv-SVM(C)  enjoys a faster convergence rate and lower test error although it may oscillate slightly in the stationary phase, which is consistent with our analysis.
	\vspace*{-4pt}
	\section{Conclusion}
	To alleviate SVMs' fragility to adversarial examples, we propose an adversarial training strategy named as adv-SVM which is applicable to kernel SVM. DSG algorithm is also applied to improve its scalability. Although we use the principle of approximation, the theoretical analysis shows that our algorithm can converge to the optimal solution. Moreover, comprehensive experimental results also reveal its efficiency in adversarial training models and robustness against various attacks.
	\vspace*{-4pt}
	\section*{Acknowledgments}
	B. Gu was partially supported by National Natural Science Foundation of China (No: 62076138), the Qing Lan Project (No.R2020Q04), the National Natural Science Foundation of China (No.62076138),  the Six talent peaks project (No.XYDXX-042) and the 333 Project (No. BRA2017455) in Jiangsu Province.
	
	\appendix
\section{Proof of Theorem 1}
	\begin{theorem}\label{theorem: theorem1}
	With the constraint $\|\Phi(x'_i)-\phi(x_i)\|\leq\epsilon'$, the maximization problem  
	$
	\max_{x'}[1-y_i(w^T\Phi(x'_i)+b)]_+
	$
	is equivalent to the  regularized loss function
	$
	[1-y_iw^T\phi(x_i)+\epsilon'\left\|w\right\|_2-y_ib]_+
	$.
\end{theorem}
\begin{proof}
	Since $\Phi(x)=\phi(x)+\delta_{\phi}$, the constraint can also be write as $\|\delta_{\phi}\|_2\leq\epsilon'$, let $\mathcal{T}=\{\delta_{\phi}\ |\ \|\delta_{\phi}\|_2\leq\epsilon'\}$. We define $v=[1-y_iw^T\phi(x_i)+\epsilon'\left\|w\right\|_2-y_ib]_+$.	
	To prove the theorem, we first prove  $v\leq\max_{x'}[1-y_i(w^T\Phi(x'_i)+b)]_+$, and then prove $\max_{x'}[1-y_i(w^T\Phi(x'_i)+b)]_+\leq v$. In the following, we give the details to prove these two sub-conclusions.
	
	\noindent\textbf{Step 1:} We first prove $v\leq\max_{x'}[1-y_i(w^T\Phi(x'_i)+b)]_+$.
	
	Since, $\mathcal{T}=\{\delta_{\phi}\ |\ \|\delta_{\phi}\|_2\leq\epsilon'\}$,  we define a subset of $\mathcal{T}$ as   $\mathcal{T}'=\{-y_i\epsilon'\frac{w}{\|w\|_2}\}$.
	
	Hence,
	\begin{align}
	\nonumber&\max_{\delta_{\phi}^i\in\mathcal{T}'}[1-y_i(w^T(\phi(x_i)+\delta_{\phi}^i)+b)]_+\\
	\nonumber=&\max_{\delta_{\phi}^i\in\mathcal{T}'}[1-y_iw^T\phi(x_i)-y_iw^T\delta_{\phi}^i-y_ib]_+\\
	\nonumber=&[1-y_iw^T\phi(x_i)+\epsilon'\left\|w\right\|_2-y_ib]_+
	\end{align}
	Since $\mathcal{T}'\subseteq \mathcal{T}$, the first sub-conclusion can be proved.
	
	\noindent\textbf{Step 2:} Next we prove
	$\label{eq: right}
	\max_{x'}[1-y_i(w^T\Phi(x'_i)+b)]_+\leq v
	$.
	\begin{align}
	\nonumber &\max_{\delta_{\phi}^i\in\mathcal{T}}[1-y_i(w^T(\phi(x_i)+\delta_{\phi}^i)+b)]_+\\
	\nonumber =&\max_{\delta_{\phi}^i\in\mathcal{T}}[1-y_iw^T\phi(x_i)-y_iw^T\delta_{\phi}^i-y_ib]_+\\
	\nonumber \leq &\max_{\delta_{\phi}^i\in\mathcal{T}}[1-y_iw^T\phi(x_i)+y_i\|w\|_2\cdot\|\delta_{\phi}^i\|_2-y_ib]_+\\
	\nonumber \leq &[1-y_iw^T\phi(x_i)+\epsilon'\left\|w\right\|_2-y_ib]_+
	\end{align}
	The first inequality is due to the Cauchy-Schwarz ineuqality. The second inequality holds since $\|\delta_{\phi}^i\|_2\leq\epsilon'$. Hence the second sub-conclusion holds. 
	
	\textbf{Step 3:} Combining  these two steps, we have (\ref{eq3}):
	\begin{align}\label{eq3}
	\nonumber&\max_{\|\delta_{\phi}^i\|_2\leq\epsilon}[1-y_i(w^T(\phi(x_i)+\delta_{\phi}^i)+b)]_+\\
	=&[1-y_iw^T\phi(x_i)+\epsilon'\left\|w\right\|_2-y_ib]_+
	\end{align}
\end{proof}

\section{Detailed Proof of Convergence Rate}

\begin{assumption}(Bound of kernel function)
	The kernel function is bounded, i.e., there exists $\kappa\textgreater 0$, such that $k(x,x')\leq \kappa$.
\end{assumption}

\begin{assumption}(Bound of random feature norm)\label{assumption:upper bound}
	There exits $\phi\textgreater 0$, such that $\left|\phi_{\omega}(x)\phi_{\omega}(x')\right|\leq\phi$.
\end{assumption}

\begin{assumption}\label{assumption:upper bound}
	The spectral radius $\rho(f)$ of a function $f(\cdot)$ has a lower bound that $\rho(f)\geq\epsilon'C\geq 0$, where a spectral radius is the maximum modulus of eigenvalues\cite{Gurvits2007On}, i.e., $\rho(f)=\max_{1\leq i\leq \infty}\left\lbrace\left|\sqrt{\lambda_i}\right| \right\rbrace$.
\end{assumption}
\subsection{Convergence Analysis on Diminishing Stepsize}
In this section, we aim to prove that our algorithm with diminishing stepsize converges to the optimal solution at a rate of $O(1/t)$.


\begin{lemma}\textbf{(Error due to random features)}\label{lemma: error due to random features1}
	For any $x\in\mathcal{X}$, 
	\begin{align}
	\nonumber\mathbb{E}_{\mathcal{D}^t,\omega^t}\left[\left|f_{t+1}(x)-h_{t+1}(x)\right|^2 \right] \leq& C^2\sum_{i=1}^{t}\left|a_t^i\right|^2(\kappa+\phi)^2
	\end{align}
\end{lemma}
\begin{proof}

\begin{align}
\nonumber f_{t+1}(\cdot)-h_{t+1}(\cdot)=&\sum_{i=1}^ta_t^i\zeta_i(\cdot)-\sum_{i=1}^ta_t^i\xi_i(\cdot)\\
\nonumber =&\sum_{i=1}^ta_t^i\left[-Cy_i\phi_{\omega(x_i)}\phi_{\omega}(\cdot)+Cy_ik(x_i,\cdot)\right]\\
\nonumber =&Cy_i\sum_{i=1}^ta_t^i[k(x_i,\cdot)-\phi_{\omega}(x_i)\phi_{\omega}(\cdot)]\\
\nonumber \leq&Cy_i\sum_{i=1}^ta_t^i(k(x_i,\cdot)+\phi_{\omega}(x_i)\phi_{\omega}(\cdot))\\
\leq&Cy_i\sum_{i=1}^ta_t^i(\kappa+\phi)
\end{align}

Thus we can get 
\begin{align}
\nonumber \left|f_{t+1}-h_{t+1}\right|^2\leq& C^2y_i^2\sum_{i=1}^t\left|a_t^i\right|^2(\kappa+\phi)^2\\
=&C^2\sum_{i=1}^{t}\left|a_t^i\right|^2(\kappa+\phi)^2    
\end{align}
	
\end{proof}

\begin{lemma}
	Suppose $\gamma_i=\frac{\theta}{i}(1\leq i\leq t,\; \frac{1}{2}<\theta\leq1)$, then we have $\left|a_t^i\right|\leq\frac{\theta}{t}$ and $\sum_{i=1}^t\left|a_t^i\right|^2\leq\frac{\theta^2}{t}$.
\end{lemma}
\begin{proof}
The proof of DSG first gives the upper bound of $A_t^t$ and proves that $a_t^i$ monotonically increasing. 
According to the definition of $\left|a_t^i\right|$, we have

\begin{align}
\nonumber&\left|a_t^{i+1}\right|-\left|a_t^i\right|\\
\nonumber=&\left|-\gamma_{i+1}(1-\gamma_{i+2}(1+\frac{\epsilon'C}{\left\|f_{i+2}\right\|}))\cdots(1-\gamma_t(1+\frac{\epsilon'C}{\left\|f_t\right\|}))\right|\\\nonumber &-\left|-\gamma_i(1-\gamma_{i+1}(1+\frac{\epsilon'C}{\left\|f_{i+1}\right\|}))\cdots(1-\gamma_t(1+\frac{\epsilon'C}{\left\|f_t\right\|}))\right|\\
\nonumber=&\left( \left|\gamma_{i+1}\right|-\left|\gamma_i(1-\gamma_{i+1}(1+\frac{\epsilon'C}{\left\|f_{i+1}\right\|}))\right|\right)\\
\nonumber &\cdot \left|(1-\gamma_{i+2}(1+\frac{\epsilon'C}{\left\|f_{i+2}\right\|}))\cdots(1-\gamma_t(1+\frac{\epsilon'C}{\left\|f_t\right\|}))\right|\\
\nonumber \geq& \frac{\theta}{i+1}\left(1-\left|1-\frac{\theta}{i+1}(1+\frac{\epsilon'C}{\left\|f_{i+1}\right\|})\right| \right)\\
\nonumber &\cdot \left|(1-\gamma_{i+2}(1+\frac{\epsilon'C}{\left\|f_{i+2}\right\|}))\cdots(1-\gamma_t(1+\frac{\epsilon'C}{\left\|f_t\right\|}))\right|
\end{align}
Then according to Assumption \ref{assumption:upper bound} and the theorem that $\rho(f)$ is the lower bound of any matrix norm of $f(\cdot)$ that $\left\|f\right\|\geq\rho(f)$, we can get that $\left\|f\right\|\geq\epsilon'C$, thus we have
\begin{align}  
\nonumber\left|a_t^{i+1}\right|-\left|a_t^i\right|\geq 0
\end{align}
	


In this way, we come to a conclusion that the value of $\left|a_t^i\right|$ is monotonically increasing, since $\left|a_t^t\right|=\frac{\theta}{t}$, we can get that $\left|a_t^i\right|\leq\frac{\theta}{t}$.
\end{proof}

\begin{lemma}\label{lemma: error due to random data1}\textbf{(Error due to random data)} Let $f_*$ be the optimal solution to our target problem, we set $\gamma_t=\frac{\theta}{t}$ with $\theta$ such that $\frac{1}{2}<\theta\leq1$, then we have 
	\begin{align}
	\mathbb{E}_{\mathcal{D}_t,\bm{\omega}_t}[\left\|h_{t+1}-f_*\right\|_2^2]\leq\frac{Q_1^2}{t}
	\end{align}
	where $Q_1=\max\left\lbrace\left\|f_*\right\|_2,\frac{\beta_0+\sqrt{\beta_0^2+4(2\theta-1)\beta}}{2(2\theta-1)} \right\rbrace $, $\beta=C^2\theta^2\left[(\kappa+\epsilon')+\kappa^{1/2}\theta\right]^2$, $\beta_0$ is a constant value and $\beta_0\textgreater 0$,
	
\end{lemma}
\begin{proof}
	
For the sake of simple notations, we first denote the following two different gradient terms, which are
\begin{align}
\nonumber g_t&=p_t+h_t=-Cy_tk(x_t,\cdot)+C\epsilon'\frac{f(\cdot)}{\left\|f\right\|}+h_t\\
\nonumber \overline{g}_t&=\mathbb{E}_{\mathcal{D}_t}[g_t]=\mathbb{E}_{\mathcal{D}_t}[-Cy_tk(x_t,\cdot)+C\epsilon'\frac{f(\cdot)}{\left\|f\right\|}]+h_t
\end{align}
Note that by our previous definition, we have $h_{t+1}=h_t-\gamma_tg_t$, $\forall t\geq 1$.

Denote $A_t=\left\|h_t-f_*\right\|^2$. Then we have
\begin{align}
\nonumber A_{t+1}=&\left\|h_t-f_*-\gamma_tg_t\right\|^2\\
\nonumber =&A_t+\gamma_t^2\left\|g_t\right\|^2-2\gamma_t\langle h_t-f_*,g_t\rangle\\
=&A_t+\gamma_t^2\left\|g_t\right\|^2-2\gamma_t\langle h_t-f_*,\overline{g}_t\rangle\\
\nonumber &+2\gamma_t\langle h_t-f_*,\overline{g}_t-g_t\rangle
\end{align}
Because of the strongly convexity of the objective function and optimality condition, we have
\begin{align}
\langle h_t-f_*,\overline{g}_t\rangle\geq\left\|h_t-f_*\right\|^2
\end{align}
Hence, we have
\begin{align}
A_{t+1}\leq(1-2\gamma_t)A_t+\gamma_t^2\left\|g_t\right\|^2+2\gamma_t\langle h_t-f_*,\overline{g}_t-g_t\rangle
\end{align}
Let us denote $\mathcal{M}_t=\left\|g_t\right\|^2$, $\mathcal{N}_t=\langle h_t-f_*,\overline{g}_t-g_t\rangle$. We first show that $\mathcal{M}_t$ and $\mathcal{N}_t$ are bounded.
\begin{enumerate}
	\item $\mathcal{M}_t\leq C\left[(\kappa+\epsilon')+\kappa^{1/2}c_t\right]^2$, where $c_t=\sqrt{\sum_{i=1}^{t-1}\sum_{j=1}^{t-1}\left|a_{t-1}^i\right|\cdot\left|a_{t-1}^j\right|}$
	\begin{align}
	\nonumber \mathcal{M}_t=\left\|g_t\right\|^2=\left\|p_t+h_t\right\|^2\leq (\left\|p_t\right\|+\left\|h_t\right\|)^2
	\end{align}
	Since 
	\begin{align}
	\nonumber \left\|p_t\right\|=&\left\|-Cy_tk(x_t,\cdot)+C\epsilon'\frac{f(\cdot)}{\left\|f\right\|}\right\|\\
	\nonumber\leq & C\left\|y_t\kappa\right\|+C\epsilon'\left\|\frac{f(\cdot)}{\left\|f\right\|}\right\|\\
	\nonumber =&C(\kappa+\epsilon')
	\end{align}
	and 
	\begin{align}
	\nonumber \left\|h_t\right\|^2=&\sum_{i=1}^{t-1}\sum_{j=1}^{t-1}a_{t-1}^ia_{t-1}^jC^2y_iy_jk(x_i,x_j)\\
	\nonumber\leq&\kappa C^2\sum_{i=1}^{t-1}\sum_{j=1}^{t-1}\left|a_{t-1}^i\right|\cdot\left|a_{t-1}^j\right|,
	\end{align}
	we can get
	\begin{align}
	\mathcal{M}_t\leq C^2\left[(\kappa+\epsilon')+\kappa^{1/2}\sqrt{\sum_{i=1}^{t-1}\sum_{j=1}^{t-1}\left|a_{t-1}^i\right|\cdot\left|a_{t-1}^j\right|}\right]^2.
	\end{align}
	
	\item $\mathbb{E}_{\mathcal{D}_t,\mathcal{\omega}_t}[\mathcal{N}_t]=0$;
	
	This is because $\mathcal{N}_t=\langle h_t-f_*,\overline{g}_t-g_t\rangle$,
	\begin{align}
	\nonumber\mathbb{E}_{\mathcal{D}_t,\mathcal{\omega}_t}[\mathcal{N}_t]=&\mathbb{E}_{\mathcal{D}_{t-1},\mathcal{\omega}_t}[\mathbb{E}_{\mathcal{D}_t}[\langle h_t-f_*,\overline{g}_t-g_t\rangle|\mathcal{D}^{t-1},\omega^t]]\\
	\nonumber =&\mathbb{E}_{\mathcal{D}_t,\mathcal{\omega}_t}[\langle h_t-f_*,\mathbb{E}_{\mathcal{D}_t}[\overline{g}_t-g_t]\rangle]\\
	\nonumber =&0.
	\end{align}
	
\end{enumerate}

Let us denote $e_t=\mathbb{E}_{\mathcal{D}^{t-1},\omega^{t-1}}[A_t]$, given the above bounds, we arrive at the following recursion,
\begin{align}
e_{t+1}\leq(1-2\gamma_t)e_t+C^2\gamma_t^2\left[(\kappa+\epsilon')+\kappa^{1/2}c_t\right]^2
\end{align}
When $\gamma_t=\theta/t$ and $\left|a_t^i\right|\leq\frac{\theta}{t}$, where $\theta$ is a constant value.
\begin{align}\label{e_t+1}
e_{t+1}\leq (1-2\frac{\theta}{t})e_t+C^2\frac{\theta^2}{t^2}\left[(\kappa+\epsilon')+\kappa^{1/2}\theta\right]^2
\end{align}
Equation (\ref{e_t+1}) can be further written as
\begin{align}
e_{t+1}\leq(1-\frac{2\theta}{t})e_t+\frac{\beta}{t^2}
\end{align}
where $\beta=C^2\theta^2\left[(\kappa+\epsilon')+\kappa^{1/2}\theta\right]^2$.

Invoking Lemma \ref{recursion bounds} with $\eta=2\theta\textgreater 1$, we obtain
\begin{align}
e_t\leq\frac{Q_1^2}{t}
\end{align}
where $Q_1=\max\left\lbrace\left\|f_*\right\|,\frac{\beta_0+\sqrt{\beta_0^2+4(2\theta-1)\beta}}{2(2\theta-1)} \right\rbrace $ and $\beta=C^2\theta^2\left[(
\kappa+\epsilon')+\kappa^{1/2}\theta\right]^2$.
\end{proof}

\begin{lemma}\label{recursion bounds}
	Suppose the sequence $\left\lbrace\Gamma_t\right\rbrace _{t=1}^{\infty}$ satisfies $\Gamma_1\geq 0$, and $\forall t\geq 1$
	\begin{align}
	\Gamma_{t+1}\leq (1-\frac{\eta}{t})\Gamma_t+\frac{\beta}{t^2}
	\end{align}
	where $\eta \textgreater 1$, $\beta_0,\;\beta\textgreater 0$. Then $\forall t\geq 1$,
	\begin{align}
	\nonumber\Gamma_t\leq \frac{R}{t},
	\end{align}
	 where $R=\max\left\lbrace \Gamma_1, R_0^2 \right\rbrace, \;\;R_0=\frac{\beta_0+\sqrt{\beta_0^2+4(\eta-1)\beta}}{2(\eta-1)}$.

\end{lemma}

\begin{proof}

When $t=1$, it always holds true by the definition of $R$. Assume the conclusion holds true for $t$ with $t\geq 1$, i.e., $\Gamma_t\leq\frac{R}{t}$, then we have
\begin{align}
\nonumber\Gamma_{t+1}\leq&(1-\frac{\eta}{t})\Gamma_t+\frac{\beta}{t^2}\\
\nonumber\leq& \frac{R}{t}-\frac{\eta R-\beta}{t^2}\\
\nonumber=&\frac{R}{t+1}+\frac{R}{t(t+1)}-\frac{\eta R-\beta}{t^2}\\
\nonumber \leq& \frac{R}{t+1}-\frac{1}{t^2}\left[-R+\eta R-\beta\right]\\
\nonumber \leq&\frac{R}{t+1}
\end{align}
where the last step can be verified as follows:

Assume $\beta_0$ is a constant value and $\beta_0\textgreater 0$, then we have
\begin{align}
\nonumber(\eta-1)R-\beta\textgreater&(\eta-1)R-\beta_0\sqrt{R}-\beta\\
\nonumber =&(\eta-1)(R-\frac{\beta_0\sqrt{R}}{\eta-1}-\frac{\beta}{\eta-1})\\
\nonumber =&(\eta-1)\left[\sqrt{R}-\frac{\beta_0}{2(\eta-1)}\right]^2-\frac{\beta_0^2}{4(\eta-1)}-\beta\\
\nonumber \geq&(\eta-1)\left[R_0-\frac{\beta_0}{2(\eta-1)}\right]^2-\frac{\beta_0^2}{4(\eta-1)}-\beta=0
\end{align}
where $R_0=\frac{\beta_0+\sqrt{\beta_0^2+4(\eta-1)\beta}}{2(\eta-1)}$.	
\end{proof}

\begin{theorem}(\textbf{Convergence in expectation})
	When $\gamma_t=\frac{\theta}{t}$ with $\theta\in(\frac{1}{2},1]$, $\forall x\in\mathcal{X}$,
	\begin{align}
	\mathbb{E}_{\mathcal{D}^t,\omega^t}\left[ \left|f_{t+1}(x)-f_*\right|^2\right]\leq\frac{2Q_0^2}{t}+\frac{2\kappa Q_1^2}{t} 
	\end{align}
	where $Q_0=C\theta(\kappa+\phi)$.
\end{theorem}
\begin{proof}

Combining Lemma \ref{lemma: error due to random features1} and Lemma \ref{lemma: error due to random data1} together, we have that
\begin{align}
\nonumber&\mathbb{E}_{\mathcal{D}^t,\omega^t}\left[ \left|f_{t+1}(x)-f_*\right|^2\right]\\
\nonumber \leq&2\mathbb{E}_{\mathcal{D}^t,\omega^t}\left[ \left|f_{t+1}(x)-h_{t+1}(x)\right|^2\right]\\
\nonumber &+2\kappa\mathbb{E}_{\mathcal{D}^t,\omega^t}\left[ \left\|h_{t+1}(x)-f_*\right\|^2\right]\\
\nonumber \leq& \frac{2}{t}C^2\theta^2(\kappa+\phi)^2+2\kappa\frac{Q_1^2}{t}\\
\nonumber =&2\frac{Q_0^2}{t}+2\kappa\frac{Q_1^2}{t}
\end{align}
	
\end{proof}

\subsection{Convergence Analysis on Constant Stepsize}
In this section, we prove that ouralgorithm with constant stepsize converges to the optimal solution at a rate of $O(1/t)$.
\begin{lemma}\textbf{(Error due to random features)}\label{lemma: error due to random feature_constant}
	For any $x\in\mathcal{X}$, 
	\begin{align}
	\mathbb{E}_{\mathcal{D}^t,\omega^t}\left[f_{t+1}(x)-h_{t+1}(x)\left|\right|^2 \right] \leq& C^2\sum_{i=1}^{t}\left|a_t^i\right|^2(\kappa+\phi)^2
	\end{align}
\end{lemma}
\begin{proof}

\begin{align}
\nonumber f_{t+1}(\cdot)-h_{t+1}(\cdot)=&\sum_{i=1}^ta_t^i\zeta_i(\cdot)-\sum_{i=1}^ta_t^i\xi_i(\cdot)\\
\nonumber =&\sum_{i=1}^ta_t^i\left[-Cy_i\phi_{\omega(x_i)}\phi_{\omega}(\cdot)+Cy_ik(x_i,\cdot)\right]\\
\nonumber =&Cy_i\sum_{i=1}^ta_t^i[k(x_i,\cdot)-\phi_{\omega}(x_i)\phi_{\omega}(\cdot)]\\
\nonumber \leq&Cy_i\sum_{i=1}^ta_t^i(k(x_i,\cdot)+\phi_{\omega}(x_i)\phi_{\omega}(\cdot))\\
\leq&Cy_i\sum_{i=1}^ta_t^i(\kappa+\phi)
\end{align}

Thus we can get 
\begin{align}
\nonumber \left|f_{t+1}-h_{t+1}\right|^2\leq& C^2y_i^2\sum_{i=1}^t\left|a_t^i\right|^2(\kappa+\phi)^2\\
=&C^2\sum_{i=1}^{t}\left|a_t^i\right|^2(\kappa+\phi)^2    
\end{align}
	
\end{proof}

\begin{lemma}\label{Lemma 4}
	Suppose $\eta \in(0,1)$, then we have $\left|a_t^i\right|\leq \eta$ and $\sum_{i=1}^{t}(a_t^i)^2\leq\frac{\eta}{c}$.
\end{lemma}

\begin{proof}
	Obviously, $\left|a_t^i\right|=\eta\prod_{j=i+1}^{t}|1-\eta (1+\frac{\epsilon'C}{\left\|f_j\right\|})| \leq\eta$ where $\left\|f_j\right\|\geq\epsilon'C$ and $\eta\in(0,1)$, then according to the summation formula of geometric progression, $\sum_{i=1}^{t}\left|a_t^i\right|\leq\frac{1}{c}$ and $c\in(0,1)$. 
	
	Then $\sum_{i=1}^{t}\left|a_t^i\right|^2=\left|a_t^1\right|^2+\left|a_t^2\right|^2+\cdots+\left|a_t^t\right|^2\leq\max\left\lbrace \left|a_t^i\right|\right\rbrace_{i=1}^t \sum_{i=1}^{t}\left|a_t^i\right|\leq\eta \sum_{i=1}^{t}\left|a_t^i\right|$.
	Consequently, $\sum_{i=1}^{t}\left|a_t^i\right|^2\leq\frac{\eta}{c}$.
\end{proof}

\begin{lemma} \label{lemma: error due to random data_constant}(\textbf{Error due to random data})
	Let $f_*$ be the optimal solution to our target problem, set $t\in[T]$ and $\eta\in(0,1)$, with $\eta=\frac{\epsilon\vartheta}{2B}$ for $\vartheta\in(0,1]$, we will reach $\mathbb{E}_{\mathcal{D}^t,\omega^t,\omega'^t}[\left\|h_{t+1}-f_*\right\|_{\mathcal{H}}^2]\leq\epsilon$ after
	\begin{align}
	&T\geq\frac{B\log(2e_1/\epsilon)}{\vartheta\epsilon}
	\end{align}
	iterations, where $B=\frac{1}{2}C\left[(\kappa+\epsilon')+\kappa^{1/2}\frac{1}{c}\right]^2$ and $e_1=\mathbb{E}_{\mathcal{D}_1,\mathcal{\omega}_1}[\left\|h_1-f_*\right\|_\mathcal{H}^2]$.
\end{lemma}

\begin{proof}
	We first denote two different gradient terms, which are
	\begin{align}
	\nonumber g_t&=p_t+h_t=-Cy_tk(x_t,\cdot)+C\epsilon'\frac{f(\cdot)}{\left\|f\right\|}+h_t\\
	\nonumber \overline{g}_t&=\mathbb{E}_{\mathcal{D}_t}[g_t]=\mathbb{E}_{\mathcal{D}_t}[-Cy_tk(x_t,\cdot)+C\epsilon'\frac{f(\cdot)}{\left\|f\right\|}]+h_t
	\end{align}
	Note that by our previous definition, we have $h_{t+1}=h_t-\eta g_t$, $\forall t\geq 1$.
	
	Denote $A_t=\left\|h_t-f_*\right\|^2$. Then we have
	\begin{align}
	\nonumber A_{t+1}=&\left\|h_t-f_*-\gamma_t g_t\right\|^2\\
	\nonumber =&A_t+\gamma_t^2\left\|g_t\right\|^2-2\gamma_t\langle h_t-f_*,g_t\rangle\\
	=&A_t+\eta^2\left\|g_t\right\|^2-2\eta\langle h_t-f_*,\overline{g}_t\rangle\\\nonumber&+2\eta\langle h_t-f_*,\overline{g}_t-g_t\rangle
	\end{align}
	Because of the strongly convexity of the objective function and optimality condition, we have
	\begin{align}
	\langle h_t-f_*,\overline{g}_t\rangle\geq\left\|h_t-f_*\right\|^2
	\end{align}
	Hence, we have
	\begin{align}
	A_{t+1}\leq(1-2\eta)A_t+\eta^2\left\|g_t\right\|^2+2\eta\langle h_t-f_*,\overline{g}_t-g_t\rangle
	\end{align}
	Let us denote $\mathcal{M}_t=\left\|g_t\right\|^2$, $\mathcal{N}_t=\langle h_t-f_*,\overline{g}_t-g_t\rangle$. We first show that $\mathcal{M}_t$ and $\mathcal{N}_t$ are bounded, similar to the process of Lemma \ref{lemma: error due to random data1}, we can get that $\mathcal{M}_t\leq C\left[(\kappa+\epsilon')+\kappa^{1/2}\frac{1}{c}\right]^2$ and $\mathbb{E}_{\mathcal{D}^t,\omega^t}[\mathcal{N}_t]=0$.
	We denote that $e_t=\mathbb{E}_{\mathcal{D}^{t-1},\omega^{t-1}}[A_t]$, given the above bounds, we arrive at the following recursion,
	\begin{align}\label{e_t}
	\nonumber e_{t+1}&\leq(1-2\eta)e_t+C\eta^2\left[(\kappa+\epsilon')+\kappa^{1/2}\frac{1}{c}\right]^2\\
	&=\beta_1e_t+\beta_2
	\end{align}
	where $\beta_1=1-2\eta$ and $\beta_2=C\eta^2\left[(\kappa+\epsilon')+\kappa^{1/2}\frac{1}{c}\right]^2$.
	
	Then consider that when $t\rightarrow \infty$, we have
	\begin{align}\label{e_infty}
	e_{\infty}=\beta_1e_{\infty}+\beta_2
	\end{align}
	and the solution of the above recursion is
	\begin{align}\label{e_infty2}
	e_{\infty}=\frac{\beta_2}{1-\beta_1}=\eta\frac{1}{2}C\left[(\kappa+\epsilon')+\kappa^{1/2}\frac{1}{c}\right]^2=\eta B,
	\end{align}
	where $B=\frac{1}{2}C\left[(\kappa+\epsilon')+\kappa^{1/2}\frac{1}{c}\right]^2$.
	
	We use Eq. (\ref{e_t}) minus Eq. (\ref{e_infty}), then we get
	\begin{align}\label{e_t2}
	\nonumber  e_{t+1}&\leq\beta_1(e_t-e_{\infty})+e_{\infty}\\
	&= (1-2\eta)(e_t-e_{\infty})+e_{\infty}
	\end{align}
	
	We can easily apply Eq. (5.1) in \cite{recht2011hogwild:} to Eq. (\ref{e_t2}), and Eq. (\ref{e_infty2}) satifies $a_{\infty}(\gamma)\leq\gamma B$.  Particularly, unwrapping (\ref{e_t2}) we have 
	\begin{align}\label{e_t3}
	e_{t+1}\leq(1-2\eta)^t(e_1-e_\infty)+e_\infty.
	\end{align}
	
	Suppose we want this quantity (\ref{e_t3}) to be less than $\epsilon$. Similarly, we let both terms are less than $\epsilon/2$, then for the second term, we have
	\begin{align}
	\nonumber e_{\infty}=\eta B\leq\frac{\epsilon}{2},
	\end{align}
	then we get
	\begin{align}\label{eta}
	\eta\leq\frac{\epsilon}{2B}=\frac{\epsilon}{C\left[(\kappa+\epsilon')+\kappa^{1/2}\frac{1}{u}\right]^2}.
	\end{align}
	For the first term, we need
	\begin{align}
	\nonumber (1-2\eta)^te_1\leq\frac{\epsilon}{2}
	\end{align}
	which holds if
	\begin{align}\label{t1}
	t\geq\frac{log(2e_1/\epsilon)}{2\eta}.
	\end{align}
	According to Eq. (\ref{eta}), we should pick $\eta=\frac{\epsilon\vartheta}{2B}$ where $\vartheta\in(0,1]$. Combining this with Eq. (\ref{t1}), after
	\begin{align}
	T\geq t\geq\frac{B\log(2e_1/\epsilon)}{\vartheta\epsilon}
	\end{align}
	iterations, we will have $e_\infty\leq\epsilon$ and that give us a $1/t$ convergence rate, if eliminating the $\log(1/\epsilon)$ factor.
\end{proof}

According to Lemma \ref{lemma: error due to random feature_constant} and \ref{lemma: error due to random data_constant}, we obtain the final results on convergence in expection:
\begin{theorem}(\textbf{Convergence in expectation})
	Set $t\in [T]$, $T\textgreater 0$ and $\epsilon\textgreater 0$, $0\textless\eta\textless 1$, with $\eta=\frac{\epsilon\vartheta}{8\kappa B}$ where $\vartheta\in(0,1]$, we will reach $\mathbb{E}_{\mathcal{D}_t,\omega_t}\left[\left|f_{t+1}(x)-f_*\right|^2\right]\leq\epsilon$ after 
	\begin{align}
	T\geq \frac{4\kappa B\log(8\kappa e_1/\epsilon)}{\vartheta\epsilon}
	\end{align}
	iterations, where $B$ and $e_1$ are defined in Lemma \ref{lemma: error due to random data_constant}.
\end{theorem}

\begin{proof}
	Combining Lemma \ref{lemma: error due to random feature_constant} and Lemma \ref{lemma: error due to random data_constant} together, we have that
	\begin{align}
	\nonumber\mathbb{E}_{\mathcal{D}^t,\omega^t,\omega'^t}&[\left|f_{t+1}(x)-f_*\right|^2]\\
	\nonumber\leq&2\mathbb{E}_{\mathcal{D}^t,\omega^t,\omega'^t}[\left|f_{t+1}(x)-h_{t+1}(x)\right|^2]\\
	\nonumber&+2\kappa\mathbb{E}_{\mathcal{D}^t,\omega^t,\omega'^t}[\left\|h_{t+1}(x)-f_*\right\|_{\mathcal{H}}^2]\\
	\nonumber\leq\epsilon
	\end{align}
	Again, we let both terms to be less than $\epsilon/2$. For the second term, we can directly derive from Lemma \ref{lemma: error due to random data_constant}. As for the first term, we let $2\mathbb{E}_{\mathcal{D}^t,\omega^t,\omega'^t}[\left|f_{t+1}(x)-h_{t+1}(x)\right|^2]\leq\frac{\epsilon}{2}$, then we can get a upper bound of $\eta$ that $\eta\leq\frac{\epsilon c}{4C^2(\kappa+\phi)^2}$. Then we can obtain the above theorem.
\end{proof}

\section{Experiments}
The experimental results on the other four datasets are as follows, which include MNIST 1vs.7, MNIST 8 vs. 9, CIFAR10 dog vs. horse and MNIST8M 6vs.8. 

\begin{figure*}[htb]
	\centering
	\begin{subfigure}[b]{0.24\textwidth}
		\includegraphics[width=1.85in]{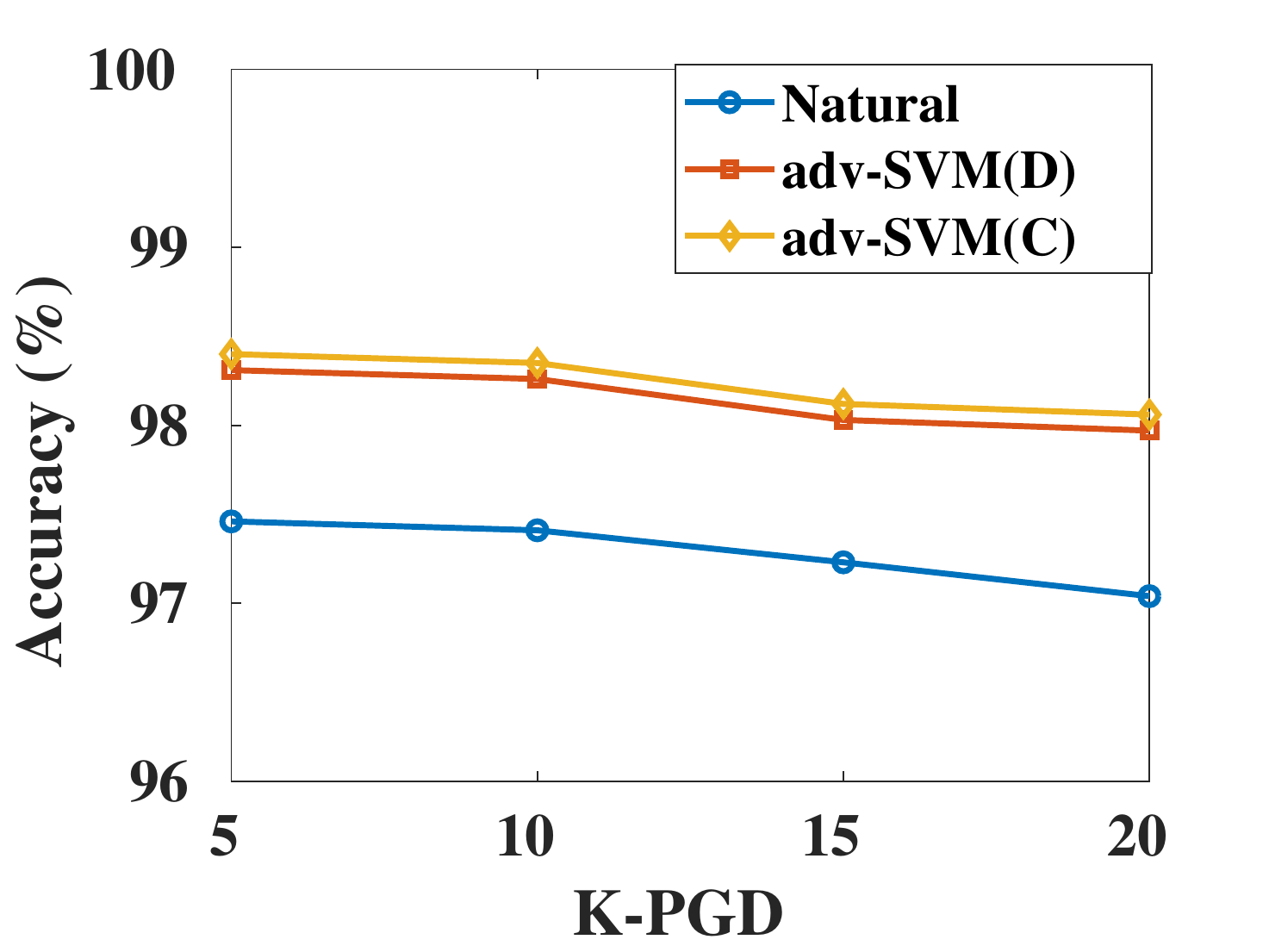}
		\caption{MNIST 1 vs. 7}\label{fig:MNIST1v7 K-PGD}
	\end{subfigure}
	\begin{subfigure}[b]{0.24\textwidth}
		\includegraphics[width=1.85in]{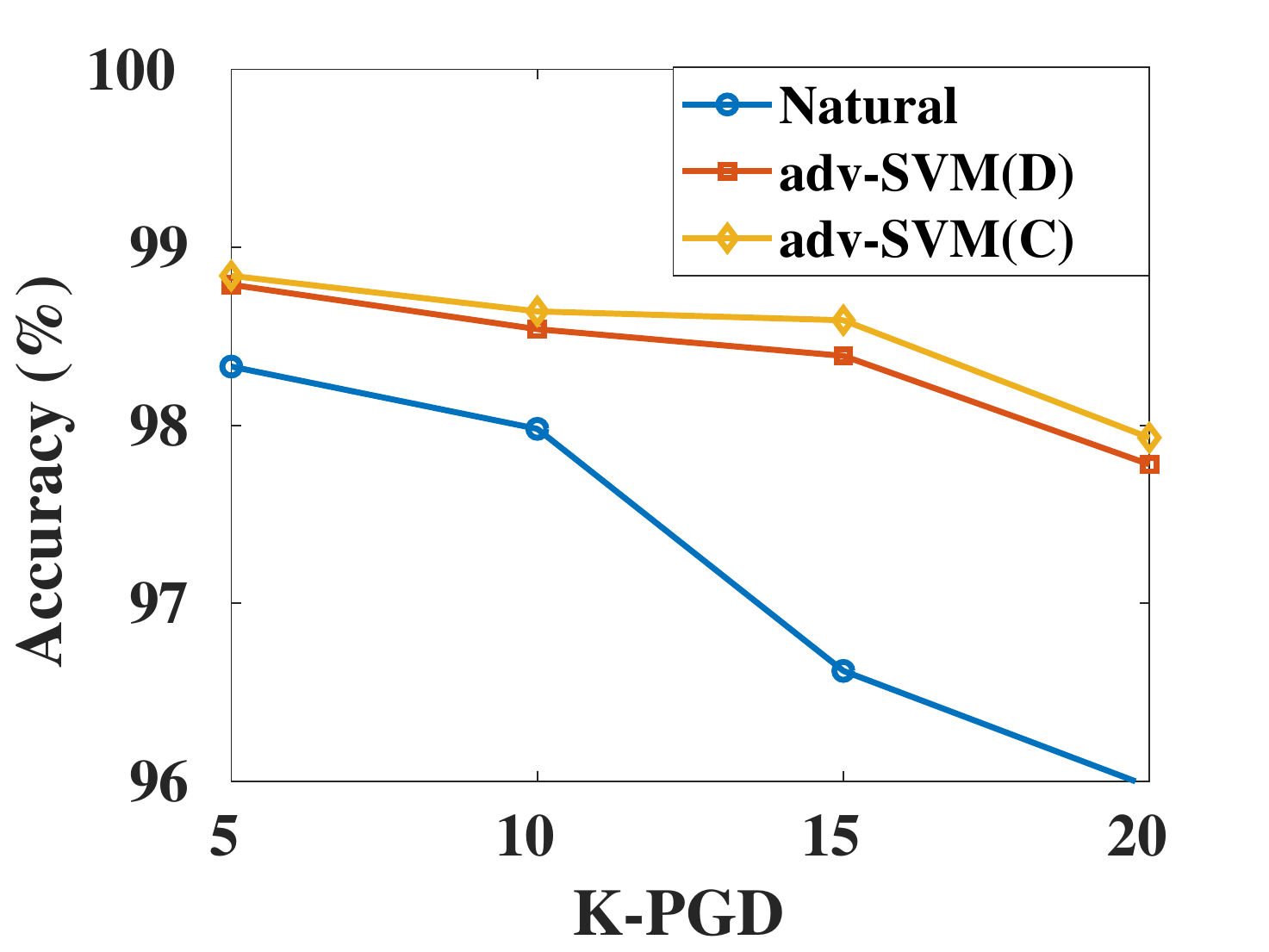}
		\caption{MNIST 8 vs. 9}\label{fig:MNIST8v9 K-PGD}
	\end{subfigure}
	\begin{subfigure}[b]{0.24\textwidth}
		\includegraphics[width=1.85in]{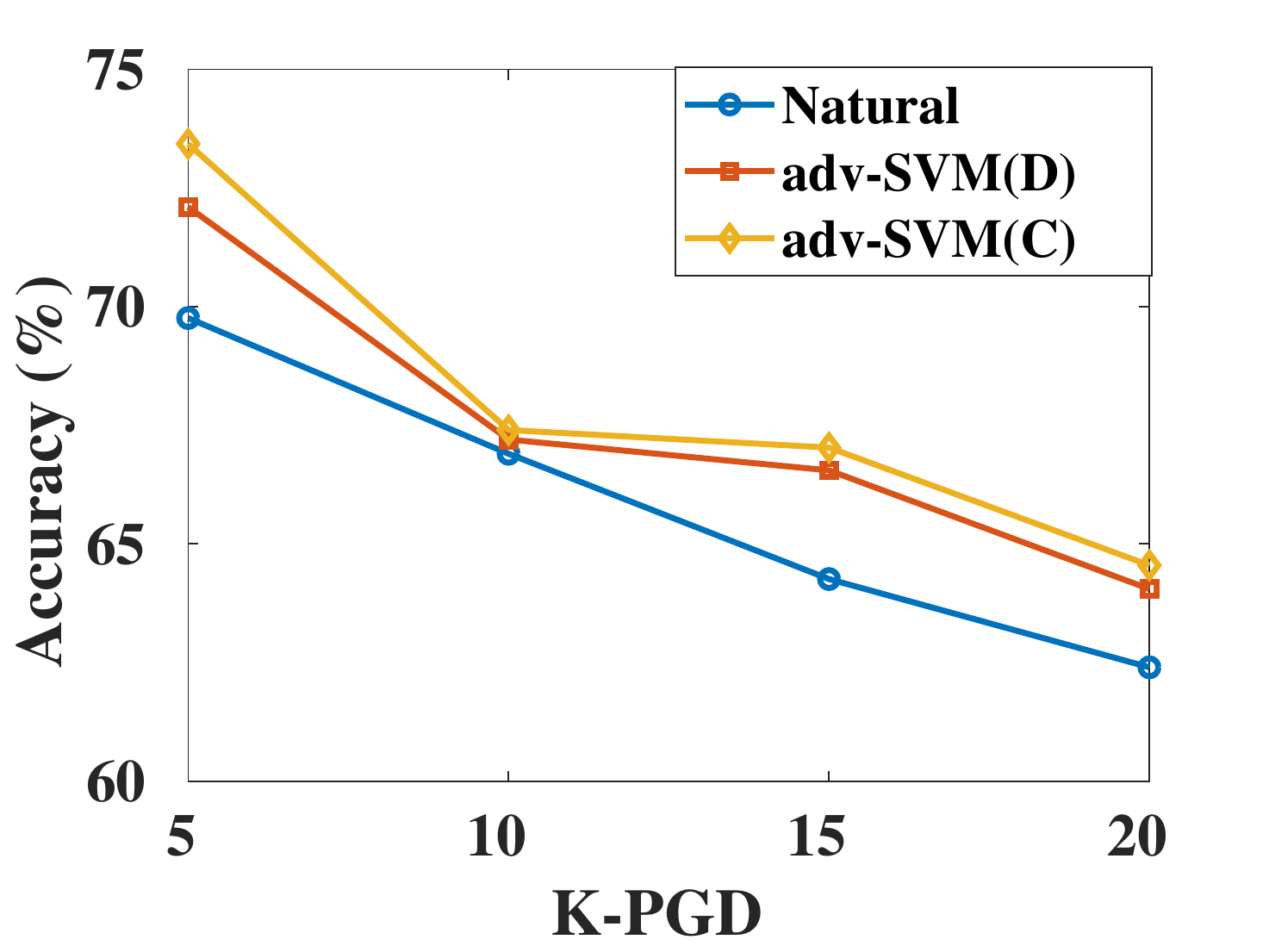}
		\caption{CIFAR10 dog vs. horse}\label{fig:CIFAR5v7 K-PGD}
	\end{subfigure}
    \begin{subfigure}[b]{0.24\textwidth}
	   \includegraphics[width=1.85in]{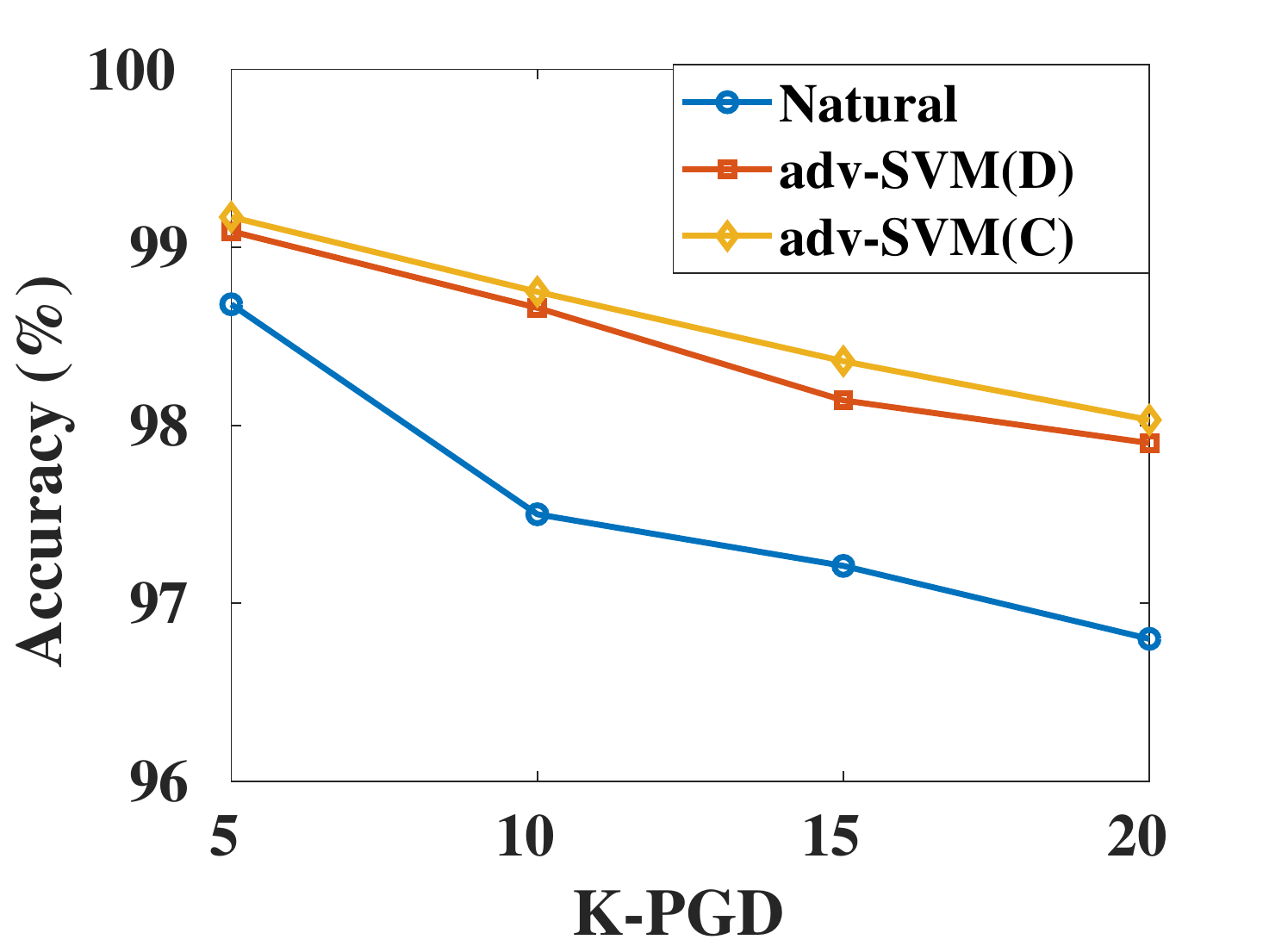}
	   \caption{MNIST8M 6 vs. 8}\label{fig:MNIST8m_6v8 K-PGD}
    \end{subfigure}
    \begin{subfigure}[b]{0.24\textwidth}
    	\includegraphics[width=1.85in]{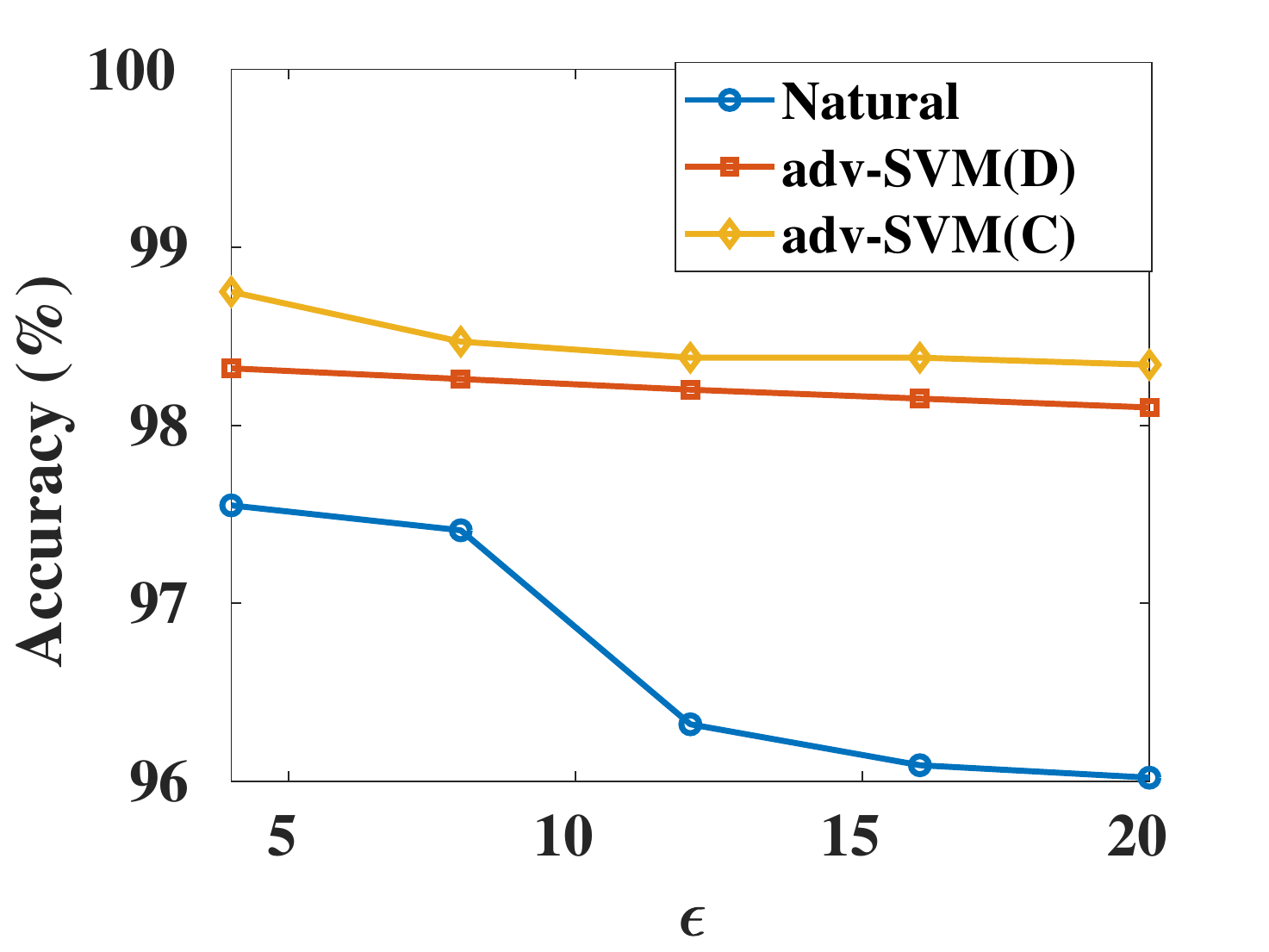}
    	\caption{MNIST 1 vs. 7}\label{fig:MNIST1v7 epsilon}
    \end{subfigure}
	\begin{subfigure}[b]{0.24\textwidth}
		\includegraphics[width=1.85in]{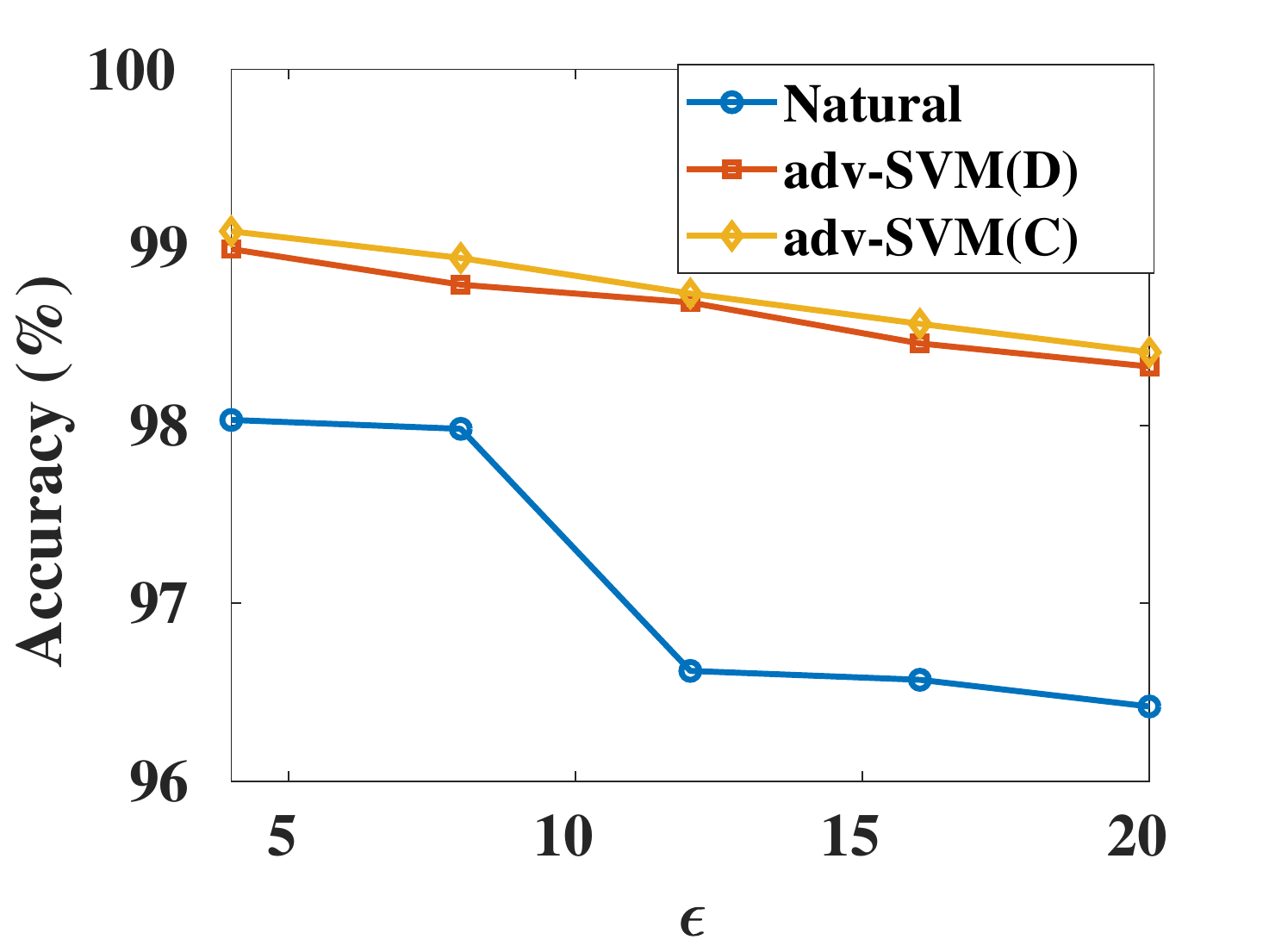}
		\caption{MNIST 8 vs. 9}\label{fig:MNIST8v9 epsilon}
	\end{subfigure}
	\begin{subfigure}[b]{0.24\textwidth}
		\includegraphics[width=1.85in]{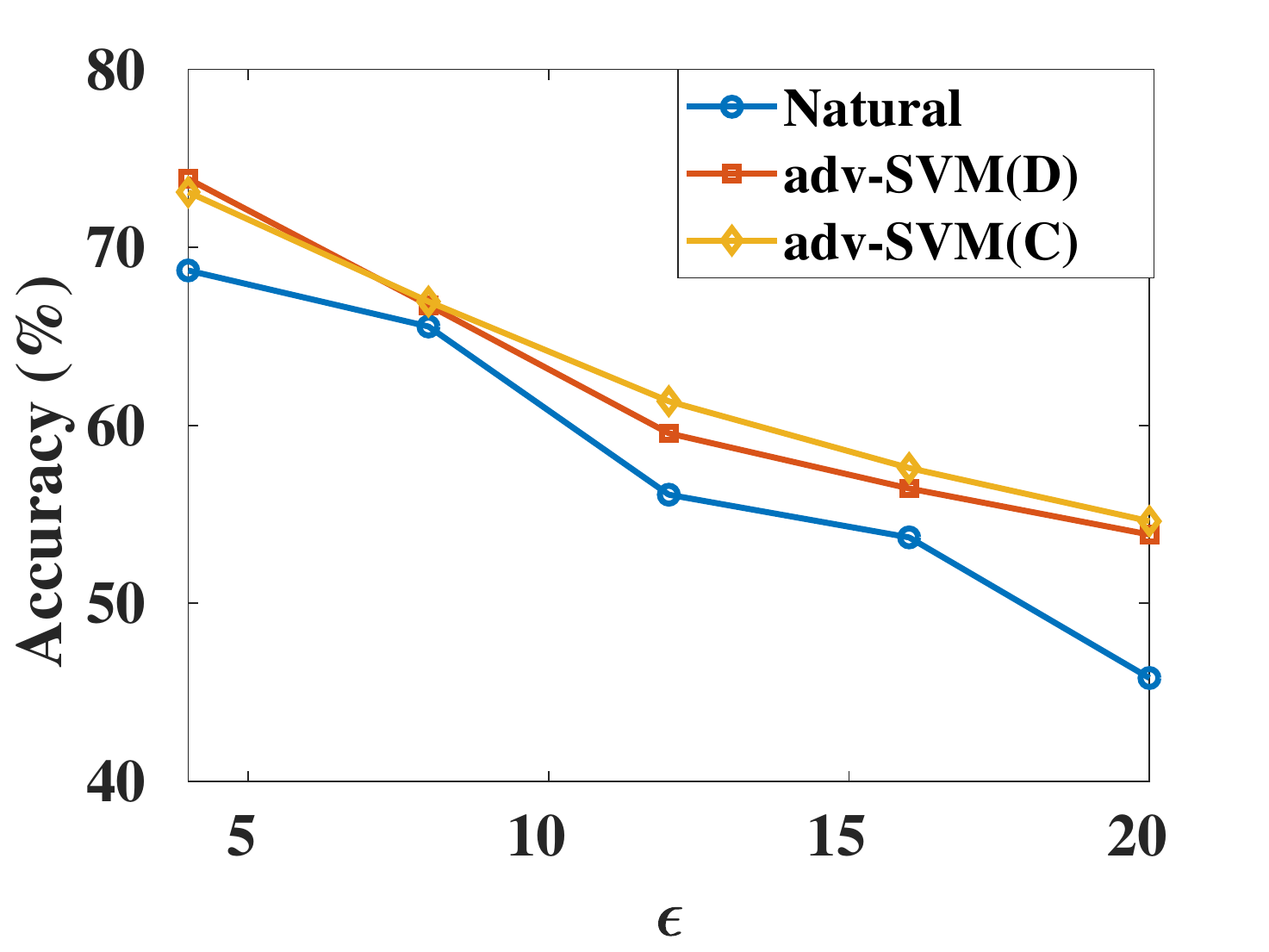}
		\caption{CIFAR10 dog vs. horse}\label{fig:CIFAR5v7 epsilon}
	\end{subfigure}
    \begin{subfigure}[b]{0.24\textwidth}
    	\includegraphics[width=1.85in]{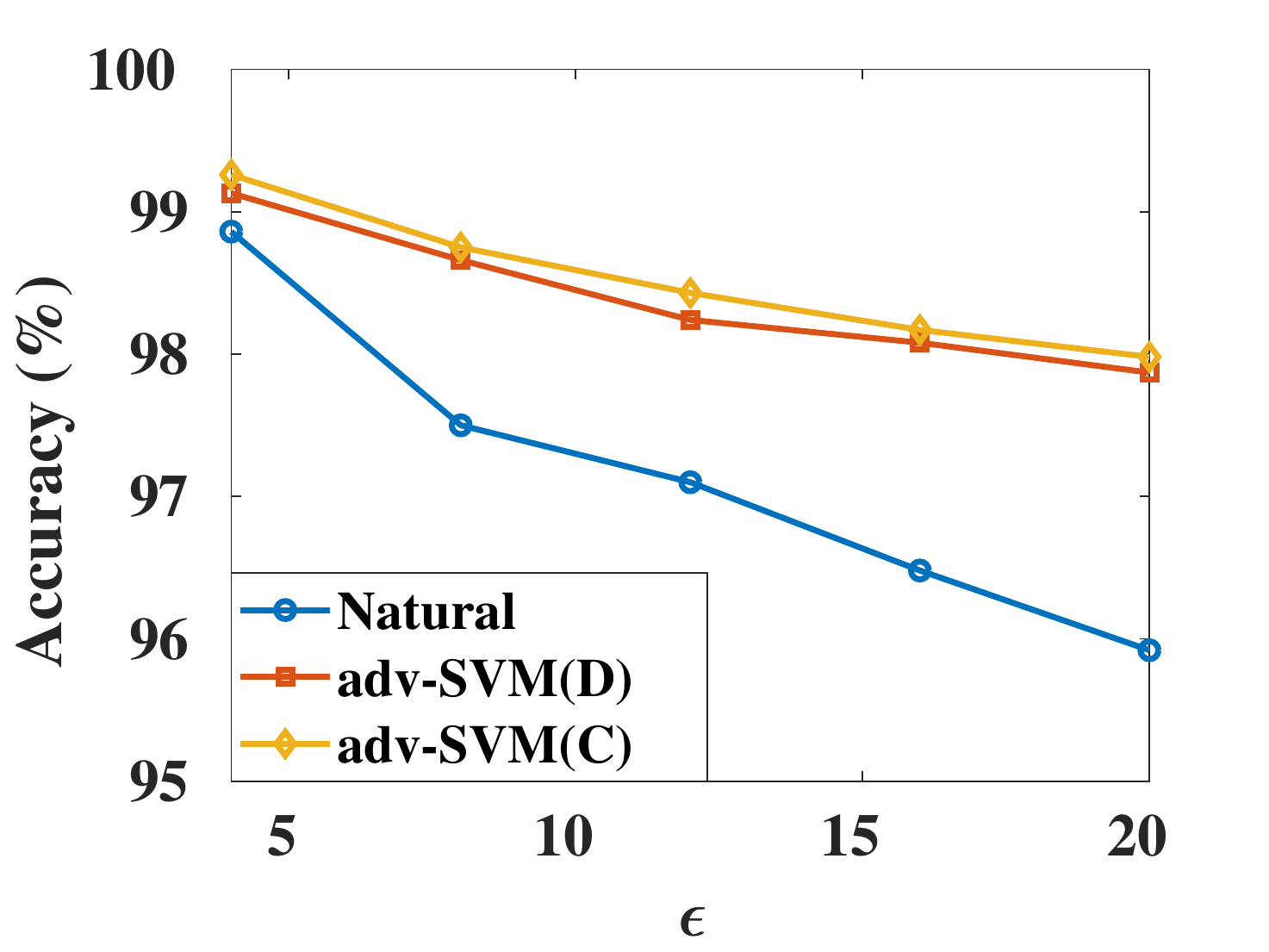}
    	\caption{MNIST8M 6 vs. 8}\label{fig:MNIST8m_6v8 epsilon}
    \end{subfigure}
	\caption{Accuracy of different models when applying different steps PGD attack (Fig. \ref{fig:MNIST1v7 K-PGD}-\ref{fig:MNIST8m_6v8 K-PGD}) and different max perturbation $\epsilon$ (Fig. \ref{fig:MNIST1v7 epsilon}-\ref{fig:MNIST8m_6v8 epsilon}) to generate adversarial examples.
	}
	\label{fig:figure 2}
\end{figure*}


\begin{figure*}[htb]
	\centering
	\begin{subfigure}[b]{0.24\textwidth}
 		\includegraphics[width=1.85in]{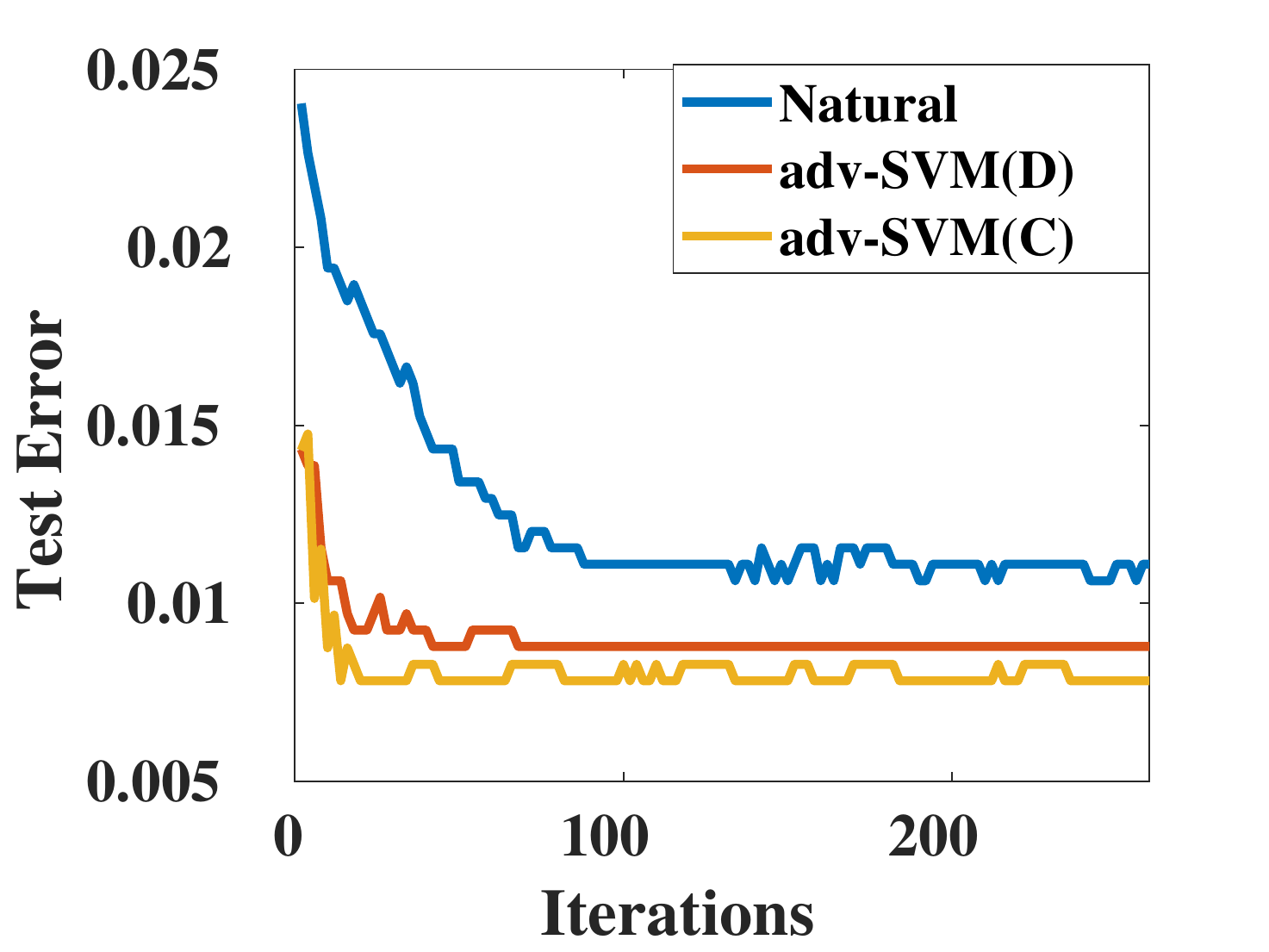}
 		\caption{FGSM}\label{fig:mnist1v7_FGSM}
	 \end{subfigure}
 	\begin{subfigure}[b]{0.24\textwidth}
 		\includegraphics[width=1.85in]{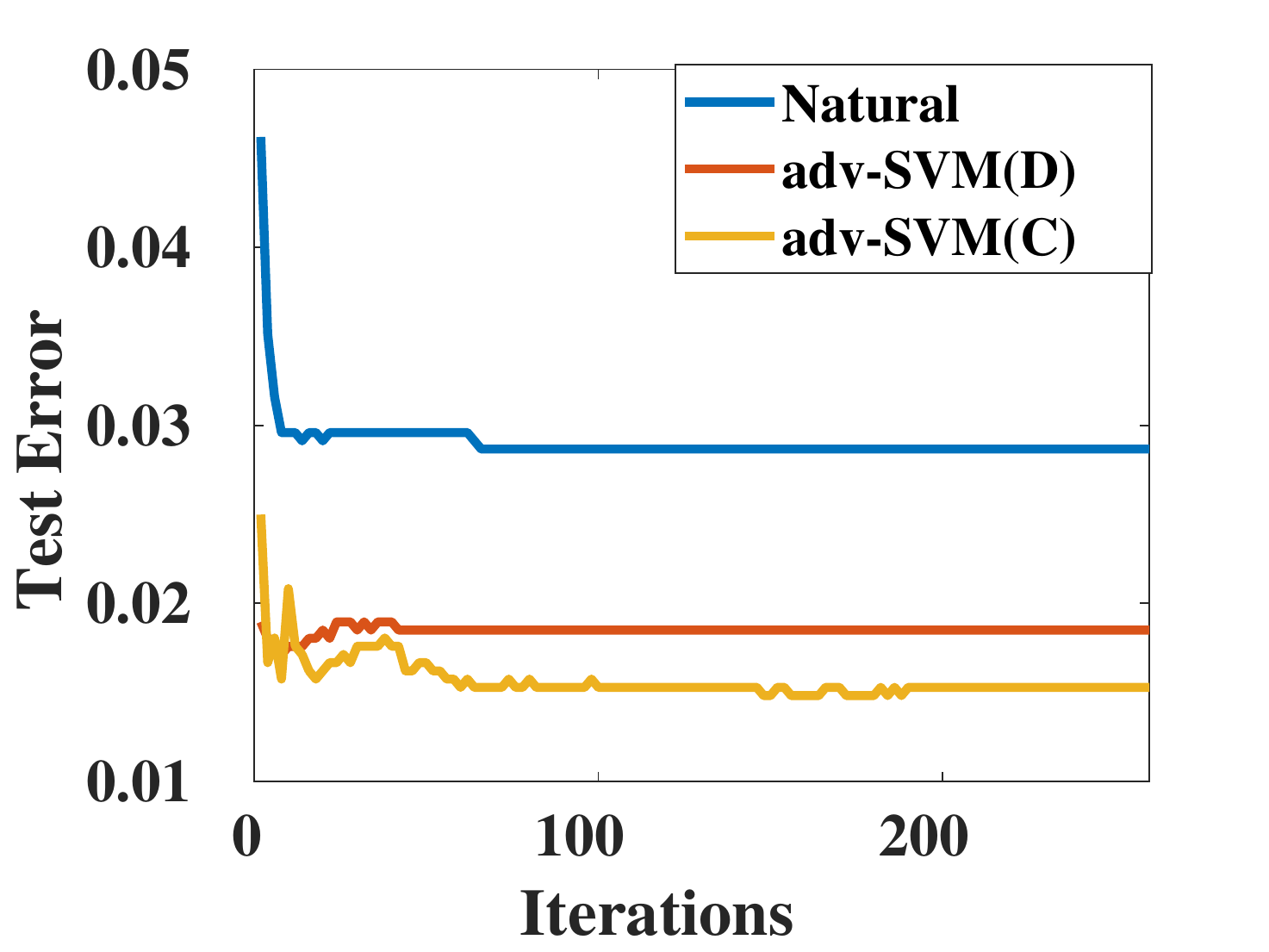}
 		\caption{PGD}\label{fig:mnist1v7_PGD}
 	\end{subfigure}
 	\begin{subfigure}[b]{0.24\textwidth}
 		\includegraphics[width=1.85in]{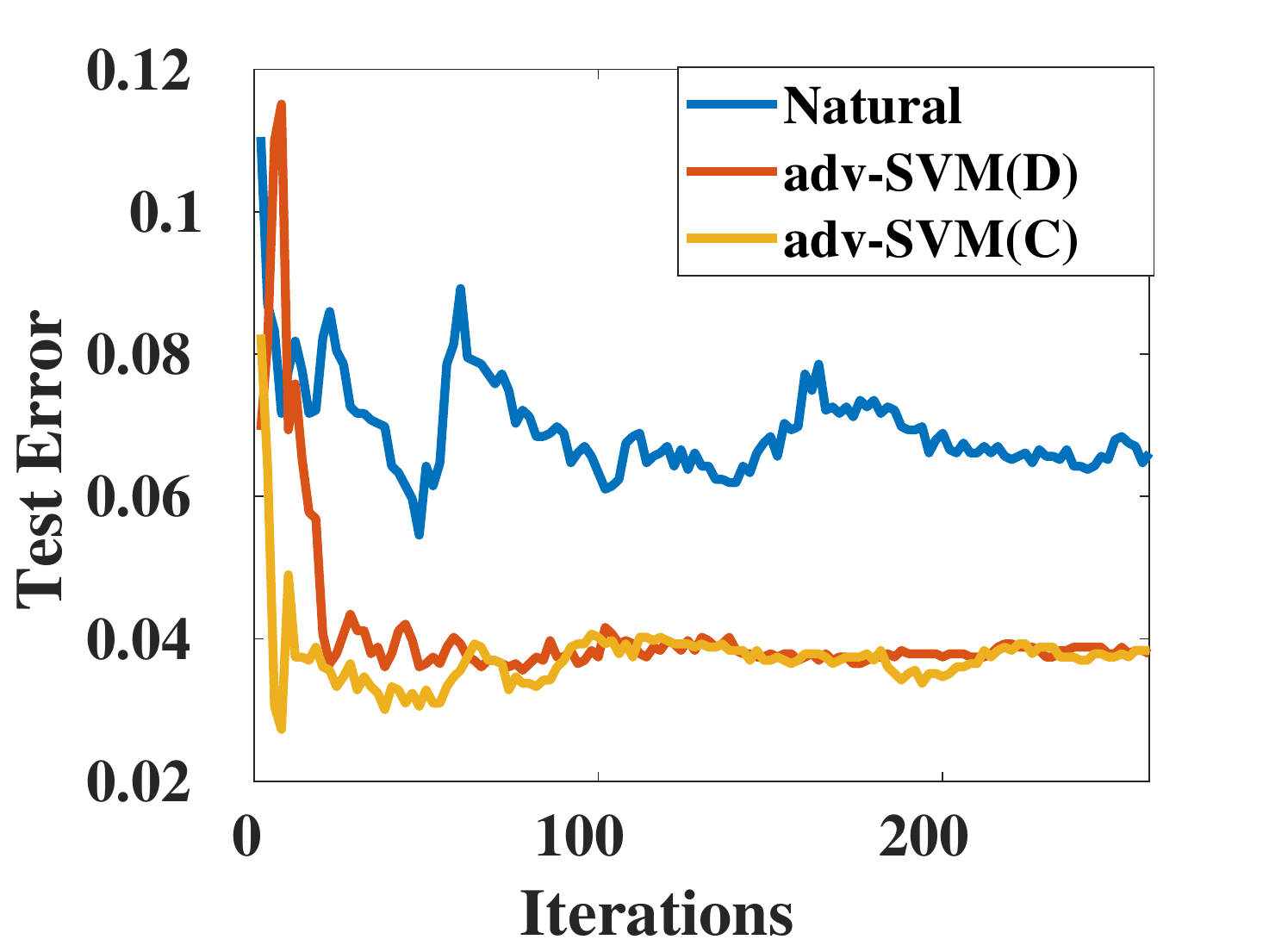}
 		\caption{C$\&$W}\label{fig:mnist1v7_cw}
	\end{subfigure}
 	\begin{subfigure}[b]{0.24\textwidth}
 		\includegraphics[width=1.85in]{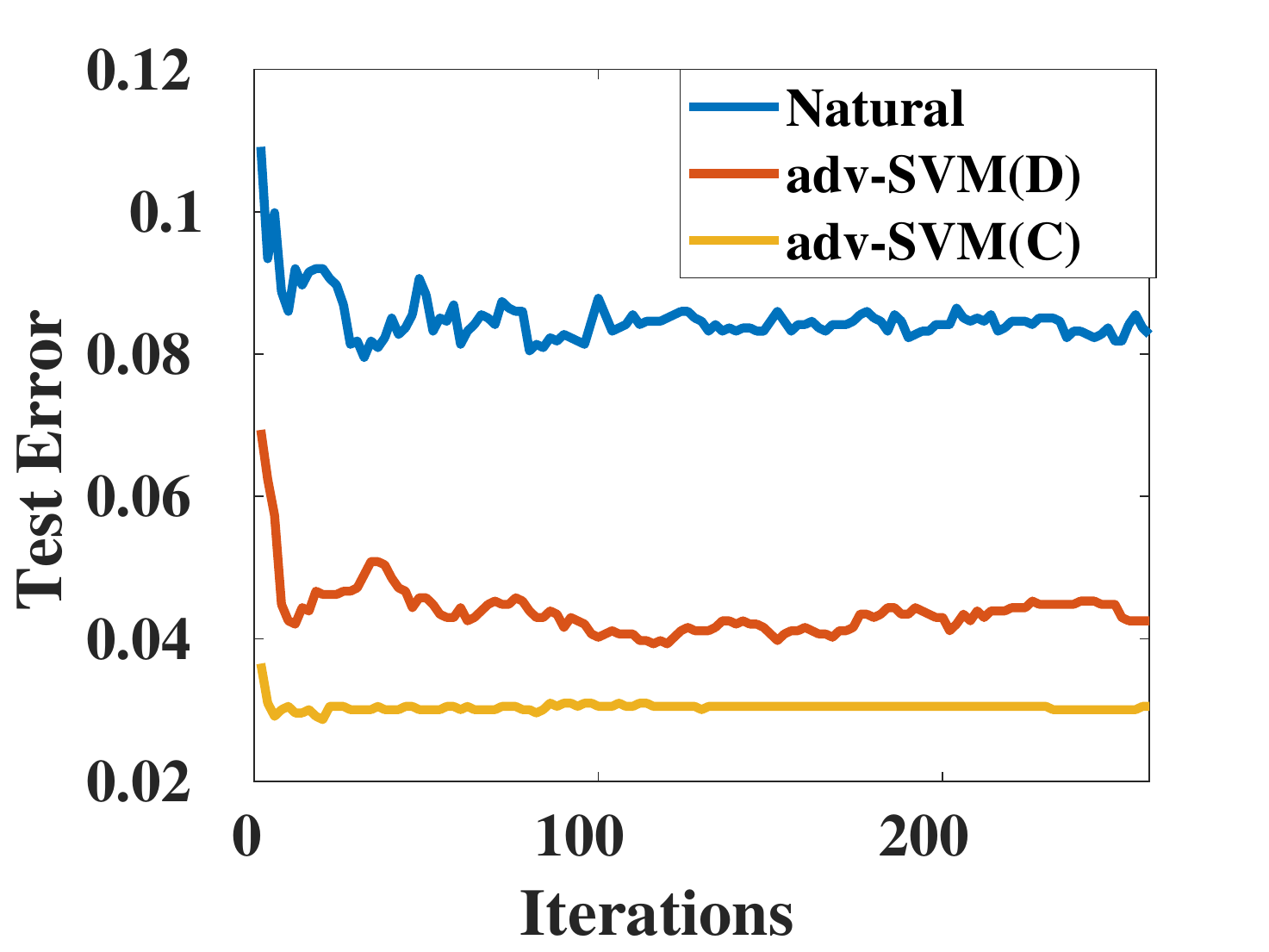}
 		\caption{ZOO}\label{fig:mnist1v7_zoo}
 	\end{subfigure}
	\begin{subfigure}[b]{0.24\textwidth}
		\includegraphics[width=1.85in]{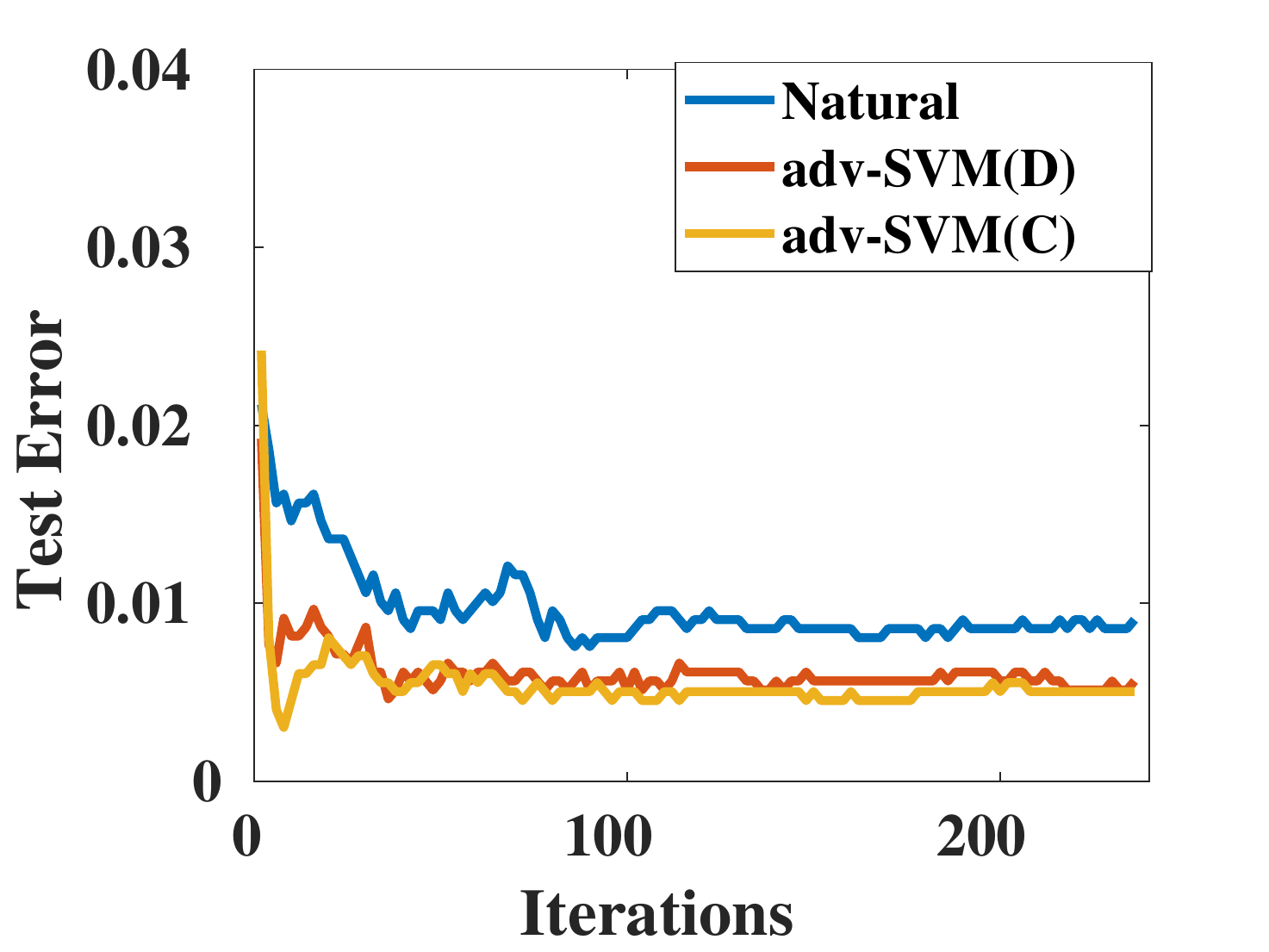}
		\caption{FGSM}\label{fig:mnist8v9_FGSM}
	\end{subfigure}
	\begin{subfigure}[b]{0.24\textwidth}
		\includegraphics[width=1.85in]{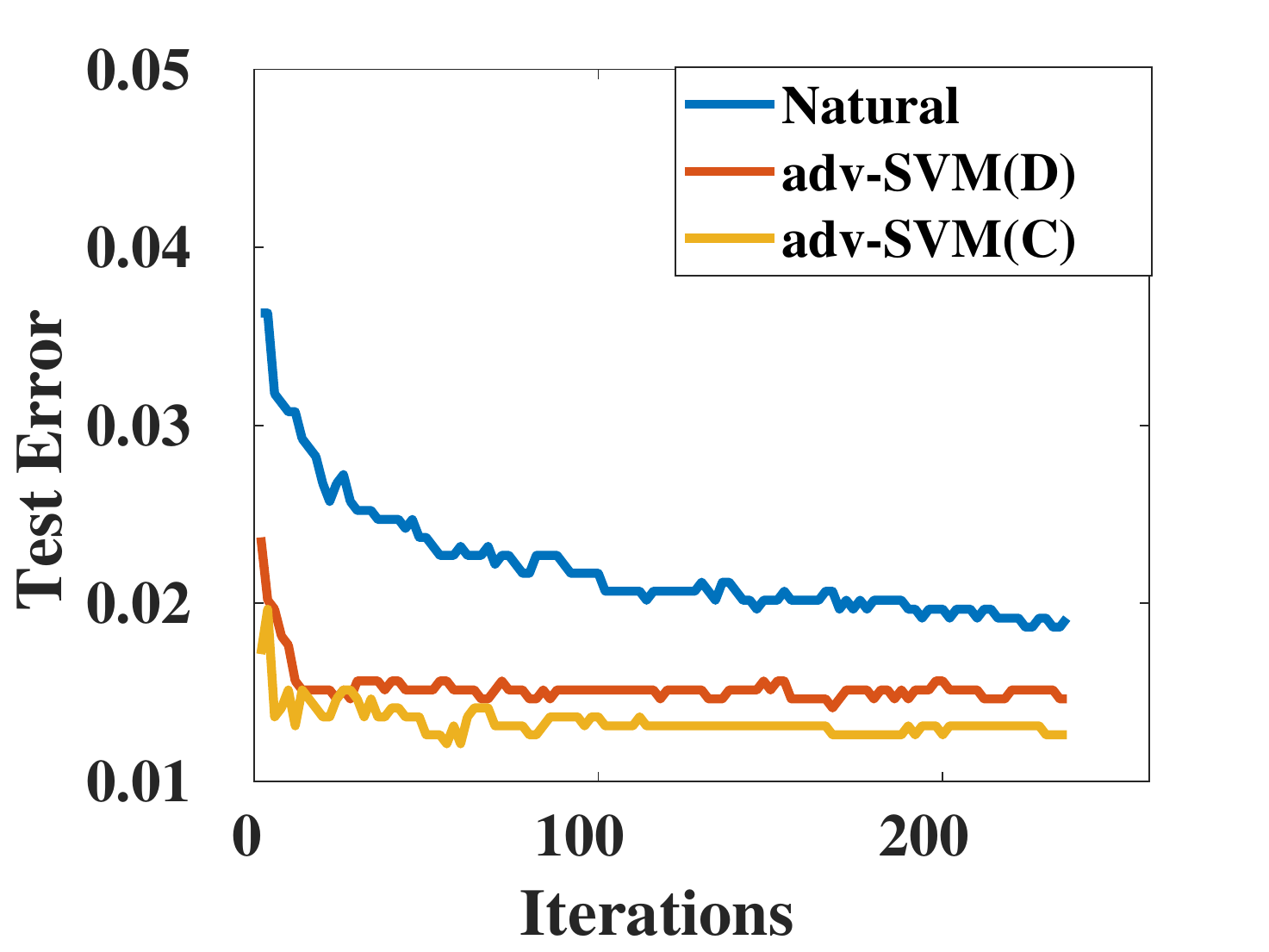}
		\caption{PGD}\label{fig:mnist8v9_PGD}
	\end{subfigure}
	\begin{subfigure}[b]{0.24\textwidth}
		\includegraphics[width=1.85in]{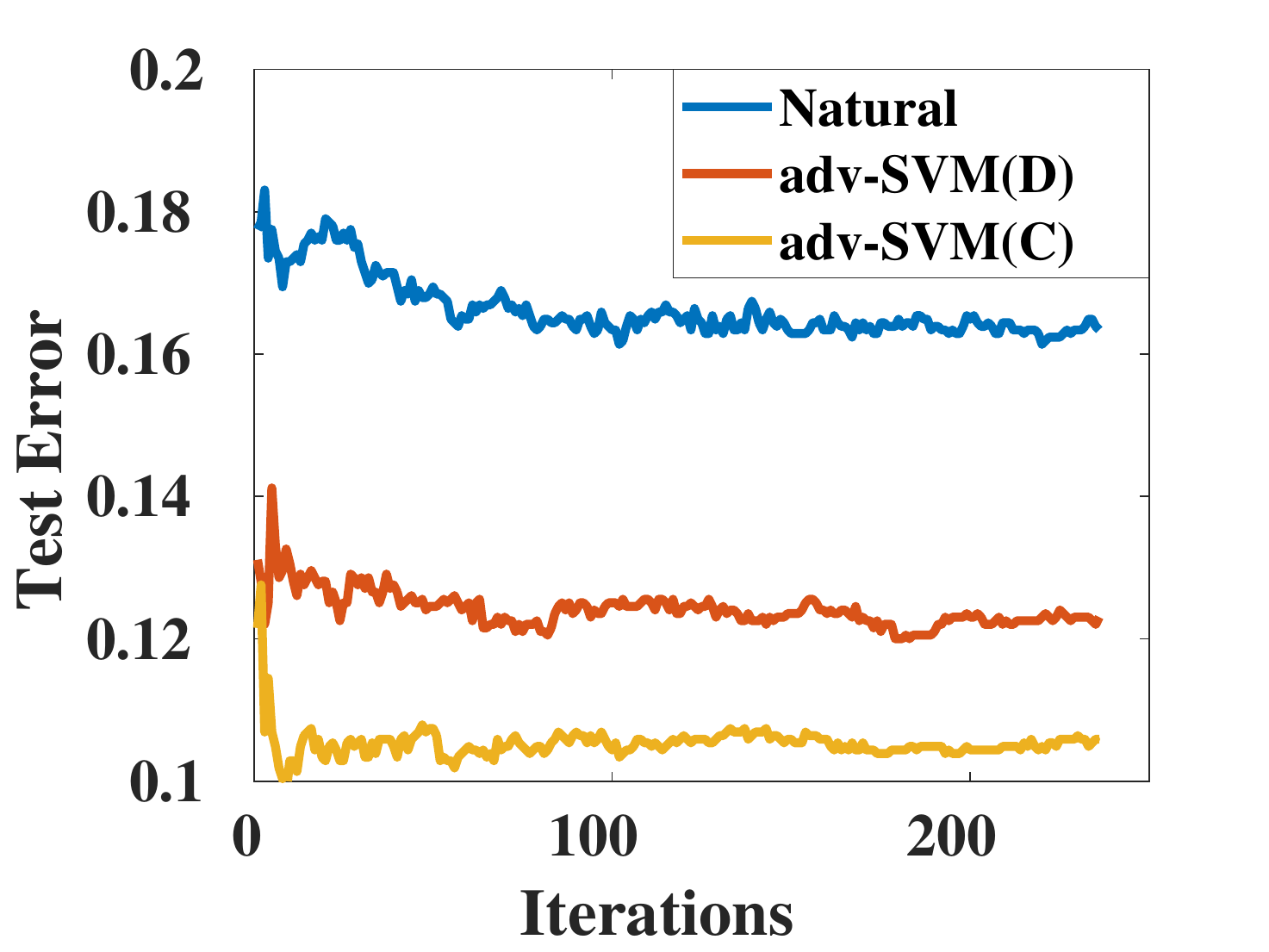}
		\caption{C$\&$W}\label{fig:mnist8v9_cw}
	\end{subfigure}
	\begin{subfigure}[b]{0.24\textwidth}
		\includegraphics[width=1.85in]{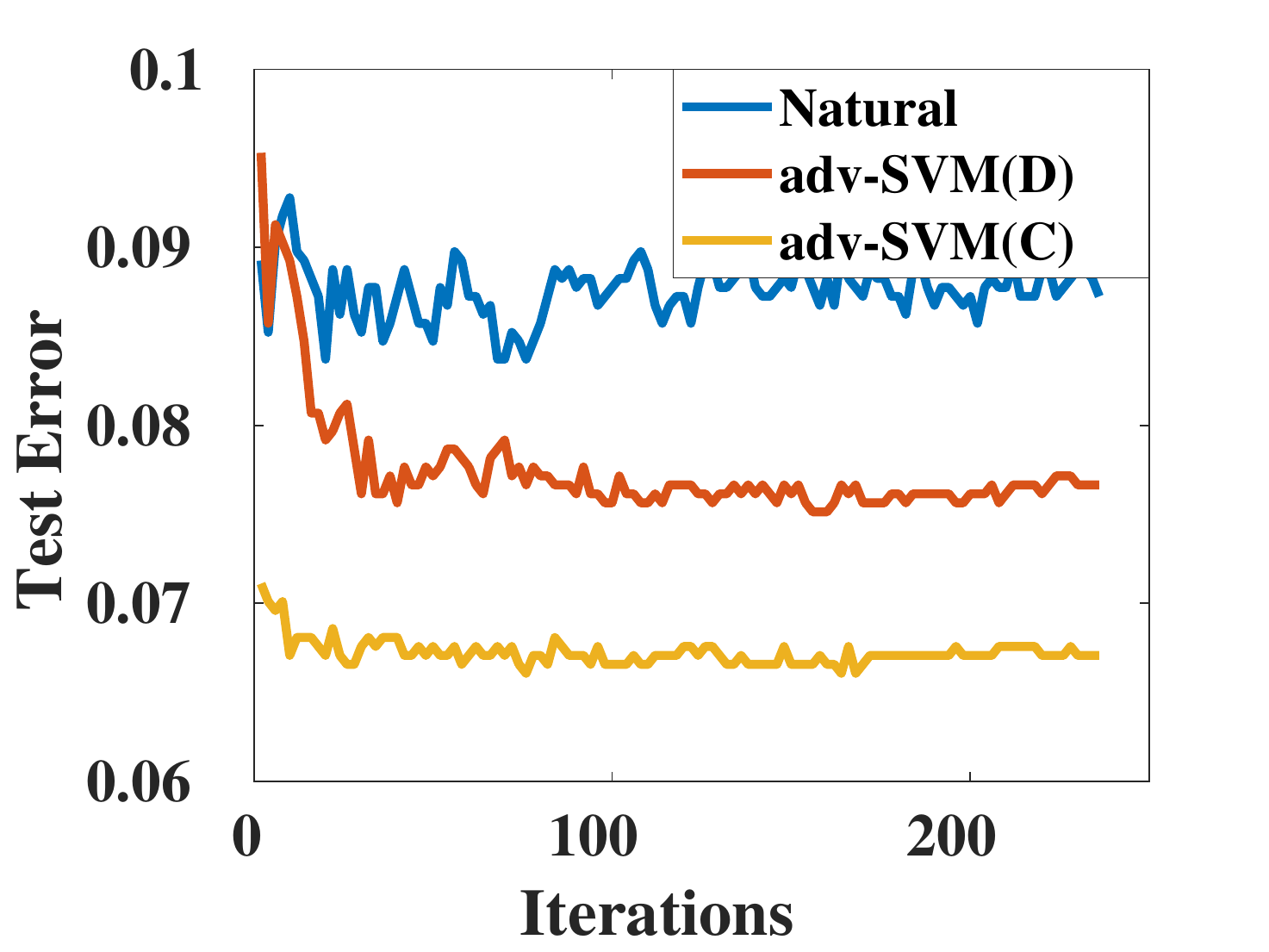}
		\caption{ZOO}\label{fig:mnist8v9_zoo}
	\end{subfigure}
	\begin{subfigure}[b]{0.24\textwidth}
		\includegraphics[width=1.85in]{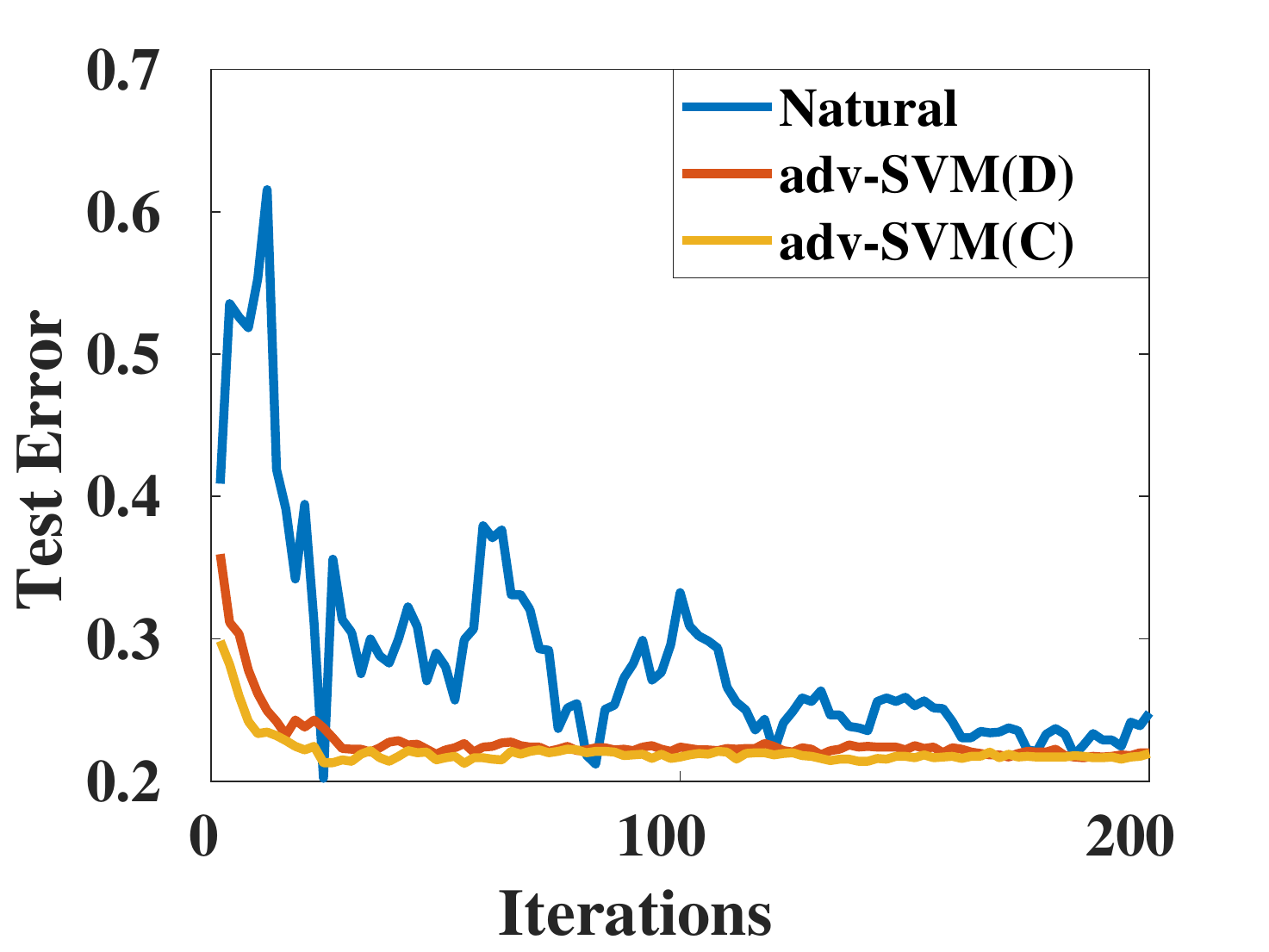}
		\caption{FGSM}\label{fig:cifar5v7_FGSM}
	\end{subfigure}
	\begin{subfigure}[b]{0.24\textwidth}
		\includegraphics[width=1.85in]{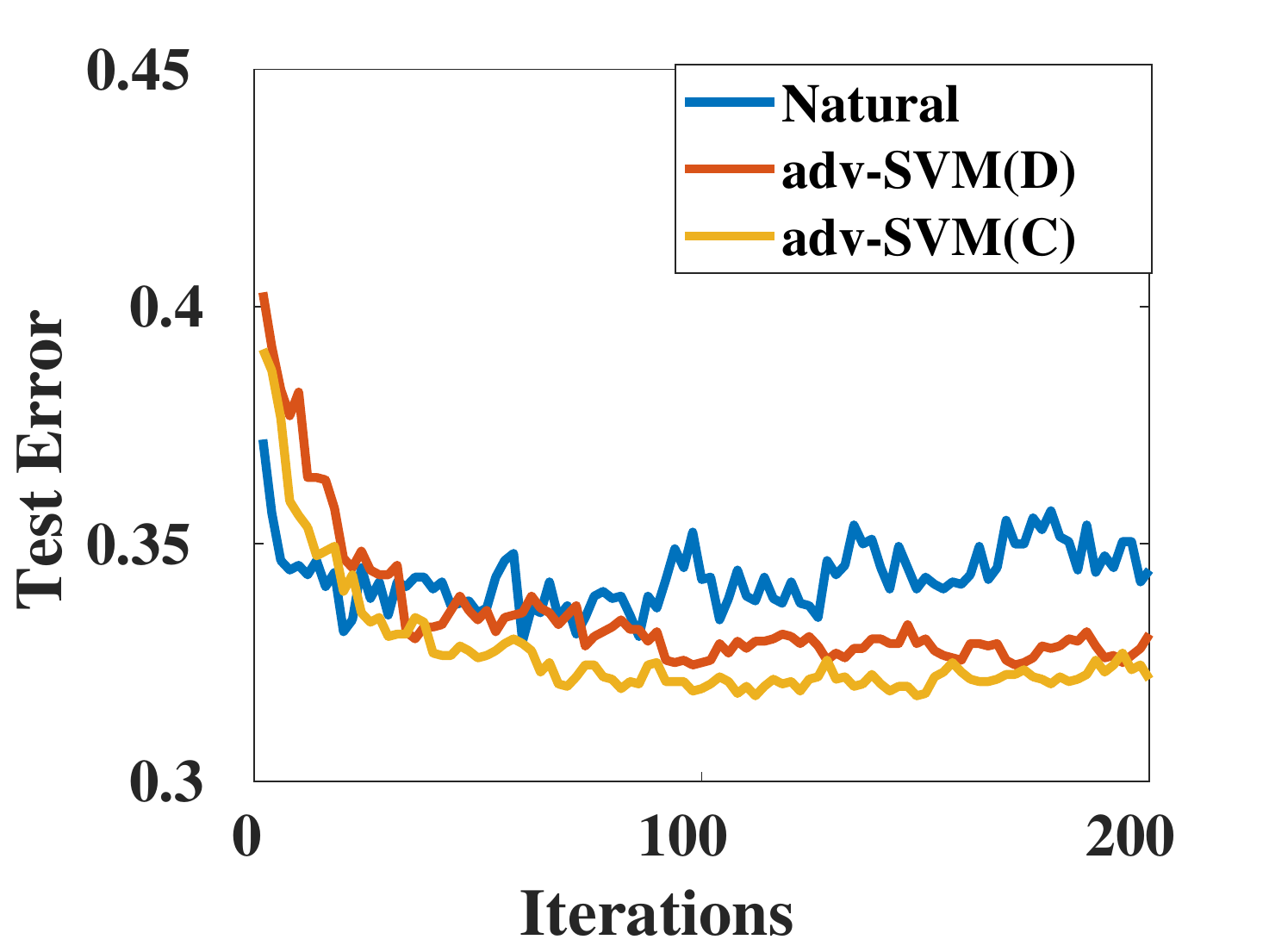}
		\caption{PGD}\label{fig:cifar5v7_PGD}
	\end{subfigure}
	\begin{subfigure}[b]{0.24\textwidth}
		\includegraphics[width=1.85in]{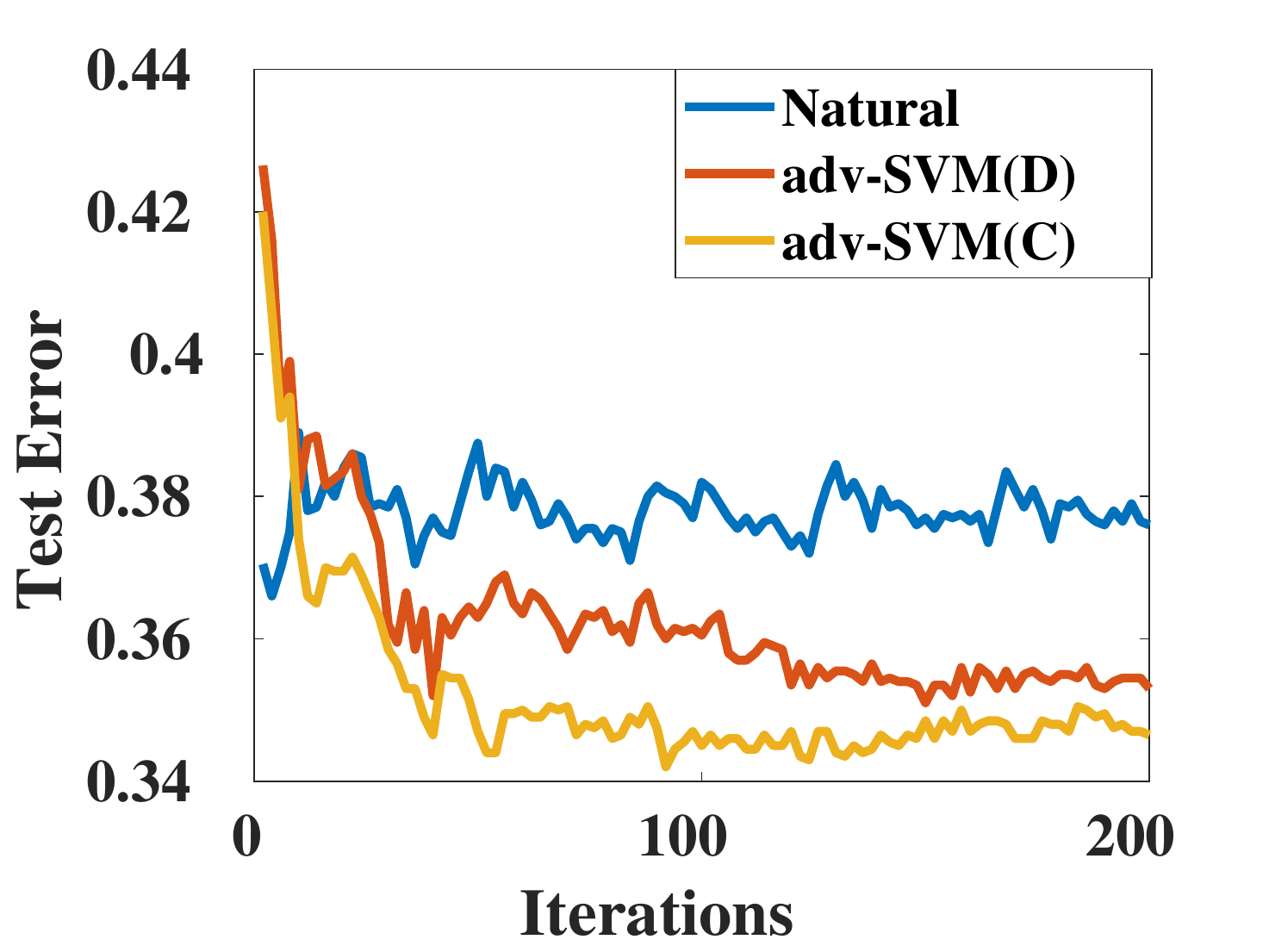}
		\caption{C$\&$W}\label{fig:cifar5v7_cw}
	\end{subfigure}
	\begin{subfigure}[b]{0.24\textwidth}
		\includegraphics[width=1.85in]{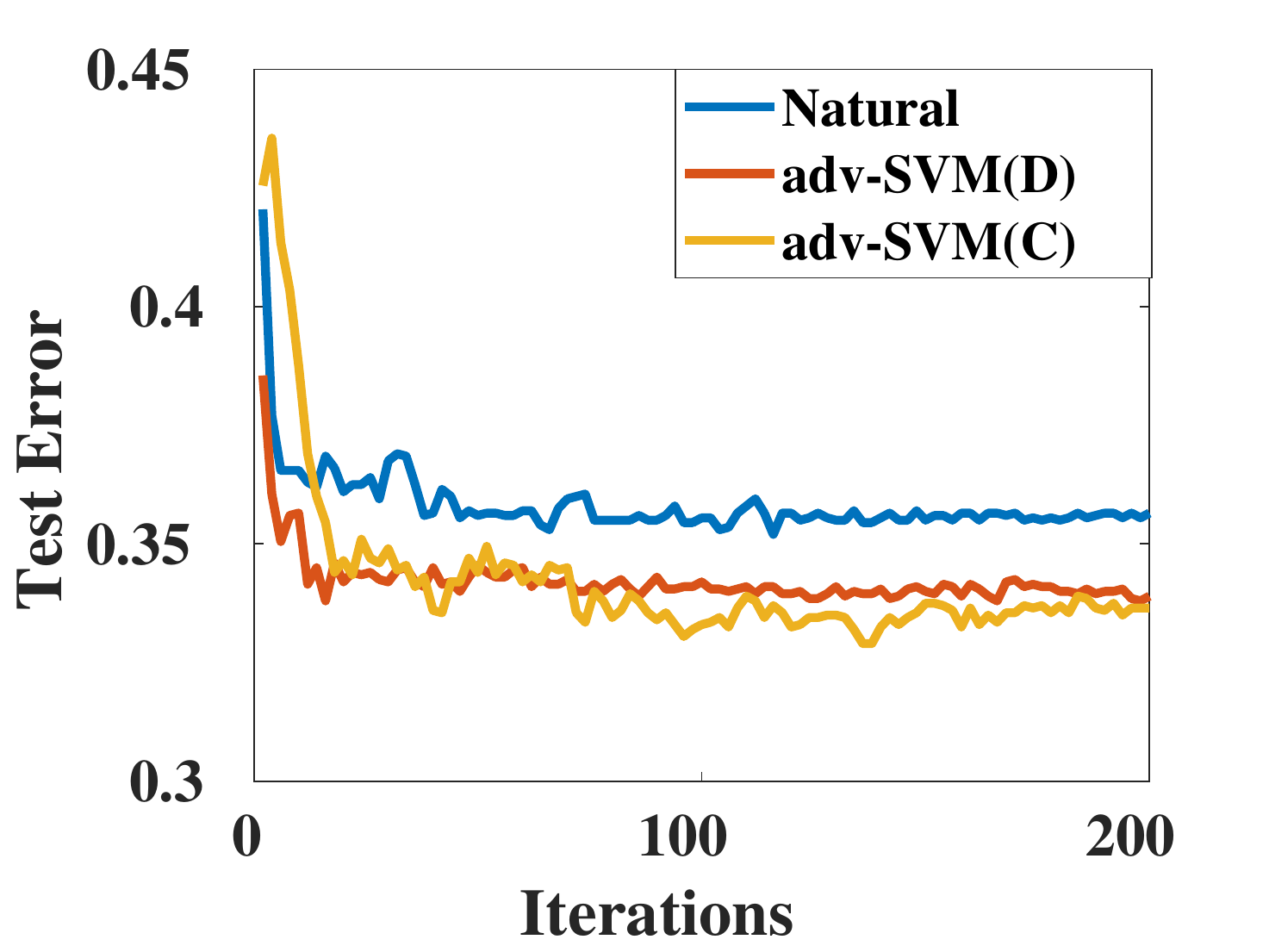}
		\caption{ZOO}\label{fig:cifar5v7_zoo}
	\end{subfigure}
   \begin{subfigure}[b]{0.24\textwidth}
   	\includegraphics[width=1.85in]{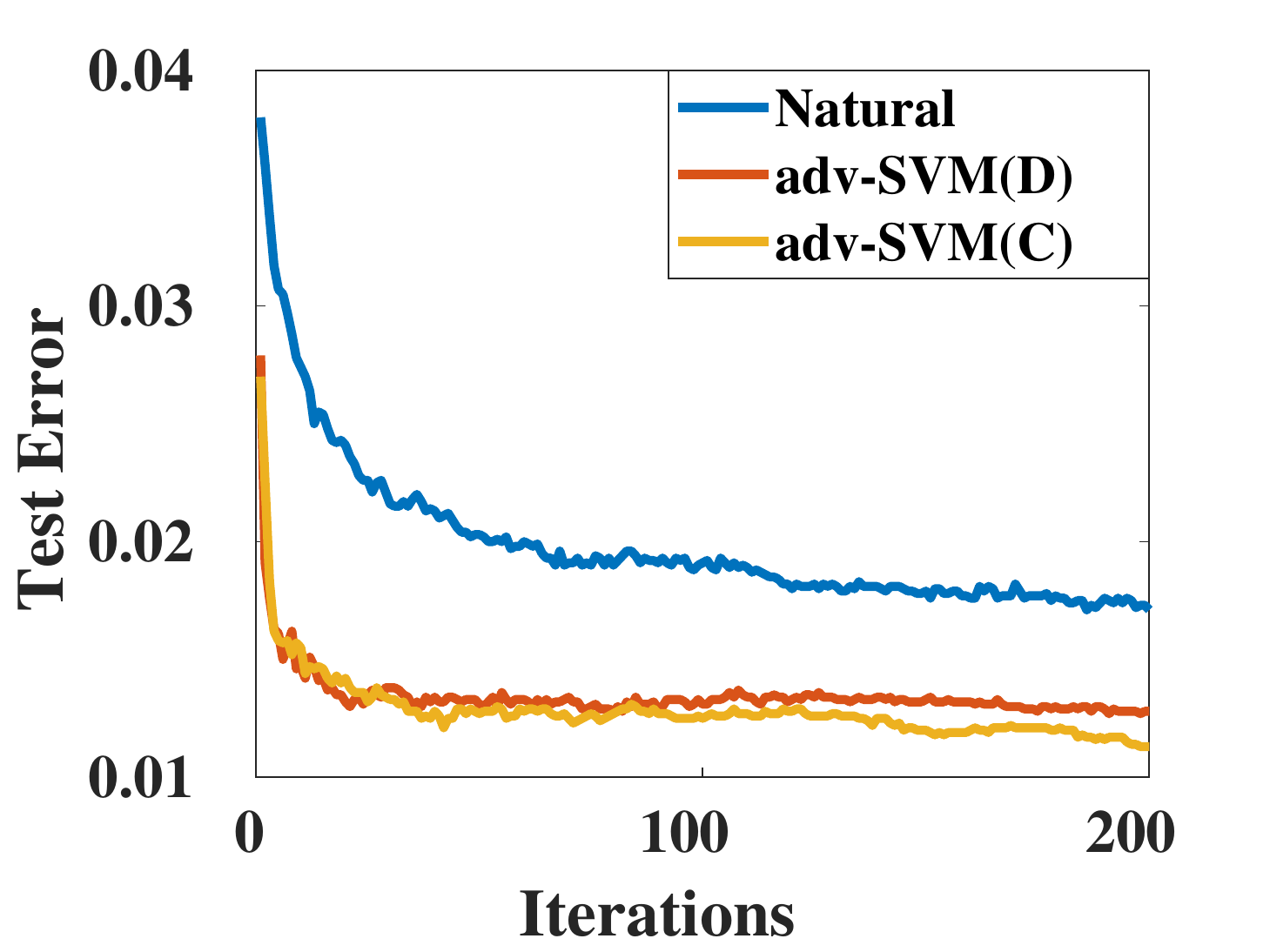}
   	\caption{FGSM}\label{fig:mnist8m_6v8_FGSM}
   \end{subfigure}
   \begin{subfigure}[b]{0.24\textwidth}
   	\includegraphics[width=1.85in]{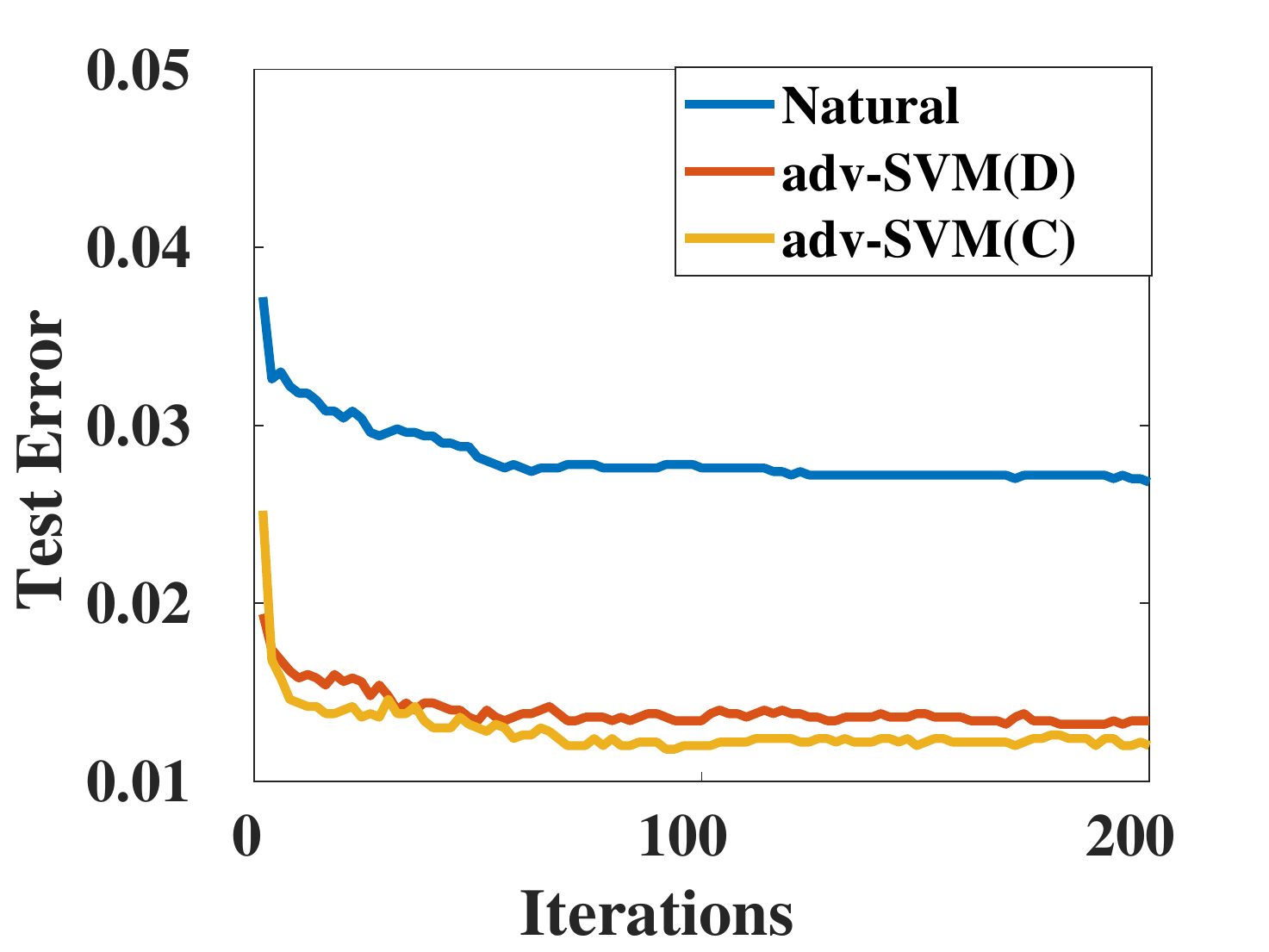}
   	\caption{PGD}\label{fig:mnist8m_6v8_PGD}
   \end{subfigure}
   \begin{subfigure}[b]{0.24\textwidth}
   	\includegraphics[width=1.85in]{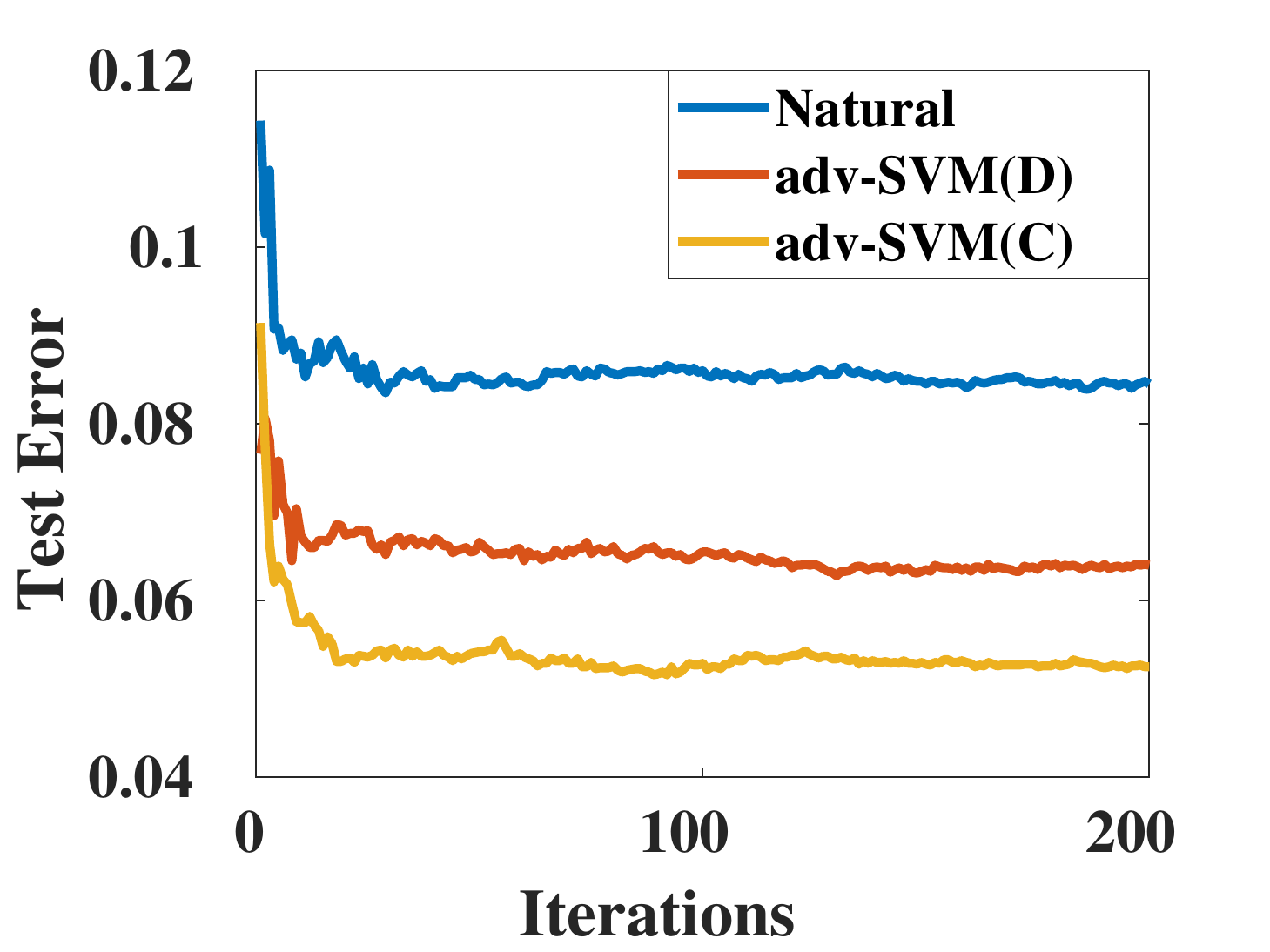}
   	\caption{C$\&$W}\label{fig:mnist8m_6v8_cw}
   \end{subfigure}
   \begin{subfigure}[b]{0.24\textwidth}
   	\includegraphics[width=1.85in]{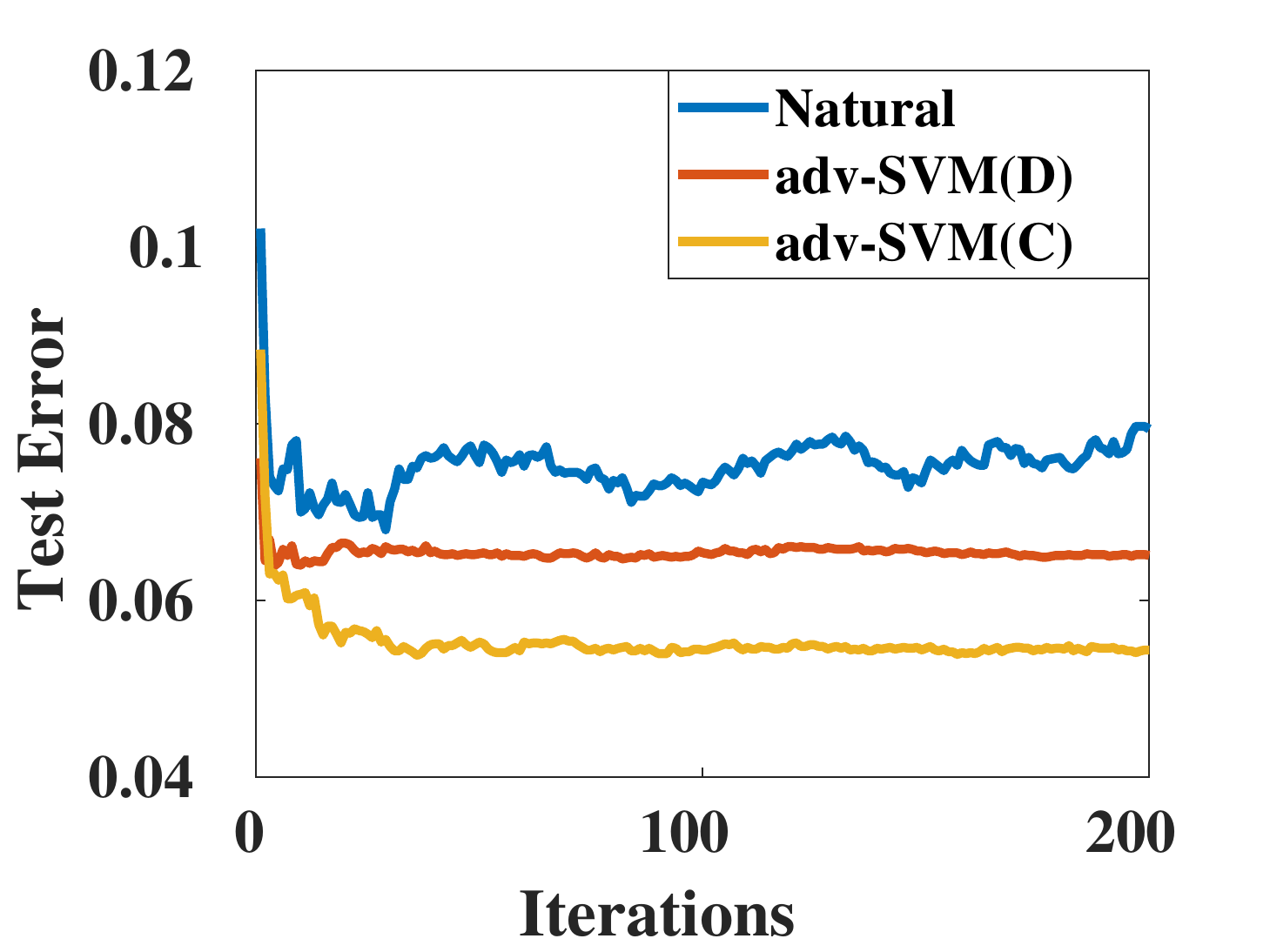}
   	\caption{ZOO}\label{fig:mnist8m_6v8_zoo}
   \end{subfigure}
	\caption{Test error vs. iterations of different models on four attack methods on MNIST 1 vs. 7 (Fig. \ref{fig:mnist1v7_FGSM}-\ref{fig:mnist1v7_zoo}), MNIST 8 vs. 9 (Fig. \ref{fig:mnist8v9_FGSM}-\ref{fig:mnist8v9_zoo}), CIFAR10 dog vs. horse (Fig. \ref{fig:cifar5v7_FGSM}- \ref{fig:cifar5v7_zoo}) and MNIST8M 6 vs. 8 (Fig. \ref{fig:mnist8m_6v8_FGSM}-\ref{fig:mnist8m_6v8_zoo}).
	}
	\label{fig:test error vs. iteration2}
\end{figure*}

	\begin{table*}[]
	\small
	\centering
	\setlength{\abovecaptionskip}{-0.01mm}
	\setlength{\belowcaptionskip}{-4mm}
	\setlength{\tabcolsep}{1.4mm}
	\linespread{1.1}\selectfont
	\begin{tabular}{c|c|c|c|c|c|c|c|c|c|c}
		\hline
		\multirow{2}{*}{model} & \multicolumn{2}{c|}{Normal} & \multicolumn{2}{c|}{FGSM} & \multicolumn{2}{c|}{PGD} & \multicolumn{2}{c|}{C\&W} & \multicolumn{2}{c}{ZOO} \\ \cline{2-11}
		& acc          & time          & acc         & time        & acc        & time        & acc         & time        & acc        & time        \\ \hline
		Natural                & \textbf{99.45$\pm$0.23}             &    9.13         &    98.89$\pm$0.34         & 7.63           & 97.13$\pm$0.45           & 8.70          & 93.39$\pm$0.31            & 7.98           & 91.72$\pm$0.65           & 7.91           \\ \hline
		adv-linear-SVM &99.26$\pm$0.22 &478.24 &98.94$\pm$0.37 &459.51 &98.07$\pm$0.32 &481.24  &90.78$\pm$0.90 &393.05 &92.33$\pm$0.39 &401.51 \\ \hline
		adv-SVM(D)             & 99.35$\pm$0.28             & 7.49              & 99.12$\pm$0.24            & 7.44           & 98.15$\pm$0.56           & 8.28            & 96.16$\pm$0.57            & 8.09            & 95.75$\pm$0.79           & 7.68            \\ \hline
		adv-SVM(C)             &  99.40$\pm$0.37            &  7.15            & \textbf{99.17$\pm$0.39}             & 7.41            &  \textbf{98.47$\pm$0.68}          & 7.92           & \textbf{96.21$\pm$0.84}             & 8.26            & \textbf{96.95$\pm$0.49 }           & 7.97            \\ \hline
	\end{tabular}
	\caption{Accuracy ($\%$) and running time (min) on MNIST 1 vs. 7 against different attacks.}\label{table: table1}
\end{table*}

\begin{table*}[]
	\small
	\centering
	\setlength{\abovecaptionskip}{-0.01mm}
	\setlength{\belowcaptionskip}{-4mm}
	\setlength{\tabcolsep}{1.4mm}
	\linespread{1.1}\selectfont
	\begin{tabular}{c|c|c|c|c|c|c|c|c|c|c}
		\hline
		\multirow{2}{*}{model} & \multicolumn{2}{c|}{Natural} & \multicolumn{2}{c|}{FGSM} & \multicolumn{2}{c|}{PGD} & \multicolumn{2}{c|}{C\&W} & \multicolumn{2}{c}{ZOO} \\ \cline{2-11} 
		& acc          & time          & acc         & time        & acc        & time        & acc         & time        & acc        & time        \\ \hline
		Natural                & \textbf{99.70$\pm$0.18}             & 7.45              & 99.09$\pm$0.24            & 6.52            & 98.08$\pm$0.25           & 6.58            & 83.36$\pm$0.38            & 7.75            & 91.43$\pm$0.47           & 6.24            \\ \hline
		adv-linear-SVM &98.69$\pm$0.28 &437.20 &96.23$\pm$0.75 &408.68 &94.04$\pm$0.48 &457.39  &80.09$\pm$0.68 &473.25 &84.67$\pm$0.62 &482.13 \\ \hline
		adv-SVM(D)             & 99.55$\pm$0.25             & 8.59              & 99.44$\pm$0.39            & 7.12            & 98.54$\pm$0.27           & 7.65            & 87.70$\pm$0.64            & 7.62            & 92.33$\pm$0.84           & 9.11            \\ \hline
		adv-SVM(C)             & 99.60$\pm$0.27             & 8.09              & \textbf{99.50$\pm$0.21}            & 8.28            &  \textbf{98.74$\pm$0.37}          & 7.98            & \textbf{89.41$\pm$0.46}            & 7.48            & \textbf{93.29$\pm$0.72}           & 7.98            \\ \hline
	\end{tabular}
	\caption{Accuracy ($\%$) and running time (min) of MNIST 8 vs. 9 against different attacks.}
\end{table*}

\begin{table*}[]
	\small
	\centering
	\setlength{\abovecaptionskip}{-0.01mm}
	\setlength{\belowcaptionskip}{-4mm}
	\setlength{\tabcolsep}{1.4mm}
	\linespread{1.1}\selectfont
	\begin{tabular}{c|c|c|c|c|c|c|c|c|c|c}
		\hline
		\multirow{2}{*}{model} & \multicolumn{2}{c|}{Natural} & \multicolumn{2}{c|}{FGSM} & \multicolumn{2}{c|}{PGD} & \multicolumn{2}{c|}{C\&W} & \multicolumn{2}{c}{ZOO} \\ \cline{2-11} 
		& acc          & time          & acc         & time        & acc        & time        & acc         & time        & acc        & time        \\ \hline
		Natural                & \textbf{80.65$\pm$0.21}             & 20.69             &  75.20$\pm$0.41           &  19.96           &  65.55$\pm$0.38          & 19.16            & 62.40$\pm$0.79            & 21.57            & 64.35$\pm$0.47           & 21.18            \\ \hline
		adv-linear-SVM &80.34$\pm$0.15 &642.15 &76.38$\pm$0.54 &637.54 &64.87$\pm$0.87 &605.82  &58.45$\pm$0.76 &627.98 &61.37$\pm$0.59 &597.53 \\ \hline
		adv-SVM(D)             & 80.45$\pm$0.15             & 20.23              & 78.00$\pm$0.27            & 17.31            & 66.90$\pm$0.47           & 14.42            & 64.70$\pm$0.69            & 18.84            & 66.18$\pm$0.52           & 13.72           \\ \hline
		adv-SVM(C)             & 80.60$\pm$0.24             & 20.35              & \textbf{78.05$\pm$0.34}            & 17.50            & \textbf{67.85$\pm$0.27}           & 14.25            & \textbf{65.35$\pm$0.52}            & 17.44            & \textbf{66.35$\pm$0.37}           & 20.28            \\ \hline
	\end{tabular}
	\caption{Accuracy ($\%$) and running time (min) of CIFAR10 dog vs. horse against different attacks.}
\end{table*}

	\begin{table*}[]
	\small
	\centering
	\setlength{\abovecaptionskip}{-0.01mm}
	\setlength{\belowcaptionskip}{-4mm}
	\setlength{\tabcolsep}{1.4mm}
	\linespread{1.1}\selectfont
	\begin{tabular}{c|c|c|c|c|c|c|c|c|c|c}
		\hline
		\multirow{2}{*}{model} & \multicolumn{2}{c|}{Normal} & \multicolumn{2}{c|}{FGSM} & \multicolumn{2}{c|}{PGD} & \multicolumn{2}{c|}{C\&W} & \multicolumn{2}{c}{ZOO} \\ \cline{2-11}
		& acc          & time          & acc         & time        & acc        & time        & acc         & time        & acc        & time        \\ \hline
		Natural                &\textbf{99.70$\pm$0.08}             &57.06            & 98.29$\pm$0.24           &51.74           & 97.46$\pm$0.37          &57.03          &91.55$\pm$0.71            &61.39           &92.06$\pm$0.58        &62.30        \\ \hline
		adv-linear-SVM &98.94$\pm$0.27 &2415.60 &98.53$\pm$0.37 &3122.62 &97.16$\pm$0.46 &2997.37  &90.87$\pm$0.51 &2706.14 &89.73$\pm$0.67 &2749.64 \\ \hline
		adv-SVM(D)             &99.21$\pm$0.27            &59.05              &98.72$\pm$0.36            &56.44        &98.70$\pm$0.58           &56.98            &93.60$\pm$0.64             &62.74          &93.49$\pm$0.89        &62.26           \\ \hline
		adv-SVM(C)             &99.46$\pm$0.23          &58.64            &\textbf{98.87$\pm$0.45}              & 55.10         &\textbf{98.80$\pm$0.24 }        &55.82           & \textbf{94.75$\pm$0.63  }         &66.67           &\textbf{94.56$\pm$0.64}        &63.81          \\ \hline
	\end{tabular}
	\caption{Accuracy ($\%$) and running time (min) on MNIST8M 6 vs.8 against different attacks.}\label{table: table1}
\end{table*}

	\bibliography{aaai21}
	\bibliographystyle{aaai21}
\end{document}